\documentclass{article}

\PassOptionsToPackage{numbers,  sort&compress}{natbib}
\usepackage[nonatbib, final]{neurips_2020}
\usepackage{natbib}

\usepackage{color}
\definecolor{mycitecolor}{rgb}{0,0.08,0.45}

\usepackage[utf8]{inputenc} %
\usepackage[T1]{fontenc}    %
\usepackage[colorlinks=true, citecolor=mycitecolor]{hyperref}
\usepackage{url}            %
\usepackage{booktabs}       %
\usepackage{amsfonts}       %
\usepackage{nicefrac}       %
\usepackage{microtype}      %
\usepackage{multirow}   %
\usepackage{algorithm}
\usepackage{algorithmic}
\usepackage{wrapfig}
\usepackage{graphicx}
\usepackage{comment}
\usepackage{amsmath,amssymb} %
\usepackage{multirow}
\usepackage{ctable} %
\usepackage{pifont}%

\usepackage{amsmath,amsfonts,bm}

\newcommand{\expectation}{{\mathbb{E}}}

\def\1{\bm{1}}

\def\sp{space}

\def\vs{{\bm{s}}}

\DeclareMathAlphabet{\mathsfit}{\encodingdefault}{\sfdefault}{m}{sl}
\SetMathAlphabet{\mathsfit}{bold}{\encodingdefault}{\sfdefault}{bx}{n}

\newcommand{\vsp}{\csname v\sp \endcsname}
\newcommand{\hsp}{\csname h\sp \endcsname}

\usepackage{makecell}
\usepackage{caption}
\usepackage{amsthm}
\usepackage{thmtools} 
\usepackage{symbols}

\definecolor{mybrown}{rgb}{0.87058824, 0.56078431, 0.01960784}
\definecolor{myblue}{rgb}{0.3372549 , 0.70588235, 0.91372549}
\definecolor{mypurple}{rgb}{0.8, 0.47058824, 0.7372549 }
\definecolor{myorange}{rgb}{0.835, 0.368, 0}
\definecolor{mygreen}{rgb}{0.00784314, 0.61960784, 0.45098039}
\definecolor{mygt}{rgb}{0.0078125 , 0.57421875, 0.40625}
\definecolor{mysp}{rgb}{0.84765625, 0.515625  , 0.0234375}

\definecolor{Set2c0}{rgb}{0.54117647, 0.16862745, 0.88627451}

\title{High-Throughput Synchronous Deep RL}

\author{
  Iou-Jen Liu, ~Raymond A. Yeh,~Alexander G. Schwing  \\
  University of Illinois at Urbana-Champaign\\
   \texttt{\{iliu3, yeh17, aschwing\}@illinois.edu} \\
}

\let\cite\citep

\begin{document}
\maketitle
\begin{abstract}
Deep reinforcement learning (RL) is computationally demanding and requires processing of many data points. 
Synchronous methods enjoy training stability while having lower data throughput. In contrast, 
asynchronous methods achieve high throughput but suffer from stability issues  and lower sample efficiency due to `stale policies.' 
To combine the advantages of both methods we propose High-Throughput Synchronous Deep Reinforcement Learning (HTS-RL). In HTS-RL, we perform  learning and rollouts concurrently,  devise a system design which avoids  `stale policies'  and ensure that actors interact with environment replicas in an asynchronous manner while maintaining \emph{full determinism}. We evaluate our approach on Atari games and the Google Research Football environment. Compared to  synchronous baselines, HTS-RL is $2-6\times$ faster. Compared to state-of-the-art asynchronous methods, HTS-RL has competitive throughput and consistently achieves  higher average episode rewards. 
\end{abstract}

\section{Introduction}
\label{sec:intro}

Deep reinforcement learning (RL) has been impressively successful on a wide variety of tasks, including playing of video games~\cite{dqn1, dqn2, starcraft2, a3c, a2c, ga3c, kf, pic, Jain19, JainECCV2020} and robotic control~\cite{Levine15, Gu17, Luo19}. However, a long training time is a key challenge hindering deep RL to scale to even more complex tasks. 
To counter the often excessive training time, RL frameworks aim for two properties: (1) A high throughput which ensures that the framework collects data at very high rates. (2) A high sample efficiency which ensures that the framework learns the desired policy with fewer data. 
To achieve both, synchronous and asynchronous parallel actor-learners have been developed which accelerate RL training~\cite{a3c, a2c, impala, ga3c, seed_rl, ddppo, rlpyt, iswitch}. 

State-of-the-art synchronous methods, such as synchronous advantage actor critic (A2C)~\cite{a3c, a2c} and related algorithms~\cite{ppo, trpo, acktr} 
are popular because of their \emph{data efficiency}, \emph{training stability}, \emph{full determinism}, and \emph{reproducibility}.  
However, synchronous methods suffer from idle time as all actors need to finish experience collection before trainable parameters are updated. This is particularly problematic when the time for an environment step varies significantly. 
As a result, existing synchronous methods don't scale to environments where the step time
varies significantly due to computationally intensive 3D-rendering and (physics) simulation. 

\begin{figure}[t]
\label{fig:flow}
\centering
\includegraphics[width=0.87\linewidth]{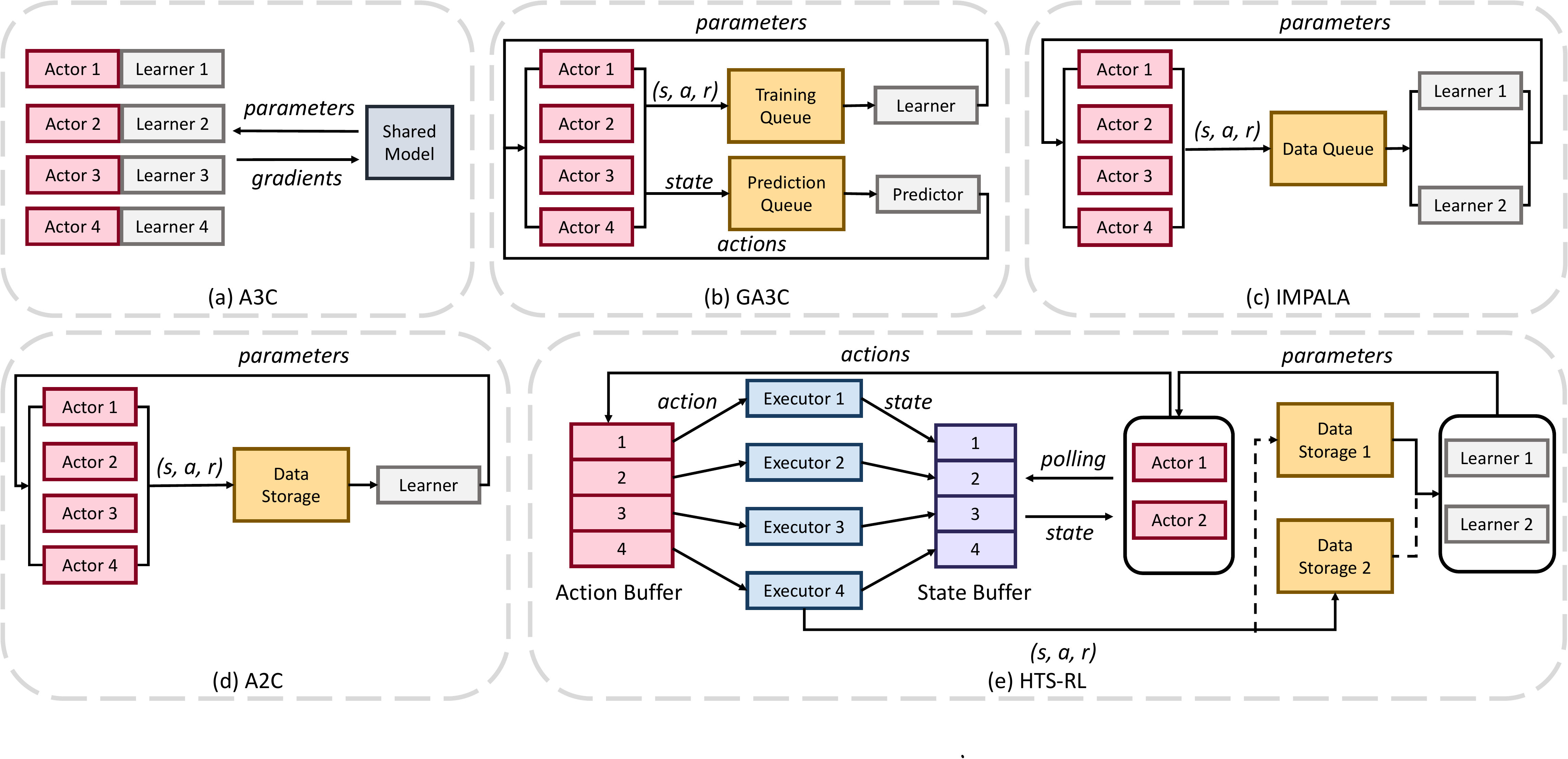}
\vspace{-0.5cm}
\caption{Structure and  flow of (a) A3C, (b) GA3C, (c) IMPALA, (d) A2C  and (e) our HTS-RL.
}

\label{fig:flow}
\end{figure}

\begin{figure}[t]
\centering
\includegraphics[width=1\linewidth]{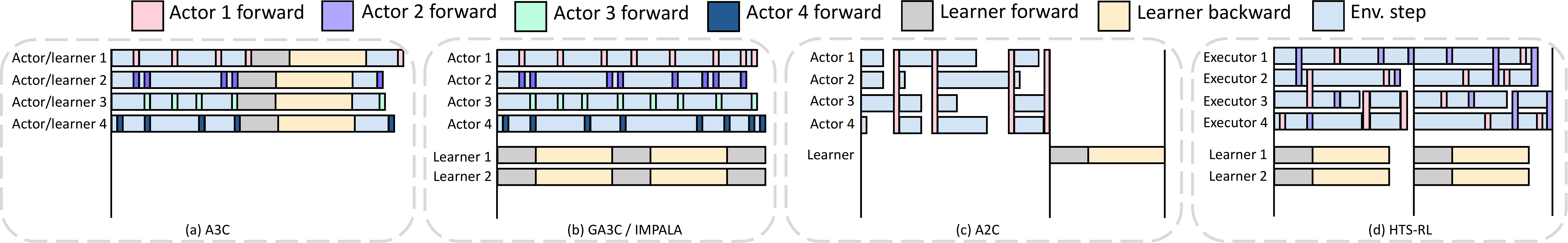}
\vsp{-0.5cm}
\caption{Processing timeline of (a) A3C , (b) GA3C/IMPALA, (c) A2C and (d) our HTS-RL.
}
\label{fig:timeline}
\end{figure}

Instead, asynchronous methods, \eg, asynchronous advantage actor-critic on a GPU (GA3C)~\cite{ga3c} and importance weighted actor-learner architectures (IMPALA)~\cite{impala}, achieve high throughput. However, these methods suffer from 
a stale-policy issue~\cite{ga3c, impala} as learning and data collecting are performed asynchronously. The policy which is used to collect data (behavior policy) is several updates behind the latest policy (target policy) used for parameter updates. This lag leads to noisy gradients and thus lower  data efficiency. Consequently, asynchronous methods trade training stability for  throughput. 
 
We show that this trade-off is not necessary and 
 propose High-Throughput Synchronous RL (HTS-RL), a technique, which achieves both high throughput and high training stability. 
HTS-RL can be applied to many off-the-shelf deep RL algorithms, such as A2C, proximal policy optimization (PPO)~\cite{ppo}, and actor-critic using Kronecker-factored trust region (ACKTR)~\cite{acktr}. HTS-RL is particularly suitable for 
environments with large step time variance: 
It performs concurrent learning and data collection, {which significantly increases throughput.} 
Additionally, actors interact with environment replicas in an asynchronous manner. The asynchronous interaction reduces the idle time.  
Lastly, by performing batched synchronization, HTS-RL ensures that the lag between target and behavior policy equals one. Thanks to this constant latency, HTS-RL uses a `one step delayed-gradient' update which has the same convergence rate as a non-delayed version. As a result, HTS-RL maintains the advantages of synchronous RL, \ie, \emph{data efficiency}, \emph{training stability}, \emph{full determinism}, and \emph{reproducibility}, while achieving  speedups, especially in environments where the step time varies. 

We show that HTS-RL permits to speedup A2C 
and PPO on Atari environments~\cite{gym, atari} and the Google Research Football environment (GFootball)~\cite{gfootball}. Following the evaluation protocol of~\citet{Henderson17} and~\citet{Colas18}, we compare to the state-of-the-art, asynchronous method, 
IMPALA~\cite{impala}: our approach has {$3.7\times$} higher average episode reward on Atari games, and {$43\%$} higher average game scores in the GFootball environment. 
When compared with synchronous A2C and PPO, our approach achieves a $2-6\times$ speedup.  Code is available at~\url{https://github.com/IouJenLiu/HTS-RL}.

\section{Related Work} 
In the following we briefly review work on asynchronous and synchronous reinforcement learning.
\noindent\textbf{Asynchronous reinforcement learning:} 
Asynchronous advantage actor-critic (A3C)~\cite{a3c} is an asynchronous multi-process variant of the advantage actor-critic algorithm~\cite{SuttonRL}. A3C runs on a single machine and does not employ GPUs. As illustrated in \figref{fig:flow}(a),
in A3C, each process is an actor-learner pair which updates the trainable parameters asynchronously. Specifically, in each process the actor collects data by interacting with the environment for a number of steps. The learner uses the collected data to compute the gradient. Then the gradient is applied to a shared model which is accessible by all processes. One major advantage of A3C: throughput increases almost linearly with the number of processes as no synchronization is used. See \figref{fig:timeline}(a) for an illustration of the timing.  

Building upon A3C,~\citet{ga3c} further introduce the GPU/CPU hybrid version GA3C. 
However, GA3C suffers from the stale-policy issue: the behavior policy is  several updates behind the target policy. This discrepancy occurs because GA3C decouples actors and learners as illustrated in \figref{fig:flow}(b) and \figref{fig:timeline}(b), \ie, actors and learners use different processes and perform update and data collection asynchronously.  
Because of this discrepancy, on-policy algorithms like actor-critic become off-policy and the target policy may assign a very small probability to an action from the rollout data generated by the behavior policy. As a result, when updating the parameters, the log-probability may tend to infinity, which makes training unstable~\cite{ga3c, impala}.  To mitigate this, GA3C uses $\epsilon$-correction. 

Importance weighted actor-learner architectures (IMPALA)~\cite{impala} (\figref{fig:flow}(c) and \figref{fig:timeline}(b)) adopt a system similar to GA3C. Hence, IMPALA also suffers from a stale-policy issue, harming its performance. In IMPALA, the behavior policy can be tens or hundreds of updates behind the target policy~\cite{impala}. To correct this stale-policy issue, instead of  $\epsilon$-correction, IMPALA  introduces the importance sampling-based off-policy correction method `V-trace.' 

In contrast to those baselines, we develop High-Throughput Synchronous RL (HTS-RL), which is a fully deterministic synchronous framework. It enjoys the advantages of synchronous methods, \ie, the impact of a stale-policy issue is minimal. More specifically,  the behavior policy is guaranteed to be only one update behind the target policy. Thanks to the small delay, we   avoid a stale-policy. 
As a result, HTS-RL  doesn't require any correction method while being  stable and data efficient.

\noindent\textbf{Synchronous reinforcement learning:} 
Synchronous advantage actor critic (A2C)~\cite{a2c}  
operates in steps as illustrated in \figref{fig:flow}(d) and 
synchronizes all actors at each environment step. As illustrated in \figref{fig:timeline}(c), A2C waits for each actor to finish all its environment steps before performing the next action. A2C has  advantages over its asynchronous counterpart:  (1) It can more effectively use GPUs, particularly when the batch size is large; (2) It is fully deterministic, which guarantees reproducibility and facilitates  development and debugging; (3) It does not suffer from stale policies; and (4) It has a better sample complexity~\cite{acktr}. 

Because of the aforementioned advantages, synchronous methods are popular. In addition to A2C, OpenAI Baselines~\cite{a2c} implement a series of synchronous versions of  RL algorithms, including actor-critic using Kronecker-factored trust region (ACKTR)~\cite{acktr}, trust region policy optimization (TRPO)~\cite{trpo}, proximal policy optimization algorithms (PPO)~\cite{ppo}, and sample efficient actor-critic with experience replay (ACER)~\cite{acer}.  
However, 
one major drawback 
is a lower throughput  compared to their asynchronous counterparts. This issue is exacerbated in environments where the time to perform a step varies.  

To address this concern, we develop High-Throughput Synchronous RL (HTS-RL). HTS-RL performs concurrent learning and data collection, which significantly enhances  throughput. 
Moreover, HTS-RL ensures that actors interact with environment replicas asynchronously while maintaining full determinism. To further reduce actor idle time, HTS-RL performs batched synchronization. Compared to A2C, which synchronizes every step (see \figref{fig:timeline}(c)), HTS-RL synchronizes every $\alpha$ steps. As we  show in~\secref{subsec:analysis}, 
batch synchronization largely reduces the actor idle time, particularly in environments where the time to perform a step varies. 

Decentralized distributed PPO (DD-PPO)~\cite{ddppo} is concurrent work which also advocates for synchronous RL. DD-PPO adopts standard synchronous training with a `preemption threshold,' \ie, once a pre-specified percentage of actors finished rollout,  slower actors are terminated. DD-PPO performs distributed training at a  large scale, \eg, 128 GPUs, and achieves impressive throughput. However, in each machine DD-PPO  follows conventional RL, \ie, alternating between learning and data collection. Also note that the `preemption threshold' introduces non-determinism. In contrast, HTS-RL targets parallel computing in a single machine. We  parallelize learning and data rollout, reduce idle time of actors, and maintain full determinism.  HTS-RL can be combined with  across machine DD-PPO training which we leave to future work.

\section{Background}
\label{sec:back}

We first introduce notation and   Markov Decision Processes (MDPs),  review policy gradient and actor-critic algorithms used by   A2C~\cite{a2c} and IMPALA~\cite{impala}  
and finally discuss the stale-policy issue. 

\noindent\textbf{Reinforcement Learning:}  An agent  interacts with an environment, collecting  rewards over discrete time, as formalized by a Markov Decision Process (MDP). Formally, an MDP ${\cal M}({\cal S, \cal A, \cal T}, r, H, \gamma)$ is defined by a set of states $\cal S$, 
a set of actions $\cal A$, 
a transition function $\cal T : \cS \times \cA  \rightarrow \cS$ which maps an action and the current agent-state to the next state, a reward function $r  : \cS \times \cA  \rightarrow \mathbb{R}$, which maps an action and a state to a scalar, a horizon $H$, and a discount factor $\gamma \in (0, 1]$. Note ${\cal T}$ and $r$ can be  stochastic or deterministic. At each time $t$, the agent selects an action $a_t \in \cal A$ according to a  policy  $\pi$. This  policy $\pi$ maps the current state $s_t \in \cal S$ to a probability distribution over the action space ${\cal A}$. Formally, we refer to the output distribution given a state $s_t$ as $\pi({\cdot|s_t})$, and denote the probability of a specific action $a_t$ given a state $s_t$ via $\pi(a_t|s_t)$. 
After executing the selected action $a_t$, the agent finds itself in state $s_{t+1} = {\cal T}(s_t, a_t) \in \cal S$ and obtains a scalar reward $r_t = {r}(s_t, a_t)$.  The discounted return from a state $s_t$ is 
$
R_t=\sum_{i=0}^{H}\gamma^ir(s_{t+i}, a_{t+i})$. 
The goal is to find a policy $\pi_{\theta}$, parameterized by $\theta$,  which maximizes the expected discounted return 
$
    J(\theta) = \expectation_{s_t \sim \rho^{\pi_{\theta}}, a_t \sim \pi_{\theta}}[R_1]$, 
where $\rho^{\pi_{\theta}}$ is the state visitation distribution
under $\pi_{\theta}$.

\noindent\textbf{Policy Gradient:} 
To maximize the expected discounted return $J(\theta)$ 
\wrt the parameters $\theta$ of the policy we use its gradient~\cite{pg} given by
\begin{equation}
    \nabla_{\theta} J(\theta) = \expectation_{s_t \sim \rho^{\pi_{\theta}}, a_t \sim \pi_{\theta}}[\nabla_{\theta} \log\pi_{\theta}(a_t|s_t)Q^{\pi_{\theta}}(s_t, a_t)].
\end{equation}
Here $Q^{\pi_{\theta}}(s_t, a_t) = \expectation_{s_{i > t} \sim \rho^{\pi_{\theta}}, a_{i > t} \sim \pi_{\theta}}[R_t|s_t, a_t]$ is the expected future return obtained by following policy $\pi_{\theta}$ after action $a_t$ was selected at state $s_t$. Because $Q^{\pi_{\theta}}$ is  unknown, we need to estimate it. For instance, in the REINFORCE algorithm~\cite{pg}, $Q^{\pi_{\theta}}$ is estimated from $R_t$. 
To reduce the variance of this estimate, a baseline is subtracted from $R_t$. A popular choice 
is the value function $V^{\pi_{\theta}} = \expectation_{s_{i > t} \sim \rho^{\pi_{\theta}}, a_{i } \sim \pi_{\theta}}[R_t|s_t]$ which estimates the expected future return by following policy $\pi_\theta$ given that we start in state $s_t$. Generally, $V^{\pi_{\theta}}$ is unknown and therefore often estimated via a function approximator $V^{\pi_{\theta}}_{\phi}$ parameterized by $\phi$.
Combined we obtain the following expression for the gradient of the expected discounted return: 
\begin{equation}
    \nabla_{\theta}J(\theta) = \expectation_{s_t\sim \rho^{\pi_{\theta}}, a_t\sim\pi_{\theta}}[\nabla_{\theta} \log\pi_{\theta}(a_t|s_t)(R_t - V^{\pi_{\theta}}_{\phi}(s_t))].
\end{equation}
Note, it is common to learn the parameters $\phi$ of the value function $V^{\pi_{\theta}}_{\phi}$ by minimizing the squared loss $J_V(\phi) = \expectation_{s_t\sim \rho^{\pi_{\theta}}, a_t\sim\pi_{\theta}}[(R_t - V_{\phi}^{\pi_{\theta}})^2]$.

\noindent\textbf{Actor-Critic Algorithms:} In actor-critic algorithms~\cite{SuttonRL, a3c}, 
the $n$-step truncated return
$
 R^{(n)}_t = \sum_{i=0}^{n-1} \gamma^ir_{t+i} + \gamma^nV^{\pi_{\theta}}_{\phi}(s_{t+k})$ 
is used to estimate $Q^{\pi_{\theta}}$. In addition, to encourage exploration, an entropy term $H$ is sometimes added. The gradient of the expected discounted return \wrt the policy parameters $\theta$ is
\begin{eqnarray}
\label{eq:ac}
  \nabla_{\theta}J(\theta) = 
   \expectation_{s_t\sim \rho^{\pi_{\theta}}, a_t\sim\pi_{\theta}}[\nabla_{\theta} \log\pi_{\theta}(a_t|s_t)(R^{(n)}_t  
   - V^{\pi_{\theta}}_{\phi}(s_t)) + \lambda H(\pi_{\theta}(\cdot|s_t))],
\end{eqnarray}
where  $\lambda \geq 0$ controls the strength of the entropy. 
In practice, the expectation is estimated via 
\begin{eqnarray}
\label{eq:ac_d}
  \nabla_{\theta}\hat{J}(\theta, \cD^{\theta}) =
   \frac{1}{|\cD^{\theta|}}\sum_{(s_t, a_t, r_t) \in \cD^{\theta}}[\nabla_{\theta} \log\pi_{\theta}(a_t|s_t)
   (R^{(n)}_t  - V^{\pi_{\theta}}_{\phi}(s_t))  + \alpha H(\pi_{\theta}(\cdot|s_t))],
\end{eqnarray}
where the set $\cD^{\theta} = \{(s_t, a_t, r_t) \}_t$ subsumes the rollouts collected by following the behavior policy $\pi_{\theta}$ which is identical to the target policy, \ie, the policy specified by the current parameters. 

\noindent\textbf{Stale Policy Issue:} In asynchronous methods like IMPALA and GA3C, the behavior policy  often lags the target policy. 
Specifically, in IMPALA and GA3C, actors send data to a non-blocking data queue, and learners consume data from the queue  as illustrated in \figref{fig:flow}(b, c). The data in the queue is  stale when being consumed by the learner and delays are not deterministic. Consequently, on-policy algorithms like actor-critic become off-policy. 
Formally, suppose the latency is $k$, and we refer to the parameters after the $j$-th update as $\theta_j$. Instead of directly estimating the gradient using samples from the current target policy as indicated in \equref{eq:ac_d}, the  gradient employed in asynchronous methods is
\begin{eqnarray}
  \nabla_{\theta_j}\hat{J}(\theta_j, \cD^{\theta_{j-k}}) = 
   \frac{1}{|\cD^{\theta_{j-k}}|}\hspace{-0.4cm}\sum\limits_{(s_t, a_t, r_t) \in \cD^{\theta_{j-k}}}\hspace{-0.6cm}[\nabla_{\theta_j} \log\pi_{\theta_j}(a_t|s_t) \cdot(R^{(n)}_t  - V^{\pi_{\theta_j}}_{\phi}(s_t))  + \alpha H(\pi_{\theta_j}(\cdot|s_t))].
   \label{eq:stale}
\end{eqnarray}
Note, the gradient \wrt $\theta_j$ is estimated using stale data $\cD^{\theta_{j-k}}$. This is concerning because the current distribution $\pi_{\theta_j}(a_t|s_t)$  often assigns small probabilities to actions that were taken by the stale policy $\pi_{\theta_{j-k}}$. 
As a result, $\nabla_{\theta_j} \log\pi_{\theta_j}(a_t|s_t)$  tends to infinity, which makes training unstable and harms the performance as also reported by~\citet{ga3c} and~\citet{impala}.
\section{High-Throughput Synchronous RL}

\vspace{-0.1cm}
\subsection{Method}
\label{subsec:method}
\vspace{-0.1cm}

We aim for the following four  features: 
(1)  batch synchronization which reduces actor idle time, 
(2) learning and rollout  take place concurrently which increases throughput, 
(3) guaranteed lag of only \textit{one} step between the behavior and target policy which ensures stability of training,
(4) asynchronous interaction between actors and executors at the rollout phase to increase throughput while ensuring determinism. 
In the following we provide an overview 
before analyzing 
each of the four features.

\noindent\textbf{Overview of HTS-RL:} As illustrated in \figref{fig:flow}(e), HTS-RL decomposes RL training into executors, actors, and learners. In addition, we introduce two buffers, \ie, the action buffer and the state buffer,  as well as two data storages, one  for data writing and another one for data reading.

As shown in~\figref{fig:flow}(e), executors \emph{asynchronously} grab actions from the action buffer, which stores actions predicted by actors as well as a pointer to the environment. Then, executors apply the predicted action to the corresponding environment and observe the next state as well as the reward. The executors then store the received state %
and an environment pointer within the state buffer, from which actors grab this information \emph{asynchronously}. Next, the actors use the grabbed states to predict the corresponding subsequent actions and send those together with the environment pointer back to the action buffer. 
The executors also write  received data (state, actions,  rewards) to one of two  storages, \eg, `storage 2.' Meanwhile, the learners read data from the other data storage, \ie, `storage 1.' 

Note, in HTS-RL, learning and data collection take place concurrently. As shown in~\figref{fig:timeline}(d), the executor and actor processes operate at the same time as the `learner' process. 
When executors fill up one of the data storages, \eg, `storage 2,'  learners concurrently consume  the data  from the other, \eg, `storage 1.' Eventually the role of the two data storages switches, \ie, the learners consume the data just collected by the executors into  `storage 2,' and the executors start to fill   `storage 1.'

The system does not switch the role of a data storage until executors  fill up and learners exhaust the data storage. This synchronization design  leads to some idle time for either learner or executor. However, the synchronization is critical to avoid the stale policy issue and to maintain training stability and data efficiency. As shown in~\secref{sec:exp}, even with this synchronization, HTS-RL has competitive throughput but, importantly, much better data efficiency than asynchronous methods like IMPALA. 
Given this design, we now discuss the aforementioned four key features in  detail.

\noindent\textbf{Batch synchronization:} To adapt to environments with large step time variance, we deploy batch synchronization in HTS-RL. Specifically, HTS-RL synchronizes actors and executors every { $\alpha$} steps. This is indicated via the black vertical lines in \figref{fig:timeline}(d) which shows an $\alpha=4$ configuration.

In contrast,  A2C (\figref{fig:timeline}(a)) synchronizes every timestep ($\alpha=1$), \ie, at every time A2C has to wait for the slowest process to finish.  As a result, the throughput of A2C drops significantly in environments with large step time variance. 
For a detailed analysis on how the step time variance and synchronization frequency impact throughput, please see~\secref{subsec:analysis}. 

\noindent\textbf{Concurrent rollout and learning:} To enhance throughput for HTS-RL, we ensure that  rollout and learning  happen concurrently and synchronize every $\alpha$ steps as mentioned before. 
More specifically, after a synchronization, each executor performs $\alpha$ steps and each learner performs one or more forward and backward passes.

Concurrency contrasts HTS-RL from A2C (\figref{fig:timeline}(a)) where rollout and learning take place alternatingly. We hence obtain a  higher throughput. Note, asynchronous methods such as IMPALA (\figref{fig:timeline}(b)) also parallelize learning and rollout. However, the intrinsic stale policy issue hurts their performance even when correction methods are applied. 

\noindent\textbf{Delayed gradient:} Due to the two data storages we ensure that the model from the previous iteration is used to collect the data. Staleness is consequently bounded by one step. Different from asynchronous methods such as IMPALA and GA3C, where the delay between target and behavior policy increases with the number of actors and is sensitive to system configurations~\cite{ga3c}, HTS-RL guarantees that the delay is fixed to one, \ie, the executors are using the parameters of the policy one step prior to synchronization. Please see~\secref{subsec:analysis} for a more detailed analysis. 

To  avoid a stale policy, we   use the behavior policy to compute a `delayed' gradient. We then apply this gradient to the parameters of the target policy.  
Formally, the update rule of HTS-RL is 
\begin{equation}
\theta_{j+1} = \theta_{j} + \eta\nabla_{\theta_{j-1}}\hat{J}(\theta_{j-1}, \cD^{\theta_{j-1}}), 
\end{equation}
where $\eta$ is the learning rate. Note that the gradient $\nabla_{\theta_{j-1}}\hat{J}(\theta_{j-1}, \cD^{\theta_{j-1}})$ is computed at $\theta_{j-1}$ before being added to $\theta_j$ (`one-step-delay'). Because the gradient is computed on $\cD^{\theta_{j-1}}$ \wrt $\theta_{j-1}$, the stale policy issue  described in~\equref{eq:stale} no longer exists. Using classical results and assumptions, the convergence rate of the one-step-delayed gradient has been shown to be $O(\frac{1}{\sqrt{T}})$, where $T$ is the number of performed gradient ascent updates~\citep{langford2009slow}.
Note that the $O(\frac{1}{\sqrt{T}})$ convergence rate is as good as the zero-delayed case. Please see the appendix for more details and assumptions.

\noindent\textbf{Asynchronous actors and executors:} To enhance an actor's GPU utilization, in HTS-RL, there are usually fewer actors than executors. More importantly, the actors and executors interact in an asynchronous manner. Specifically, as shown in~\figref{fig:flow}(e) and~\figref{fig:timeline}(d), actors keep polling an observation buffer for available observations, perform a forward pass for all available observations at once and  send the predicted actions back to the corresponding action buffer.
This asynchronous interaction prevents actors from waiting for the executors, and thus increases throughput. 

However, this asynchrony makes the generated action nondeterministic as actors' sample from a distribution over actions to facilitate exploration. Specifically, with actors grabbing observations asynchronously we can no longer guarantee that a particular actor handles a specific observation. Consequently, even when using pseudo-random numbers within actors, asynchrony will result in nondeterminism. To solve this nondeterminism issue,  we  defer all  randomness  to the executors. In practice, along with each observation, an executor sends a pseudo-random number to the observation buffer. The pseudo-random number serves as the random seed for the actor to perform  sampling on this observation. Because  generation of the pseudo-random numbers in executors does not involve any asynchrony, the predicted actions are deterministic once the random seed for the executor is properly set. Thanks to this, HTS-RL maintains full determinism. 

\begin{figure*}[t]
\begin{minipage}{.56\textwidth}
\setlength{\tabcolsep}{0pt}
\begin{tabular}{ccc}
\includegraphics[width=0.328\linewidth, trim={0.8cm 0.5cm 0.2cm 0.9cm},clip]{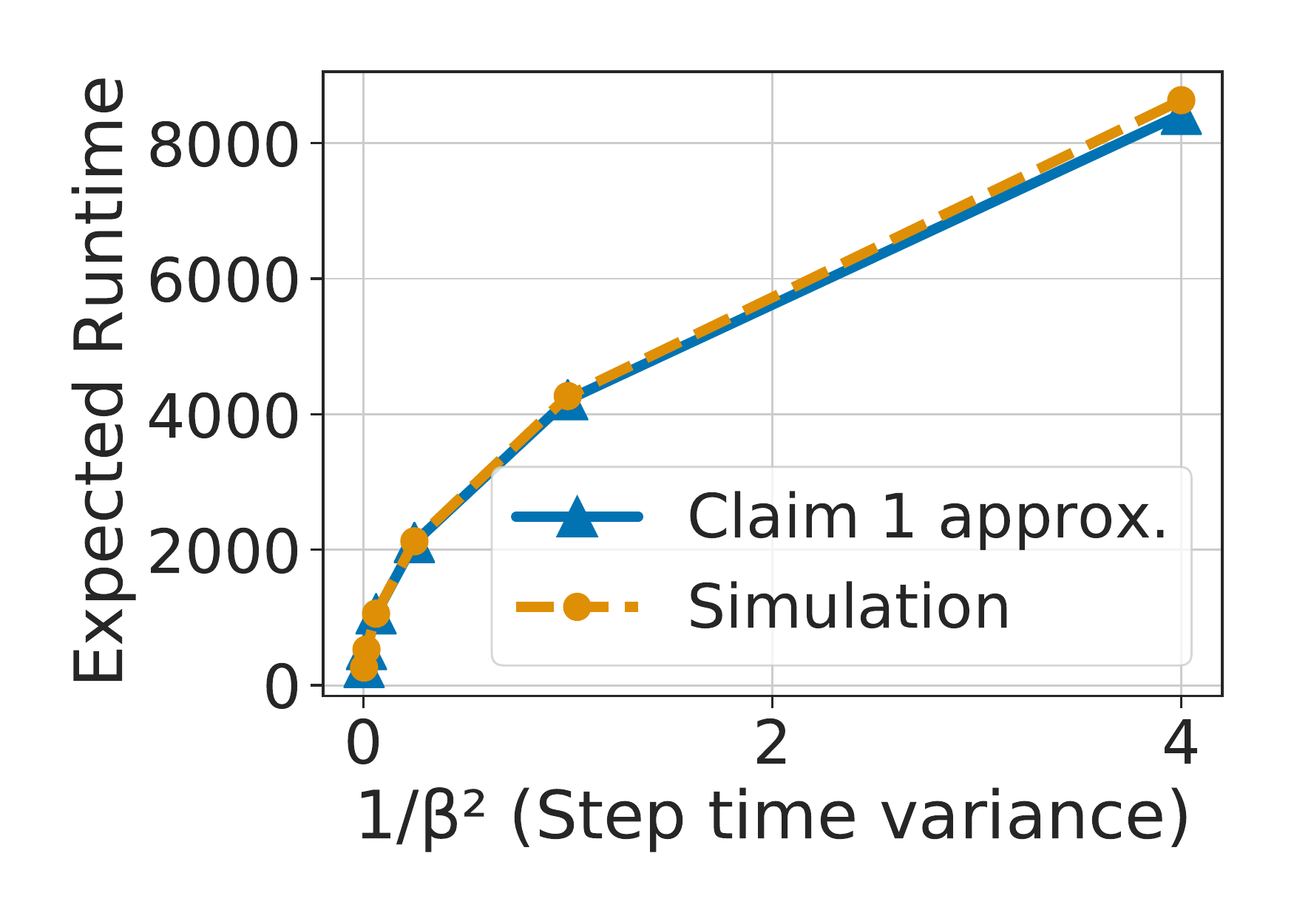} &
\includegraphics[width=0.328\linewidth, trim={0.8cm 0.5cm 0.2cm 0.9cm},clip]{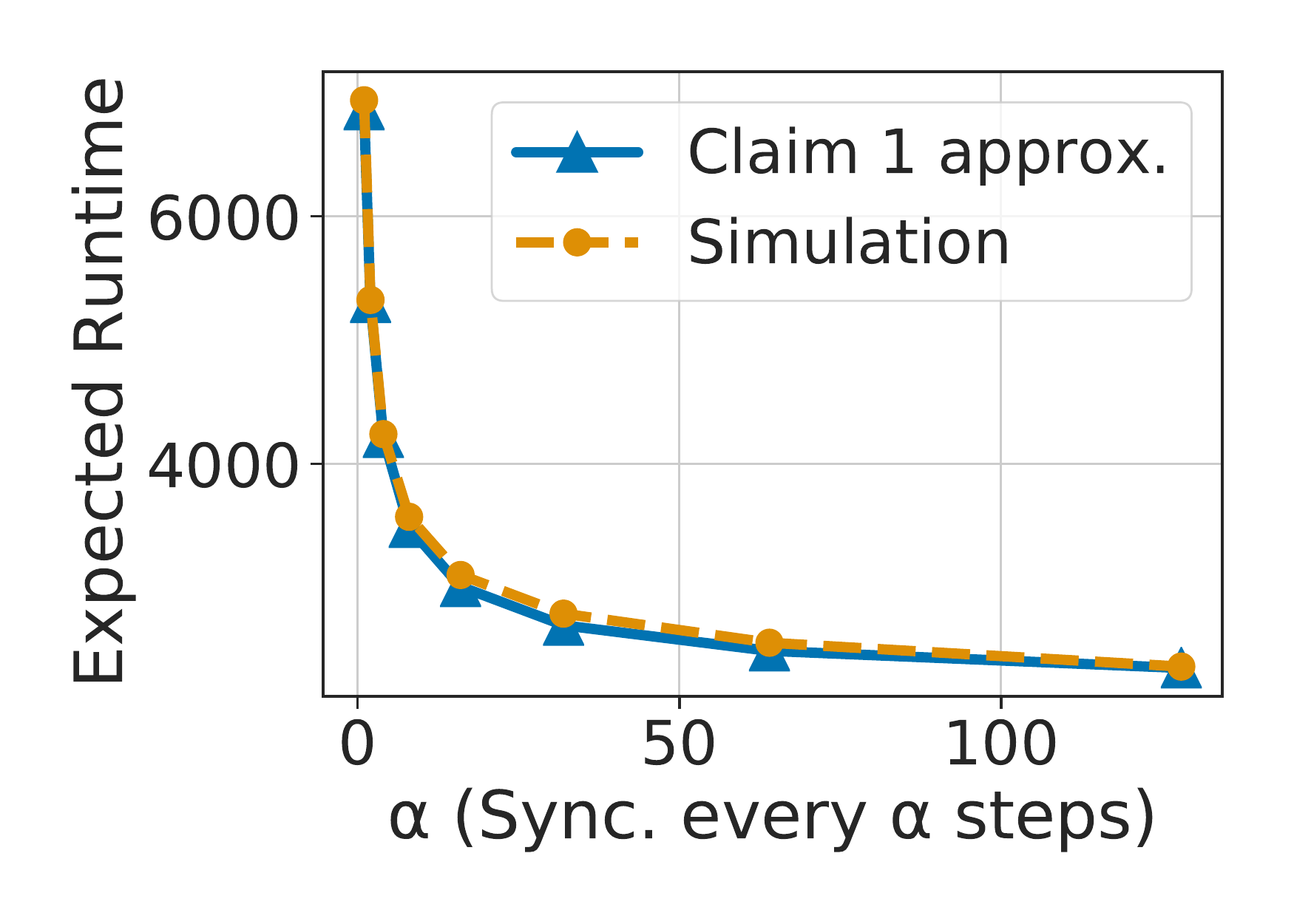}  &
\includegraphics[width=0.328\linewidth, trim={0.8cm 0.5cm 0.2cm 0.7cm},clip]{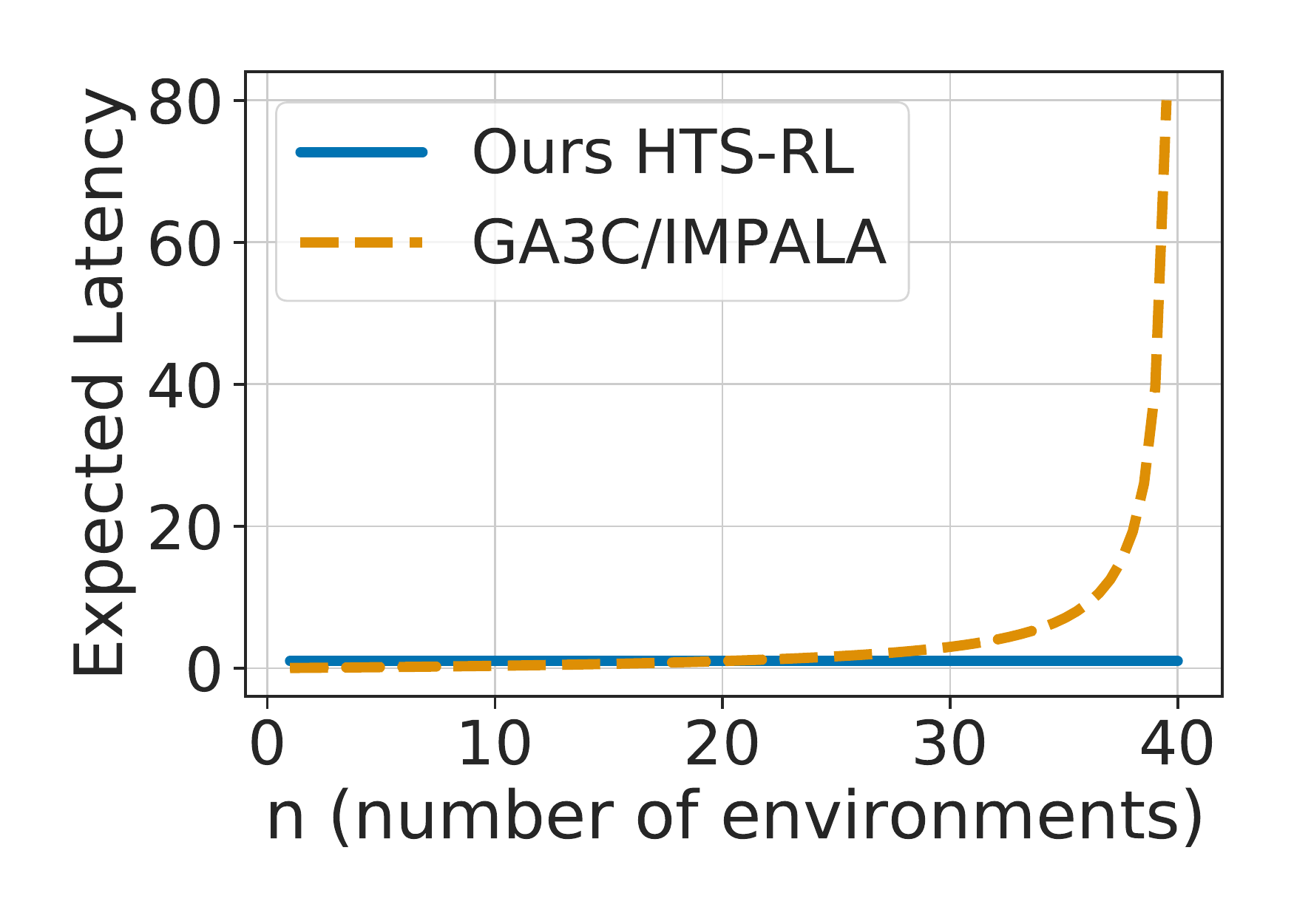} \\
(a) & (b) & (c) \\
\end{tabular}
\vspace{-0.3cm}
\captionsetup{margin=0.4cm}
\caption{
\textbf{(a):}  Expected runtime \vs Environment step time variance {($\frac{1}{\beta^2}$)}. $\alpha$ is fixed at $4$. \textbf{(b):}  Expected runtime \vs synchronization interval ($\alpha$). $\beta$ is fixed at 2. \textbf{(c):} Expected latency between behavior policy and target policy \vs number of environments ($n$).
}
\label{fig:sim}
\end{minipage}
\hspace{.2em}
\begin{minipage}{.43\textwidth}
\vsp{-.4em}
\includegraphics[width=0.47\linewidth, trim={0.6cm 0.5cm 0.2cm 0.},clip]{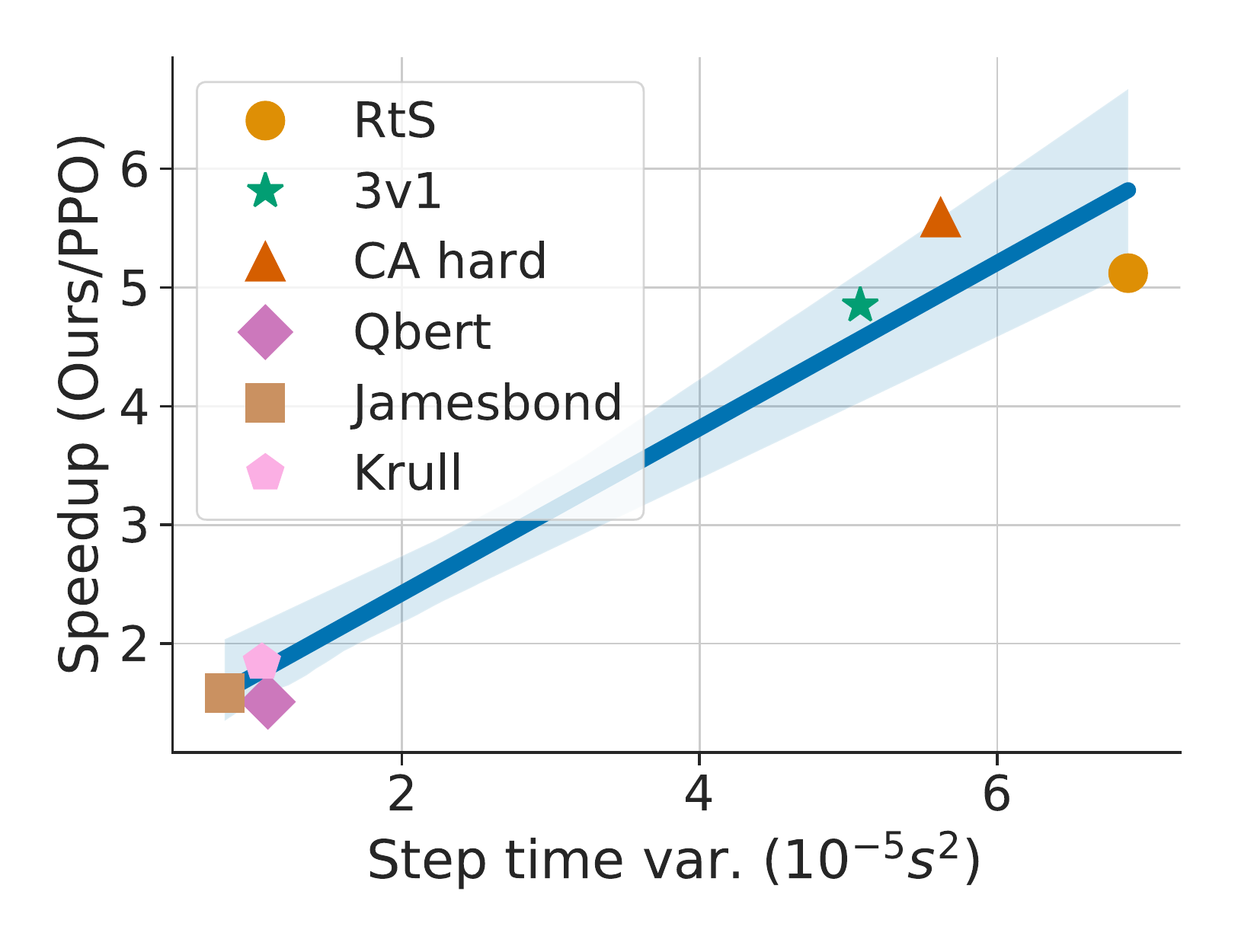}
\includegraphics[width=0.51\linewidth, trim={0.6cm 0.5cm 0.2cm 0.7cm},clip]{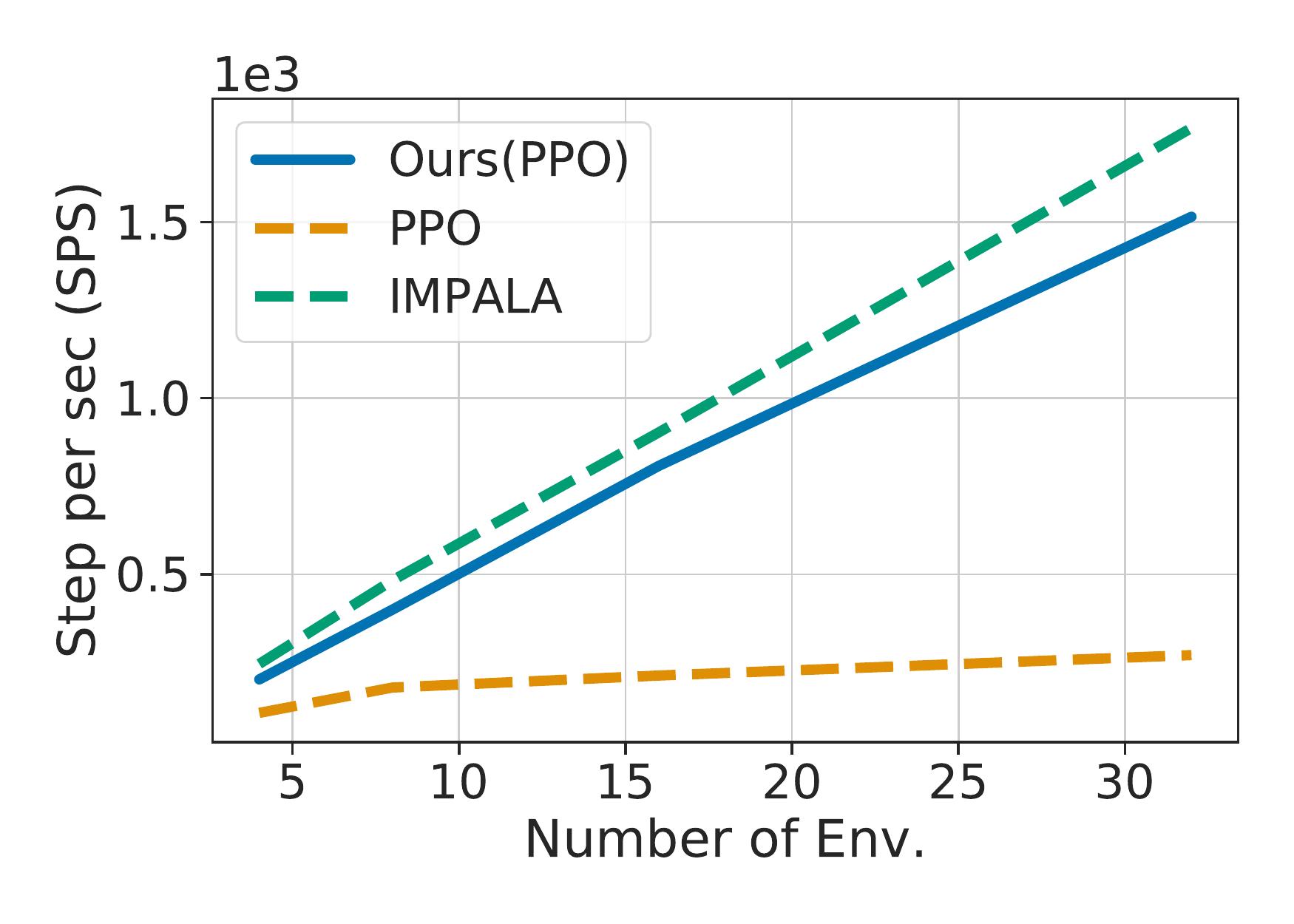}
\vsp{-1em}
\vspace{-0.25cm}
\captionsetup{margin=0.4cm}
\caption{\textbf{Left:}
Speedup \vs env.\ step time variance. RTS: `Run to score', 3v1: `3 \vs 1 with keeper', CA hard: `Conner attack hard'.  \textbf{Right:} Step per second (SPS) \vs  \# of env.\ in GFootball `counterattack hard.'
}
\label{fig:sps}
\end{minipage}
\vsp{-0.5cm}
\end{figure*}

\vspace{-0.1cm}
\subsection{Analysis}
\label{subsec:analysis}
\vspace{-0.1cm}
As discussed in~\secref{subsec:method}, the run-time is impacted by an environment's step time variance and an RL framework's synchronization frequency.
Intuitively, less synchronization and lower step time variance leads to shorter runtime. We will prove this intuition more formally.
To this end we let $T_\text{total}^{n,K}$ denote the run-time when using $n$ environments to collect $K$ states. We can compute the expected run time $\expectation[T^{n, K}_\text{total}]$ as follows.
{\begin{restatable}{claim}{time}{Consider collecting $K$ states using $n$ parallel environments. Let $X_{i}^{(j)}$ denote the time for the $j^{\text{th}}$ environment to perform its $i^{\text{th}}$ step. 
{Suppose  $X_{i}^{(j)}$ is independent and identically distributed (i.i.d.) and $\sum_{i=1}^{\alpha}X_{i}^{(j)}$ follows a Gamma distribution with shape $\alpha$  and rate $\beta$.} Assume the computation time of each actor {consistently takes time $c$.}  
Given these assumptions, the expected time $\expectation[T^{n,K}_{total}]$ to generate  $K$ states approaches \begin{equation}
\frac{K}{n\alpha}(\frac{\gamma}{\beta} (1 + \frac{\alpha - 1}{\beta F^{-1}(1 - \frac{1}{n})} %
) + F^{-1}(1 - \frac{1}{n})) + {\frac{Kc}{n}},
\label{eq:gumbel}
\end{equation} 
where $F^{-1}$ is the inverse cumulative distribution function (inverse CDF) of a gamma distribution with shape $\alpha$ 
and rate $\beta$, \ie, Gamma$(\alpha, \beta)$, and $\gamma$ is the Euler-Mascheroni constant. 
}\label{clm:time}\end{restatable}}
\vsp{-0.3cm}
\begin{proof}
See supplementary material. 
\vsp{-0.3cm}
\end{proof}

We perform a simulation to verify the tightness of the derived expected runtime. 
As shown in \figref{fig:sim}(a,b), {where we assume the step time follows an exponential distribution with rate $\beta$, \ie, $\text{exp}(\beta)$,}
the derived approximation accurately fits the simulation results. {Note that the sum of $\alpha$ i.i.d.\ exponential random variables ($\text{exp}(\beta)$) is a Gamma$(\alpha, \beta)$ random variable. }
\figref{fig:sim}(a) shows that for a fixed synchronization interval ($\alpha=4$), the total runtime increases rapidly if the variance of the environment step time increases. To reduce runtime, we can perform batch synchronization, \ie, $\alpha > 1$. As shown in~\figref{fig:sim}(b) for a rate $\beta=2$, the total runtime decreases when we increase the synchronization interval $\alpha$.
Note, when $\alpha \geq \frac{K}{n}$,~\equref{eq:gumbel} returns the time of an asynchronous system.

While fully asynchronous systems enjoy shorter runtime, the stale policy issue worsens when the number of actors increases.  Formally, we let $L$ denote the latency between behavior policy and target policy in GA3C and IMPALA. We obtain the following result for the expected latency $\expectation[L]$. 

{\begin{restatable}{claim}{queue}{Consider asynchronous parallel actor-learner systems, such as GA3C and IMPALA. Suppose the system has $n$ actors and each actor is sending data to the data queue following an i.i.d.\ Poisson distribution with rate $\lambda_0$. The learners consume data following an exponential distribution with rate $\mu$. Let $L$ denote the latency between the behavior policy and the target policy. Then, we have $\expectation[L] = \frac{n\rho_0}{1 - n\rho_0}$, where utilization $\rho_0 = \frac{\lambda_0}{\mu}$.
}
\label{clm:queue}\end{restatable}}
\vsp{-0.3cm}
\vspace{-0.4cm}
\begin{proof}
See supplementary material. 
\vsp{-0.3cm}
\end{proof}

In the GFootball environment, an actor   generates about $\lambda_0 = 100$ frames per second while a reasonable learner  consumes $\mu = 4000$ frames per second.
 The expected latency for this setting is shown in \figref{fig:sim}(c). 
With few actors, the latency is small. However, this means the learners are mostly idle. Note that  the latency grows rapidly when the number of actors grow. This large {policy lag} leads to unstable training. In contrast, the latency of  HTS-RL is always one regardless of the number of deployed actors. As a result, HTS-RL enjoys stable training and good data efficiency.

\vspace{-0.1cm}
\section{Experiments}
\label{sec:exp}
\vspace{-0.1cm}
\noindent\textbf{Environments.} We evaluate the proposed approach on a subset of the Atari games~\cite{atari, gym} and all 11 academy scenarios of the recently introduced Google Research Football (GFootball)~\cite{gfootball} environment. For the Atari environment we use raw images as agent input. For GFootball, 
the input is either raw images or an `extracted map' which is a simplified spatial representation of the current game state. . 

\noindent\textbf{Experimental Setup.} We compare the proposed approach with IMPALA, A2C, and PPO. For IMPALA, we use the Torch Beast implementation~\cite{torchbeast}, which reports  better performance than the official IMPALA implementation~\cite{impala}. For A2C and PPO, we use the popular synchronous Pytorch implementation of~\citet{pytorchrl}. { For both Atari and GFootball experiments, we use the same models as~\citet{torchbeast} and~\citet{gfootball}. We follow their default hyperparameter settings. 
Please see the supplementary material for details on hyperparameter settings and model architectures.  {Note, we study parallel RL training on a single machine while IMPALA is a distributed training algorithm, which performs RL training across multiple machines. For a fair comparison, we use 16 parallel environments on a single machine for all methods. Importantly, while being downscaled to one machine, the reported IMPALA results match the results reported in the original papers~\cite{torchbeast, gfootball}.} 
Following~\citet{pytorchrl}, all methods are trained for 20M environment steps in the Atari environment. For GFootball, following~\citet{gfootball}, we use 5M  steps. 

\noindent\textbf{Evaluation protocol.} To ensure a rigorous and fair evaluation, 
\citet{Henderson17, Colas18} suggest reporting the `\emph{final metric}.'  The \emph{final metric} is 
an average over the last 100 evaluation episodes, \ie, 10 episodes for each of the last ten policies during training. 
However,~\citet{Henderson17, Colas18} don't consider runtime measurements as a metric. To address this, 
we introduce the `\emph{final time metric}' and the `\emph{required time metric}.' The `\emph{final time metric}' reports the \emph{final  metric} given a limited training time. Training is terminated when the time limit is reached, and the final  metric at this point of time is the `\emph{final time  metric}.' The \emph{required time  metric} reports the runtime required to achieve a desired average evaluation episode reward. The average evaluation reward is the running average of the most recent 100 evaluation episodes, \ie, 10 episodes for each of the most recent ten policies. 
All experiments are repeated for five runs with different random seeds. All the plots in~\figref{fig:plot} include mean of the five runs and the $95\%$ confidence interval obtained by using the Facebook Bootstrapped implementation with 10,000 bootstrap samples. For Atari experiments, we follow the conventional `no-op' procedure, \ie, at the beginning of each evaluation episode, the agents perform up to 30 `no-op' actions. 

\begin{table*}[t]
\centering
\begin{minipage}[t]{.48\textwidth}

\setlength{\tabcolsep}{4pt}
{\footnotesize
\resizebox{\textwidth}{!}{
\begin{tabular}{l|rrr}
\specialrule{.15em}{.05em}{.05em}
         Method &  IMPALA  &      A2C &  Ours (A2C)\\
\hline
\hline
BankHeist&  339 $\pm$ 10 &   775 $\pm$ 166 &     \bf{942 $\pm$ 100}   \\
Beam Rider&  4000 $\pm$ 690 &   4392 $\pm$ 134 &     \bf{6995} $\pm$ 420  \\
Breakout&   201 $\pm$ 133 &    362 $\pm$ 29 &      \bf{413 $\pm$ 37}  \\
Frostbite&    73 $\pm$ 2 &    272 $\pm$ 14&      \bf{315 $\pm$ 12}  \\
Jamesbond&    82 $\pm$ 10 &    438 $\pm$ 59 &      \bf{474 $\pm$ 88}  \\
Krull &  2546 $\pm$ 551 &   7560 $\pm$ 892 &     \bf{7737 $\pm$ 609}  \\
KFMaster&  9516 $\pm$ 3311 &  \bf{30752 $\pm$ 6641} &    30020 $\pm$ 3559  \\
MsPacman&   807 $\pm$ 170 &   1236 $\pm$ 292 &     \bf{1675 $\pm$ 459}  \\
Qbert &  4116 $\pm$ 610 &   12479 $\pm$ 1965 &    \bf{13682 $\pm$ 1873}  \\
Seaquest&   458 $\pm$ 2  &   \bf{1833 $\pm$ 6} &     1831 $\pm$ 7 \\
S. Invader &   \bf{1142} $\pm$ 207&    596 $\pm$ 69 &      {731} $\pm$ 80  \\
Star Gunner&  8560 $\pm$ 918 &  41414 $\pm$ 3826 &    \bf{52666 $\pm$ 5182}  \\

\specialrule{.15em}{.05em}{.05em}
\end{tabular}
}
\vsp{-0.1cm}
\captionsetup{margin=0.3cm}
\caption{{Atari experiment in \emph{final time metrics}: }Average evaluation rewards achieved given limited training time. }
\label{tab:atari_results}
}

\end{minipage}
\raisebox{0.0cm}{
\begin{minipage}[t]{.50\textwidth}
\setlength{\tabcolsep}{4pt}
{\footnotesize
\resizebox{\textwidth}{!}{
\begin{tabular}{l|rrr}
\specialrule{.15em}{.05em}{.05em}
         Method & IMPALA  &      PPO &  \makecell{Ours (PPO)}\\
\hline
\hline
{Empty goal close} & 1.7/2.6&    5.4/15.5 &      \bf{1.0/2.0}  \\
                         
{Empty goal}  & 8.4/11.7&    12.8/19.2 &      \bf{2.0/3.9}  \\
                           
{Run to score}  & 27.0/34.6&    16.2/32.5 &      \bf{6.3/11.4}  \\             
                                       
{RSK}  & 52.3/-&    51.2/68.2 &      \bf{11.5/18.8} \\

{PSK}  & -/-&    70.0/- &      {\bf{38.8}}/-  \\

{RPSK} & 22.3/\bf{25.4}&    45.2/90.8 &      {\bf{13.5}}/27.1  \\
     
{3 \vs 1 w/ keeper}  & -/-&    67.4/144.2 &      \bf{15.9/25.6}\\
    
{Corner}  & -/-&    -/- &      -/- \\
                   
{Counterattack easy} & -/-&    223.2/- &      {\bf{91.3}}/- \\
                        
{Counterattack hard}  & -/-&    383.4/- &      \bf{61.8/-}  \\
{11 vs 11 w/ lazy Opp.}  &58.2/-&    95.8/260.9 &      \bf{14.4/72.1} \\
\specialrule{.15em}{.05em}{.05em}
\end{tabular}
}
\vspace{-0.1cm}
\captionsetup{margin=0.3cm}
\caption{GFootball results in \emph{required time metrics:} required time (minutes) to achieve  scores (time  to achieve score 0.4 / time  to achieve score 0.8). 
}
\label{tab:football_results}
}
\end{minipage}
}
\vsp{-0.1cm}
\end{table*}
\begin{figure*}[t]
\vsp{-0.2cm}
\centering
\begin{tabular}{cccc}
\includegraphics[width=0.25\textwidth]{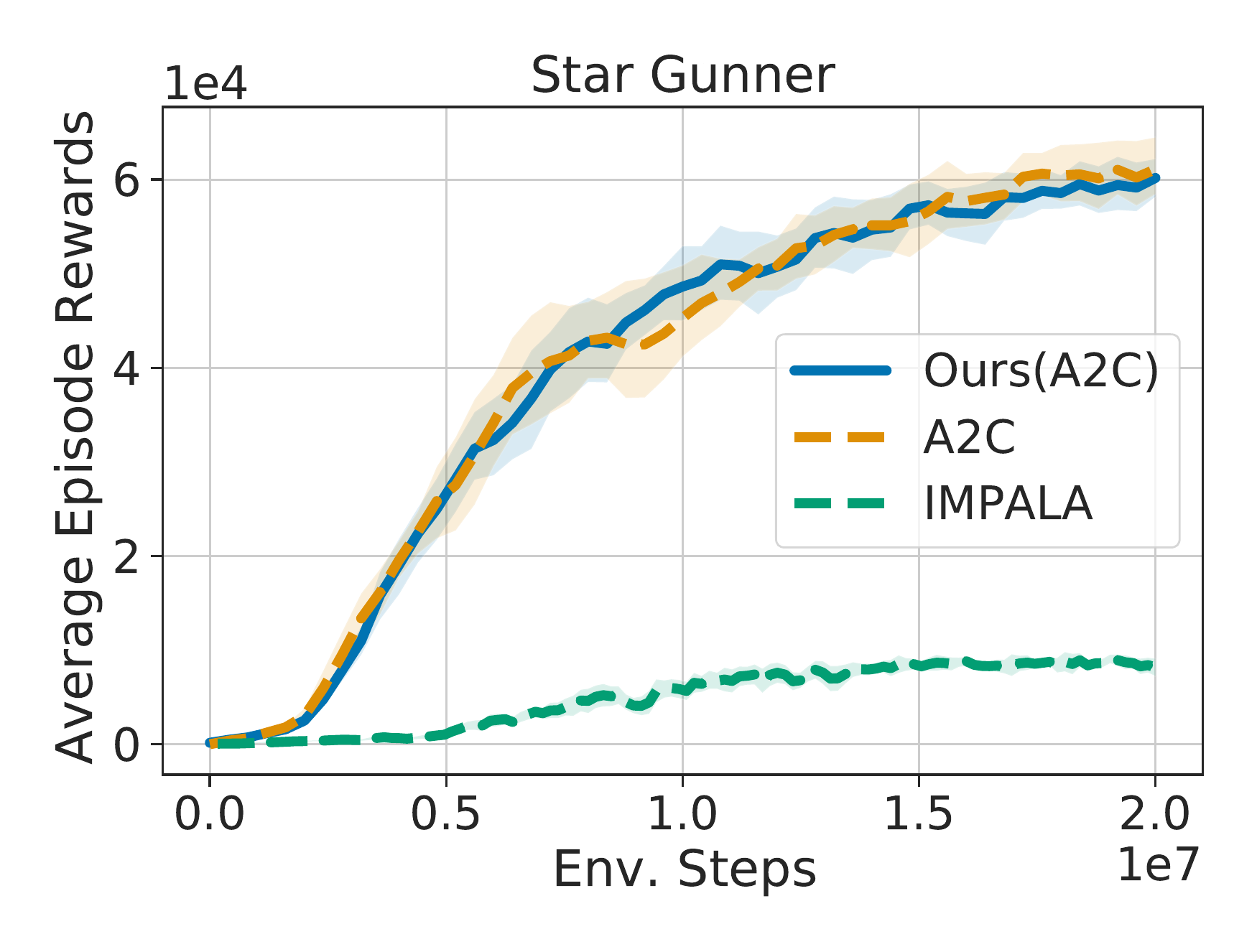}
&
\hsp{-0.6cm}
\includegraphics[width=0.25\textwidth]{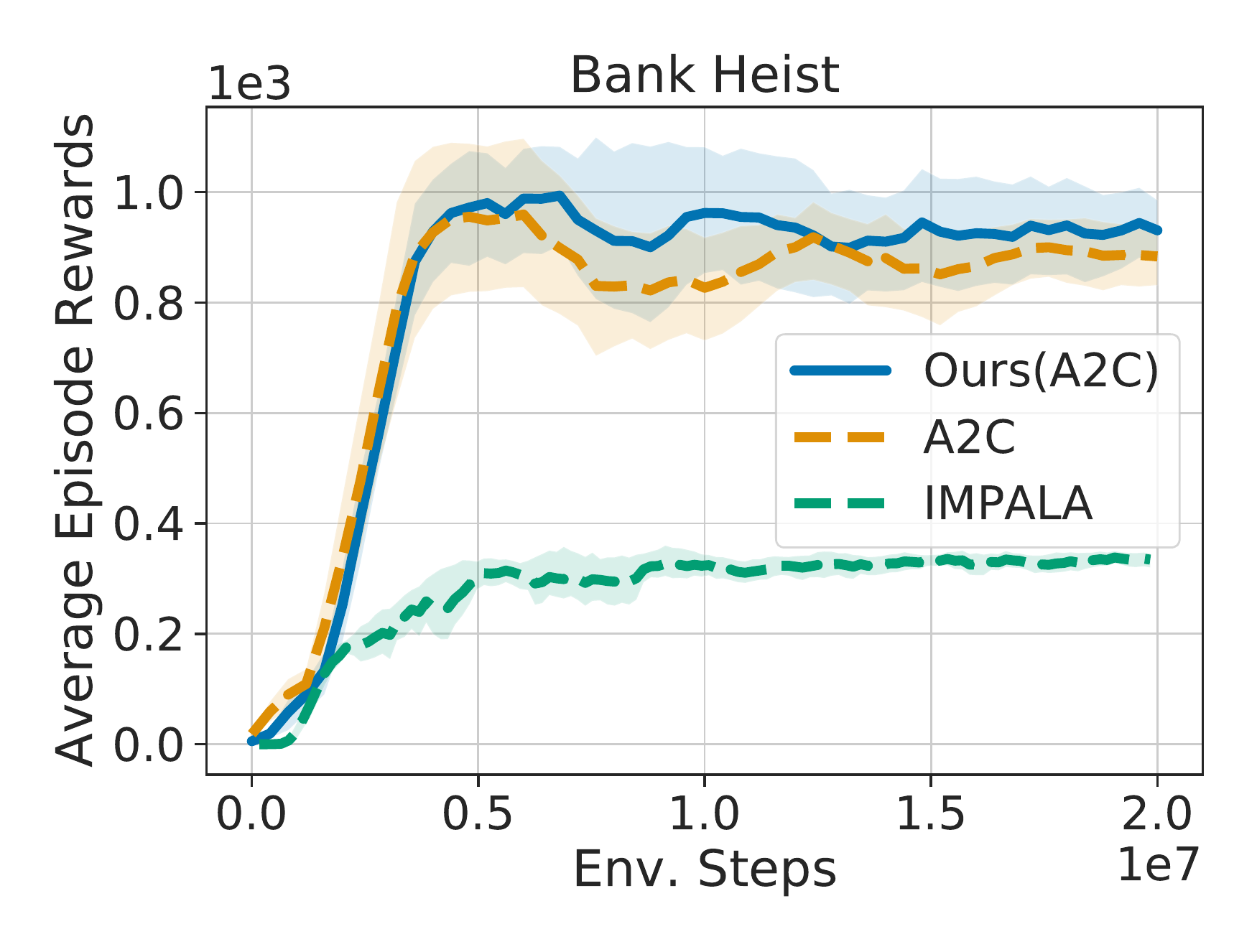}
&
\hsp{-0.6cm}
\includegraphics[width=0.25\textwidth]{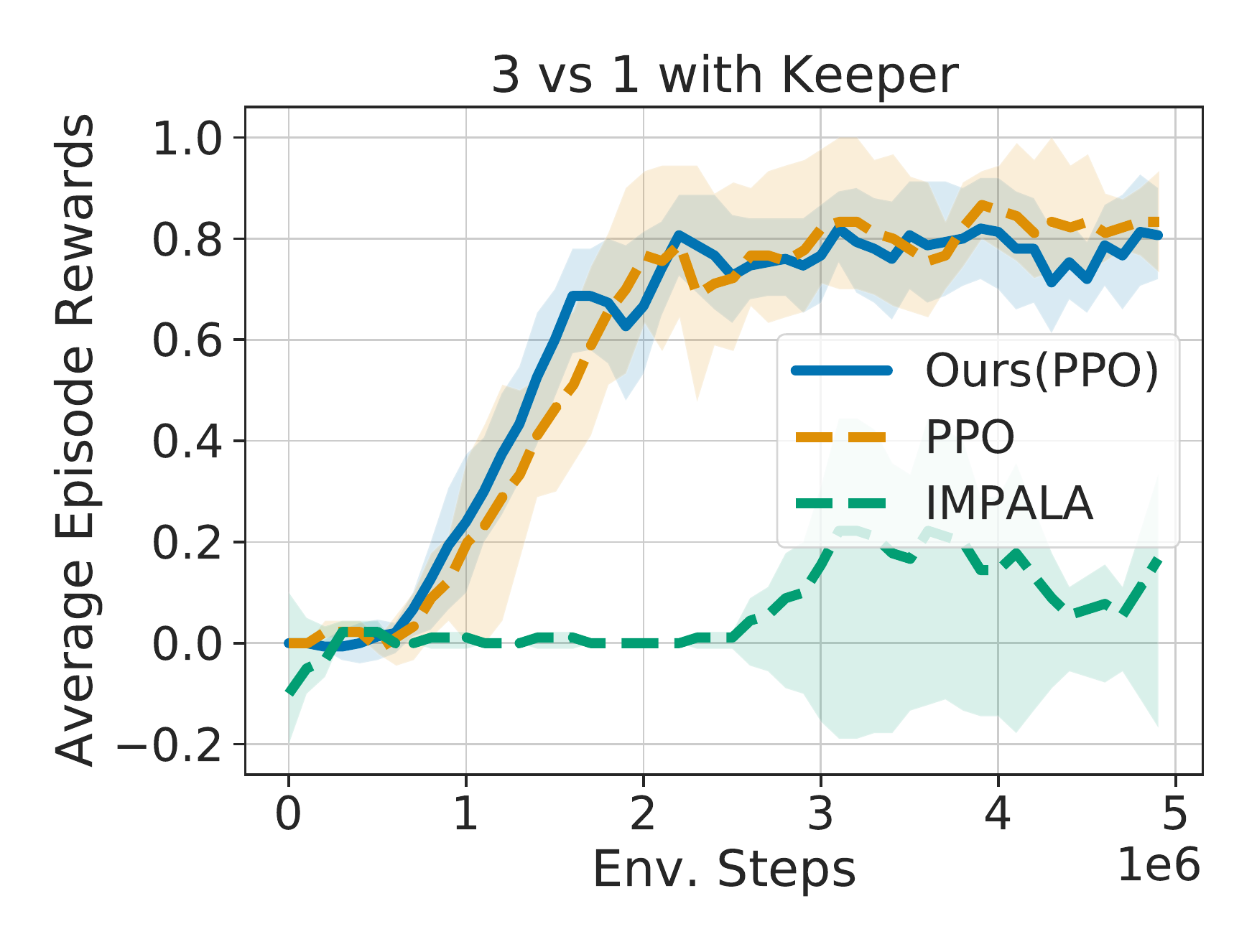}
&
\hsp{-0.6cm}
\includegraphics[width=0.25\textwidth]{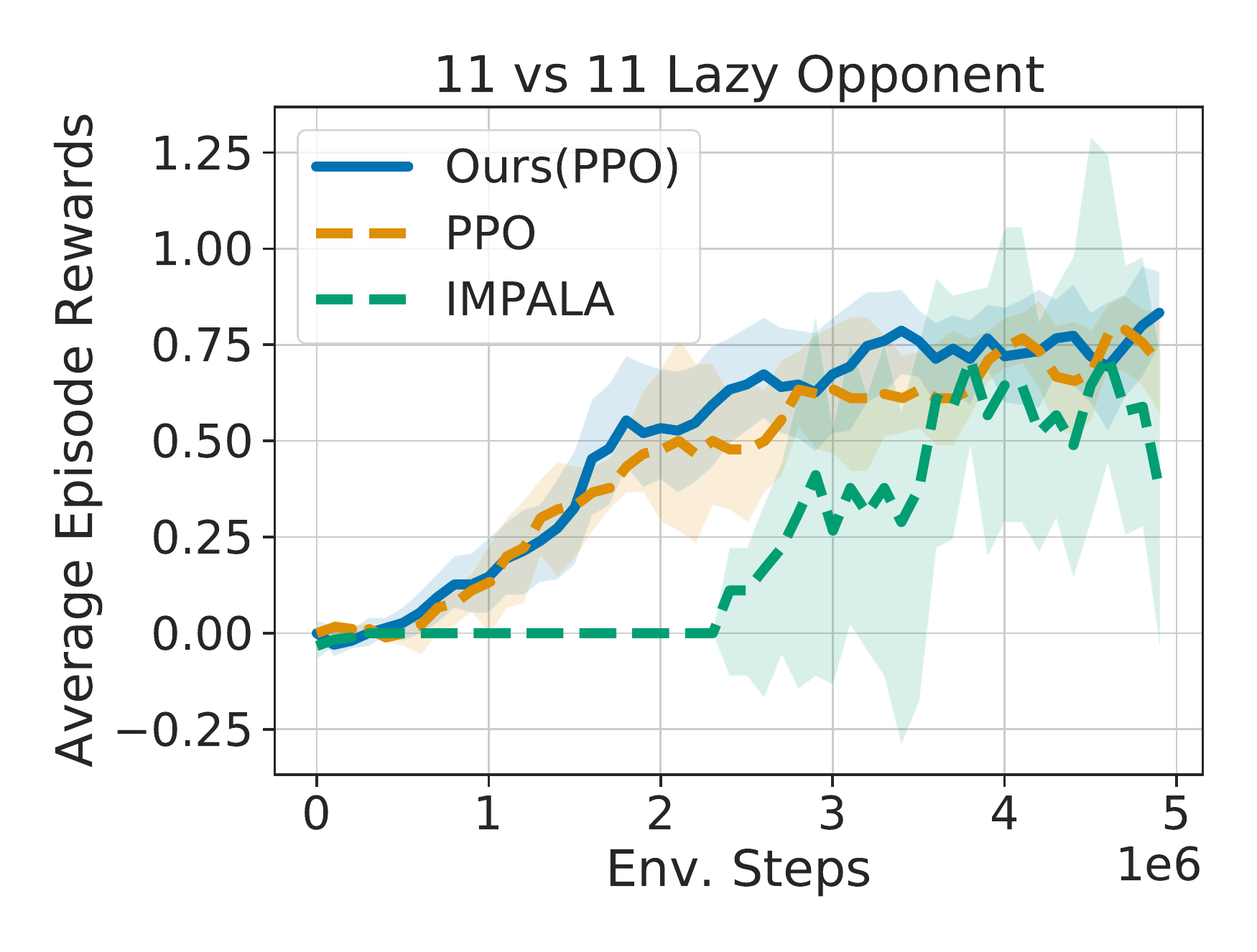}
\\[-0.1cm]
\includegraphics[width=0.25\textwidth]{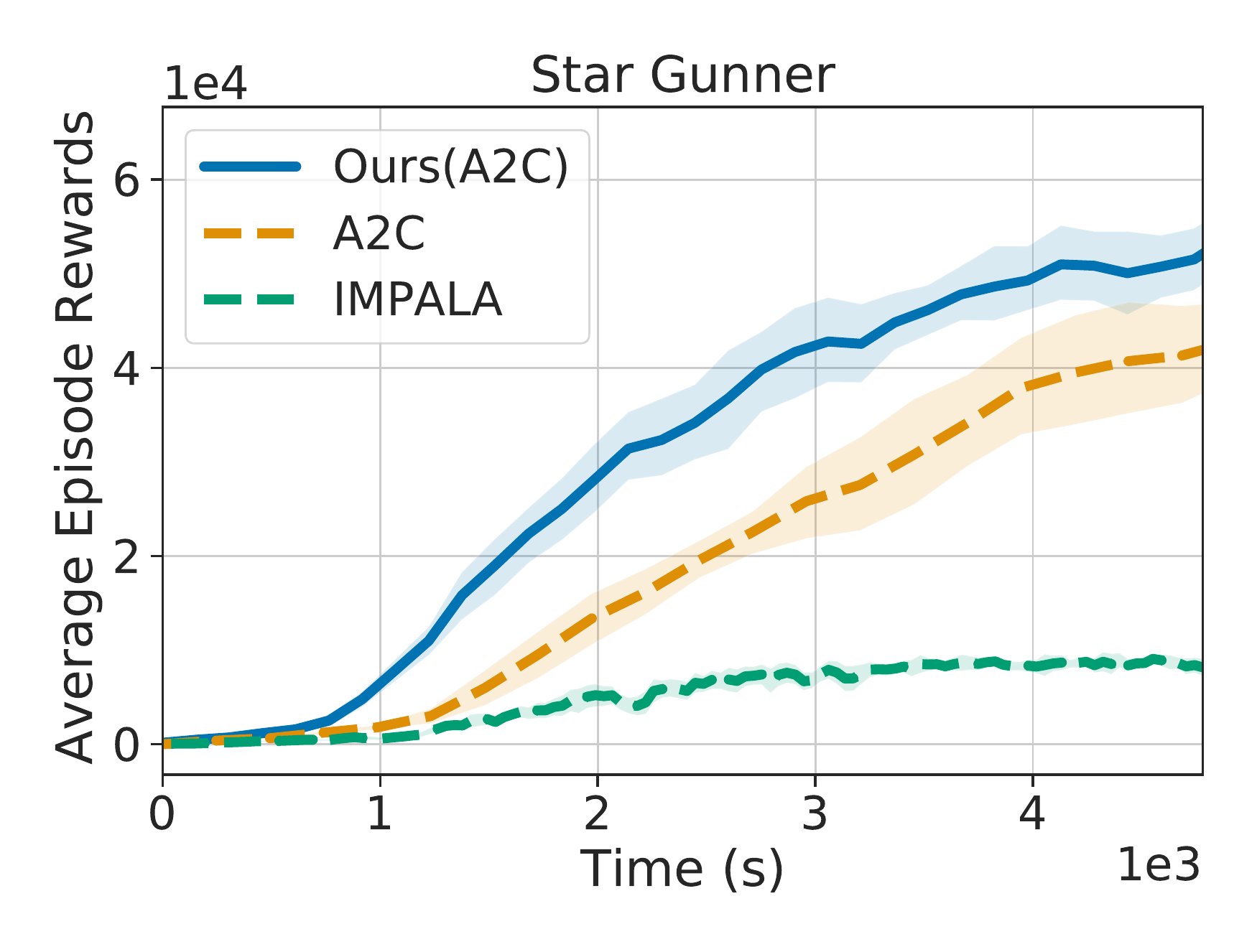}
&
\hsp{-0.6cm}
\includegraphics[width=0.25\textwidth]{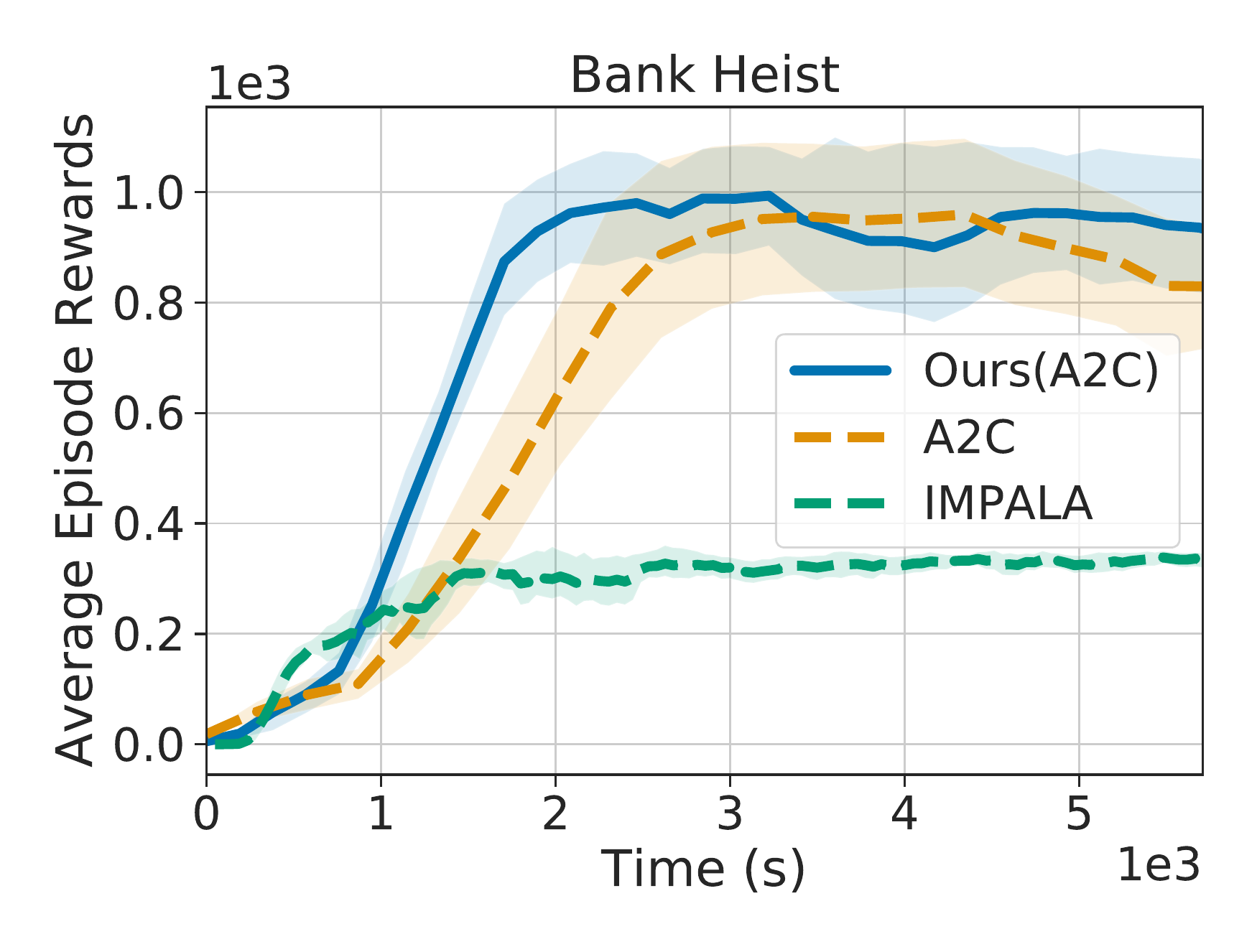}
&
\hsp{-0.6cm}
\includegraphics[width=0.25\textwidth]{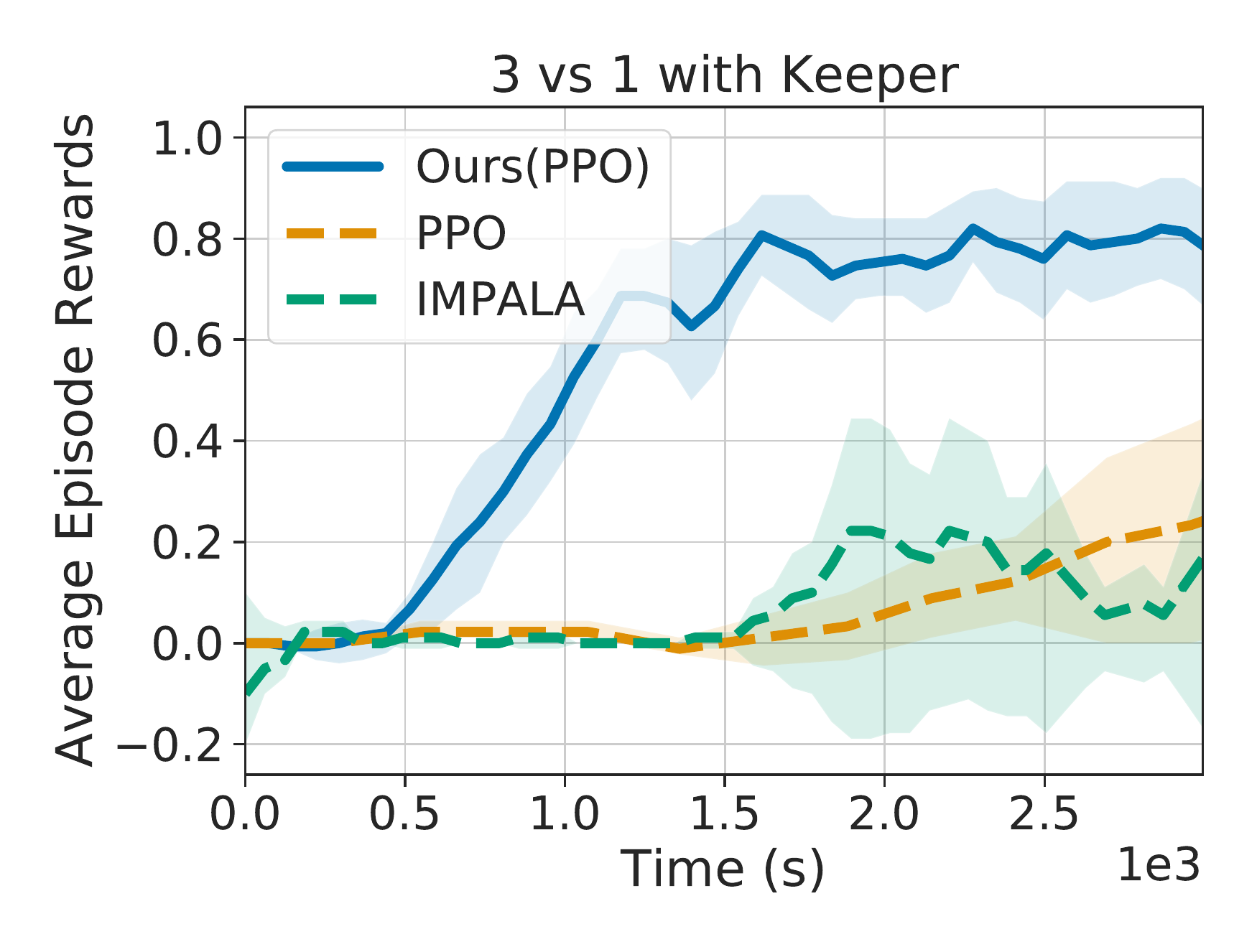}
&
\hsp{-0.6cm}
\includegraphics[width=0.25\textwidth]{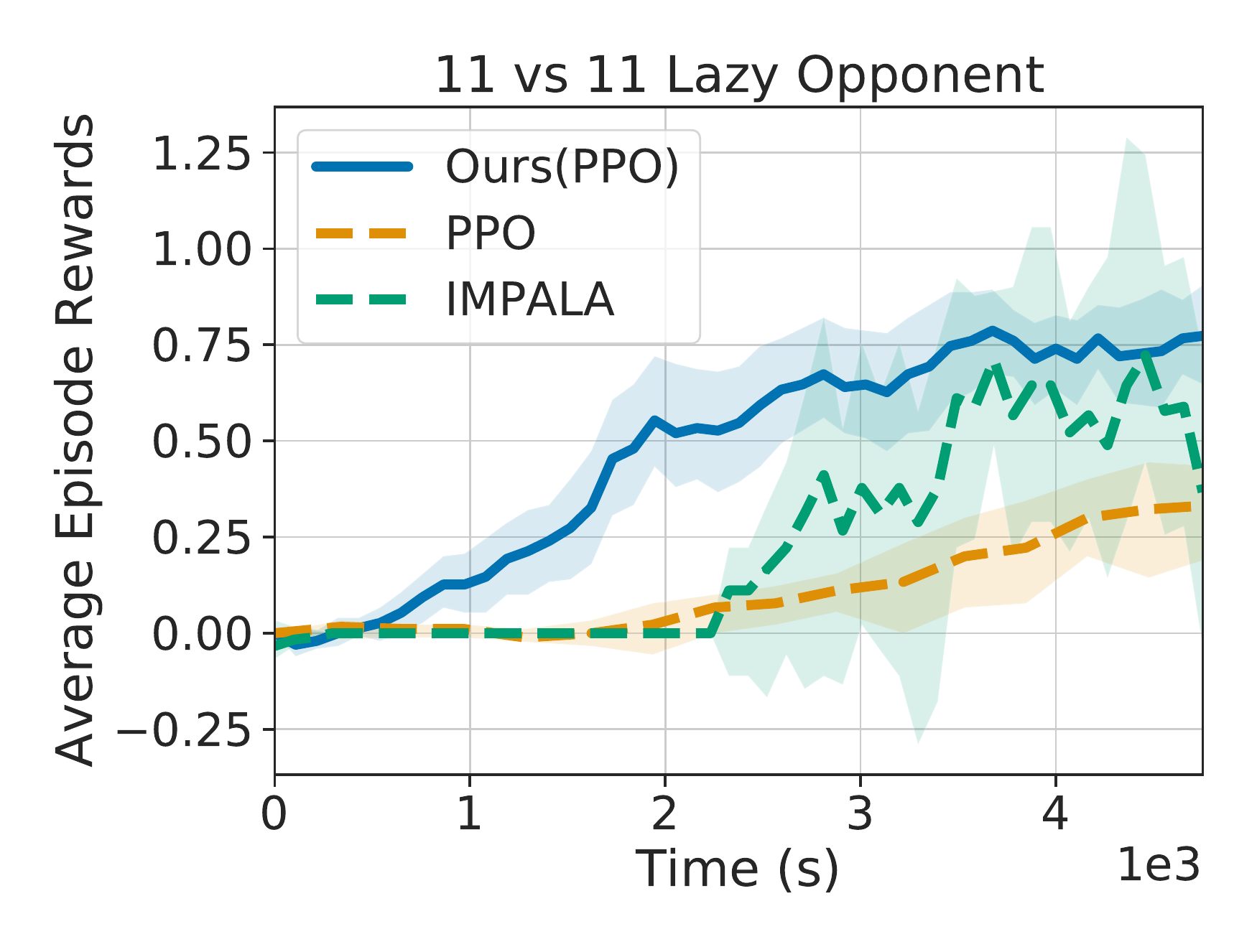}
\end{tabular}
\vsp{-0.2cm}
\caption{Training curves. {\bf{Top row:}} reward \vs environment steps. {\bf{Bottom row:}} reward \vs  time. 
} 
\label{fig:plot}
\vsp{-0.35cm}
\end{figure*}

\noindent\textbf{Results.} To confirm that~\clmref{clm:time} holds 
for complex environments such as Atari and GFootball, we study the speedup ratio of  HTS-RL to  A2C/PPO baselines when the step time variance changes. As shown in~\figref{fig:sps}(left), in environments with small variance, HTS-RL is around $1.5\times$ faster than the baselines. In contrast, in environments with large step time variance, HTS-RL is more than $5\times$ faster. 

To ensure the scalability of HTS-RL, we investigate how the throughput, \ie, the steps per second (SPS), changes when the number of environments increases. For this we conduct experiments on 
the GFootball `counterattack hard' scenario, which has the longest step time and largest step time variance among all GFootball scenarios. As shown in~\figref{fig:sps}(right), the throughput of HTS-RL with PPO increases almost linearly when the number of environments increases. In contrast, for the synchronous PPO baseline, the throughput increase is marginal. 

To study  throughput and data efficiency  on rewards, 
we speed up A2C with HTS-RL,  and compare to baseline A2C and IMPALA on Atari games. \tabref{tab:atari_results} summarizes the  \emph{final time  metric} results. The time limit for all experiments is set to the time when the fastest method, IMPALA, finishes training,  \ie, when it reached the 20M  environment step limit.  As shown in~\tabref{tab:atari_results}, given the same amount of time, HTS-RL with A2C achieves significantly higher rewards than IMPALA and A2C baselines.   

We also speedup PPO with HTS-RL  and compare with baseline PPO and IMPALA on GFootball environments. \tabref{tab:football_results}  summarizes the results using the \emph{required time  metric}. We reported the required time to achieve an average score of $0.4$ and $0.8$. Note that because each episode terminates when one of the teams scores, the maximum possible reward for GFootball scenarios is $1.0$. The results are summarized in \tabref{tab:football_results}, where `-' indicates that the method did not achieve {an average score of $0.4$ and $0.8$} after 5M environment step training. 
As shown in~\tabref{tab:football_results}, our HTS-RL achieves the target score, \ie, $0.4$ and $0.8$ significantly faster than the PPO and IMPALA baselines. Please see the supplementary material for more comparisons, all training curves, and results of final  metric and required time  metric. 
\begin{table}[]

\centering
{\small
\begin{tabular}{lrr}
\specialrule{.15em}{.05em}{.05em}
          &      \makecell{1 Agent \\(Raw Image)} &  \makecell{3 Agents \\(Raw Image)}\\
\hline
\hline
 Ours (PPO) &      0.30$\pm$ 0.11 &      \bf{0.63$\pm$ 0.15}  \\
\specialrule{.15em}{.05em}{.05em}
\end{tabular}

\caption{Average game score of 8M step multi-agent training with raw image input on `3 \vs 1 with keeper' (\emph{final metrics}).
}
\label{tab:img}
}
\end{table}

\figref{fig:plot} shows the training curves for both Atari and Gfootball experiments. The first row of~\figref{fig:plot} shows that HTS-RL achieves similar data efficiency than baseline synchronous methods. In contrast, asynchronous baseline IMPALA has much lower data efficiency.  The second row of~\figref{fig:plot} shows that  HTS-RL achieves higher reward than A2C, PPO and IMPALA in shorter training time.

Thanks to  high throughput and  stability, 
we can demo HTS-RL's potential on more complex RL tasks. Specifically, we scale HTS-RL to training of multiple players in the `3 \vs 1 with keeper' scenario of GFootball from only raw image input. To the best of our knowledge, 
this is the first work to train multiple agents \textit{from only raw image input} in this environment. The results are summarized in~\tabref{tab:img}. Training on three agents achieves higher scores than training on a single agent. Note, we only use one machine with 4 GPUs. 

{
\noindent\textbf{Ablation study.}
\begin{table}[]
\centering
\setlength{\tabcolsep}{4pt}
{\small
\begin{tabular}{l|rrrr}
\specialrule{.15em}{.05em}{.05em}
  Number of Actors       &  1  &   4 & 8 & 16 \\
\hline
\hline

Steps per second (SPS) & 1229 & 1362 & 1393 & 1388  \\
Average scores & 0.82 & 0.82 & 0.82 & 0.82  \\

 \specialrule{.15em}{.05em}{.05em}
\end{tabular}

\caption{Different number of actors on `3 \vs 1 with keeper'. 
Average scores: average score of 100 evaluation episodes.
}
\label{tab:n_actors}
}
\end{table}
\begin{table}[]
\centering
\setlength{\tabcolsep}{4pt}
{\small
\begin{tabular}{l|rrrrrr}
\specialrule{.15em}{.05em}{.05em}
  Synchronization Interval       &  4  &   16 & 64 & 128 & 256 & 512 \\
\hline
\hline

Steps per second (SPS) & 445 & 1070 & 1345 & 1349 & 1360 & 1377  \\
Average scores & 0.82 & 0.81 & 0.81 & 0.82 & 0.81 & 0.83 \\

 \specialrule{.15em}{.05em}{.05em}
\end{tabular}

\caption{Different synchronization interval on `3 \vs 1 with keeper'. 
 Average scores: average score of 100 evaluation episodes.
}
\label{tab:sync_interval}
}
\end{table}
To investigate how the number of actors impacts the performance of HTS-RL, we run HTS-RL with 1, 4, 8, and 16 actors on `3 \vs 1 with keeper'.  The results are summarized in~\tabref{tab:n_actors}, where `average scores' gives the average score of 100 evaluation episodes. The SPS improvement when using more than 4 actors is marginal, because the GFootball game engine process time dominates. That is, when using more than four actors,  most of the time, actors are waiting for response from the GFootball game engine. Note that thanks to full determinism of HTS-RL, different actor numbers have identical final average scores. 

To study the synchronization interval's impact on the performance of HTS-RL, we use synchronization intervals 4, 16, 64, 128, 256, 512 on the `3 \vs 1 with keeper' task. 
As shown in \tabref{tab:sync_interval}, longer synchronization intervals gives higher throughput, which further confirms~\clmref{clm:time} and simulation results in~\figref{fig:sim}(b). Note HTS-RL achieves consistent high average scores while the synchronization interval changes. 
}

\vsp{-0.2cm}
\section{Conclusion}
\vsp{-0.2cm}
We develop High-Throughput Synchronous RL (HTS-RL). It achieves a high throughput while maintaining data efficiency. For this HTS-RL performs batch synchronization, and  concurrent rollout and learning. 
Moreover, HTS-RL avoids the `stale-policy' issue. 

\noindent\textbf{Acknowledgements:} 

This work is supported in part by NSF under Grant $\#$ 1718221, 2008387 and MRI $\#$1725729, NIFA award 2020-67021-32799, UIUC, Samsung, Amazon, 3M, Cisco Systems Inc.\ (Gift Award CG 1377144), and a Google PhD Fellowship to RY. We thank Cisco for access to the Arcetri cluster.

\clearpage
\section*{Broader Impact}
We think artificial intelligence (AI) algorithms can significantly improve people's lives and should be accessible to everyone. 
This paper introduces a high-throughput system which permits efficient deep RL training on a single machine. We believe efficient training of deep RL with limited computational resources is critical to make artificial intelligence more accessible, particularly for people and institutions who can not afford a costly distributed system. 

In the past decade, deep RL achieved impressive results on complex tasks such as GO and 3D video games. However, those achievements rely on a huge amount of computing resources, \eg, AlphaGo~\cite{Silver17} requires 1920 CPUs and 280 GPUs. In contrast, regular personal desktops have only 4 to 16 CPUs and no more than 4 GPUs. As a result, deep RL's huge demand for computational resources makes it a research privilege for large companies and institutions which can afford these resources. 

This paper serves as a step toward making deep RL AI accessible to everyone. With the proposed HTS-RL it is possible to train deep RL agents on 3D video games using a regular desktop with 16 CPUs and 4 GPUs instead of requiring a costly distributed system.

{\small
\bibliographystyle{abbrvnat}
\bibliography{example}

\begin{thebibliography}{35}
\providecommand{\natexlab}[1]{#1}
\providecommand{\url}[1]{\texttt{#1}}
\expandafter\ifx\csname urlstyle\endcsname\relax
  \providecommand{\doi}[1]{doi: #1}\else
  \providecommand{\doi}{doi: \begingroup \urlstyle{rm}\Url}\fi

\bibitem[Babaeizadeh et~al.(2017)Babaeizadeh, Frosio, Tyree, Clemons, and
  Kautz]{ga3c}
M.~Babaeizadeh, I.~Frosio, S.~Tyree, J.~Clemons, and J.~Kautz.
\newblock Reinforcement learning through asynchronous advantage actor-critic on
  a {GPU}.
\newblock In \emph{Proc. ICLR}, 2017.

\bibitem[Bellemare et~al.(2013)Bellemare, Naddaf, Veness, and Bowling]{atari}
M.~G. Bellemare, Y.~Naddaf, J.~Veness, and M.~Bowling.
\newblock The arcade learning environment: An evaluation platform for general
  agents.
\newblock \emph{JAIR}, 2013.

\bibitem[Brockman et~al.(2016)Brockman, Cheung, Pettersson, Schneider,
  Schulman, Tang, and Zaremba]{gym}
G.~Brockman, V.~Cheung, L.~Pettersson, J.~Schneider, J.~Schulman, J.~Tang, and
  W.~Zaremba.
\newblock {OpenAI Gym}, 2016.

\bibitem[Colas et~al.(2018)Colas, Sigaud, and Oudeyer]{Colas18}
C.~Colas, O.~Sigaud, and P.-Y. Oudeyer.
\newblock {GEP-PG}: Decoupling exploration and exploitation in deep
  reinforcement learning algorithms.
\newblock In \emph{Proc. ICML}, 2018.

\bibitem[de~Haan and Ferreira(2006)]{Hann06}
L.~de~Haan and A.~Ferreira.
\newblock \emph{Extreme Value Theory: An Introduction}.
\newblock Springer, 2006.

\bibitem[Dhariwal et~al.(2017)Dhariwal, Hesse, Klimov, Nichol, Plappert,
  Radford, Schulman, Sidor, Wu, and Zhokhov]{a2c}
P.~Dhariwal, C.~Hesse, O.~Klimov, A.~Nichol, M.~Plappert, A.~Radford,
  J.~Schulman, S.~Sidor, Y.~Wu, and P.~Zhokhov.
\newblock {OpenAI} baselines, 2017.

\bibitem[Espeholt et~al.(2018)Espeholt, Soyer, Munos, Simonyan, Mnih, Ward,
  Doron, Firoiu, Harley, Dunning, Legg, and Kavukcuoglu]{impala}
L.~Espeholt, H.~Soyer, R.~Munos, K.~Simonyan, V.~Mnih, T.~Ward, Y.~Doron,
  V.~Firoiu, T.~Harley, I.~Dunning, S.~Legg, and K.~Kavukcuoglu.
\newblock Impala: Scalable distributed deep-rl with importance weighted
  actor-learner architectures.
\newblock In \emph{Proc. ICML}, 2018.

\bibitem[Espeholt et~al.(2020)Espeholt, Marinier, Stanczyk, Wang, and
  Michalski]{seed_rl}
L.~Espeholt, R.~Marinier, P.~Stanczyk, K.~Wang, and M.~Michalski.
\newblock Seed rl: Scalable and efficient deep-rl with accelerated central
  inference.
\newblock In \emph{arXiv.}, 2020.

\bibitem[Fisher and Tippett(1928)]{Fisher28}
R.~A. Fisher and L.~H.~C. Tippett.
\newblock Limiting forms of the frequency distribution of the largest or
  smallest member of a sample.
\newblock \emph{Math. Proc. Cambridge Philos. Soc.}, 1928.

\bibitem[Gu et~al.(2017)Gu, Holly, Lillicrap, and Levine]{Gu17}
S.~Gu, E.~Holly, T.~Lillicrap, and S.~Levine.
\newblock Deep reinforcement learning for robotic manipulation with
  asynchronous off-policy updates.
\newblock In \emph{Proc. ICRA}, 2017.

\bibitem[Henderson et~al.(2017)Henderson, Islam, Bachman, Pineau, Precup, and
  Meger]{Henderson17}
P.~Henderson, R.~Islam, P.~Bachman, J.~Pineau, D.~Precup, and D.~Meger.
\newblock Deep reinforcement learning that matters.
\newblock In \emph{Proc. AAAI}, 2017.

\bibitem[Jain et~al.(2019)Jain, Weihs, Kolve, Rastegari, Lazebnik, Farhadi,
  Schwing, and Kembhavi]{Jain19}
U.~Jain, L.~Weihs, E.~Kolve, M.~Rastegari, S.~Lazebnik, A.~Farhadi, A.~Schwing,
  and A.~Kembhavi.
\newblock Two body problem: Collaborative visual task completion.
\newblock In \emph{Proc. CVPR}, 2019.

\bibitem[Jain$^\ast$ et~al.(2020)Jain$^\ast$, Weihs$^\ast$, Kolve, Farhadi,
  Lazebnik, Kembhavi, and Schwing]{JainECCV2020}
U.~Jain$^\ast$, L.~Weihs$^\ast$, E.~Kolve, A.~Farhadi, S.~Lazebnik,
  A.~Kembhavi, and A.~G. Schwing.
\newblock {A Cordial Sync: Going Beyond Marginal Policies For Multi-Agent
  Embodied Tasks}.
\newblock In \emph{Proc. ECCV}, 2020.
\newblock $^\ast$ equal contribution.

\bibitem[Kostrikov(2018)]{pytorchrl}
I.~Kostrikov.
\newblock Pytorch implementations of reinforcement learning algorithms, 2018.

\bibitem[Kurach et~al.(2019)Kurach, Raichuk, Sta\'nczyk, Zajac, Bachem,
  Espeholt, Riquelme, Vincent, Michalski, Bousquet, and Gelly]{gfootball}
K.~Kurach, A.~Raichuk, P.~Sta\'nczyk, M.~Zajac, O.~Bachem, L.~Espeholt,
  C.~Riquelme, D.~Vincent, M.~Michalski, O.~Bousquet, and S.~Gelly.
\newblock Google research football: A novel reinforcement learning environment.
\newblock \emph{arXiv.}, 2019.

\bibitem[K\"{u}ttler et~al.(2019)K\"{u}ttler, Nardelli, Lavril, Selvatici,
  Sivakumar, Rockt\"{a}schel, and Grefenstette]{torchbeast}
H.~K\"{u}ttler, N.~Nardelli, T.~Lavril, M.~Selvatici, V.~Sivakumar,
  T.~Rockt\"{a}schel, and E.~Grefenstette.
\newblock {TorchBeast: A PyTorch Platform for Distributed RL}.
\newblock \emph{arXiv.}, 2019.

\bibitem[Langford et~al.(2009)Langford, Smola, and Zinkevich]{langford2009slow}
J.~Langford, A.~Smola, and M.~Zinkevich.
\newblock Slow learners are fast.
\newblock \emph{arXiv preprint arXiv:0911.0491}, 2009.

\bibitem[Levine et~al.(2015)Levine, Finn, Darrell, and Abbeel]{Levine15}
S.~Levine, C.~Finn, T.~Darrell, and P.~Abbeel.
\newblock End-to-end training of deep visuomotor policies.
\newblock In \emph{Proc. ICRA}, 2015.

\bibitem[Li et~al.(2019)Li, Liu, Yuan, Chen, Schwing, and Huang]{iswitch}
Y.~Li, I.-J. Liu, Y.~Yuan, D.~Chen, A.~Schwing, and J.~Huang.
\newblock Accelerating distributed reinforcement learning with in-switch
  computing.
\newblock In \emph{Proc. ISCA}, 2019.

\bibitem[Liu et~al.(2019{\natexlab{a}})Liu, Peng, and Schwing]{kf}
I.-J. Liu, J.~Peng, and A.~Schwing.
\newblock Knowledge flow: Improve upon your teachers.
\newblock In \emph{Proc. ICLR}, 2019{\natexlab{a}}.

\bibitem[Liu et~al.(2019{\natexlab{b}})Liu, Yeh, and Schwing]{pic}
I.-J. Liu, R.~A. Yeh, and A.~G. Schwing.
\newblock Pic: Permutation invariant critic for multi-agent deep reinforcement
  learning.
\newblock In \emph{Proc. CoRL}, 2019{\natexlab{b}}.

\bibitem[Luo et~al.(2019)Luo, Solowjow, Wen, Ojea, Agogino, Tamar, and
  Abbeel]{Luo19}
J.~Luo, E.~Solowjow, C.~Wen, J.~A. Ojea, A.~M. Agogino, A.~Tamar, and
  P.~Abbeel.
\newblock Reinforcement learning on variable impedance controller for
  high-precision robotic assembly.
\newblock In \emph{Proc. ICRA}, 2019.

\bibitem[Mnih et~al.(2013)Mnih, Kavukcuoglu, Silver, Graves, Antonoglou,
  Wierstra, and Riedmiller]{dqn1}
V.~Mnih, K.~Kavukcuoglu, D.~Silver, A.~Graves, I.~Antonoglou, D.~Wierstra, and
  M.~Riedmiller.
\newblock Playing atari with deep reinforcement learning.
\newblock In \emph{NeurIPS Deep Learning Workshop}, 2013.

\bibitem[Mnih et~al.(2015)Mnih, Kavukcuoglu, Silver, Rusu, Veness, Bellemare,
  Graves, Riedmiller, Fidjeland, Ostrovski, Petersen, Beattie, Sadik,
  Antonoglou, King, Kumaran, Wierstra, Legg, and Hassabis]{dqn2}
V.~Mnih, K.~Kavukcuoglu, D.~Silver, A.~A. Rusu, J.~Veness, M.~G. Bellemare,
  A.~Graves, M.~Riedmiller, A.~K. Fidjeland, G.~Ostrovski, S.~Petersen,
  C.~Beattie, A.~Sadik, I.~Antonoglou, H.~King, D.~Kumaran, D.~Wierstra,
  S.~Legg, and D.~Hassabis.
\newblock Human-level control through deep reinforcement learning.
\newblock \emph{Nature}, 2015.

\bibitem[Mnih et~al.(2016)Mnih, Adri\`a, Badia, Mirza, Graves, Lillicrap,
  Harley, Silver, and Kavukcuoglu]{a3c}
V.~Mnih, Adri\`a, P.~Badia, M.~Mirza, A.~Graves, T.~P. Lillicrap, T.~Harley,
  D.~Silver, and K.~Kavukcuoglu.
\newblock Asynchronous methods for deep reinforcement learning.
\newblock In \emph{Proc. ICML}, 2016.

\bibitem[Schulman et~al.(2015)Schulman, Levine, Moritz, Jordan, and
  Abbeel]{trpo}
J.~Schulman, S.~Levine, P.~Moritz, M.~I. Jordan, and P.~Abbeel.
\newblock Trust region policy optimization.
\newblock In \emph{Proc. ICML}, 2015.

\bibitem[Schulman et~al.(2017)Schulman, Wolski, Dhariwal, Radford, and
  Klimov]{ppo}
J.~Schulman, F.~Wolski, P.~Dhariwal, A.~Radford, and O.~Klimov.
\newblock Proximal policy optimization algorithms.
\newblock In \emph{arxiv}, 2017.

\bibitem[Silver et~al.(2017)Silver, Schrittwieser, Simonyan, Antonoglou, Huang,
  Guez, Hubert, Baker, Lai, Bolton, Chen, Lillicrap, Hui, Sifre, van~den
  Driessche, Graepel, and Hassabis]{Silver17}
D.~Silver, J.~Schrittwieser, K.~Simonyan, I.~Antonoglou, A.~Huang, A.~Guez,
  T.~Hubert, L.~Baker, M.~Lai, A.~Bolton, Y.~Chen, T.~Lillicrap, F.~Hui,
  L.~Sifre, G.~van~den Driessche, T.~Graepel, and D.~Hassabis.
\newblock Mastering the game of go without human knowledge.
\newblock \emph{Nature}, 2017.

\bibitem[Stooke and Abbeel(2019)]{rlpyt}
A.~Stooke and P.~Abbeel.
\newblock rlpyt: A research code base for deep reinforcement learning in
  pytorch.
\newblock In \emph{arXiv.}, 2019.

\bibitem[Sutton and Barto(2018)]{SuttonRL}
R.~S. Sutton and A.~G. Barto.
\newblock \emph{Reinforcement Learning: An Introduction}.
\newblock The MIT Press, 2018.

\bibitem[Sutton et~al.(2000)Sutton, McAllester, Singh, and Mansour]{pg}
R.~S. Sutton, D.~McAllester, S.~Singh, and Y.~Mansour.
\newblock Policy gradient methods for reinforcement learning with function
  approximation.
\newblock In \emph{Proc. NeurIPS}, 2000.

\bibitem[Vinyals et~al.(2019)Vinyals, Babuschkin, Czarnecki, Mathieu, Dudzik,
  Chung, Choi, Powell, Ewalds, Georgiev, Oh, Horgan, Kroiss, Danihelka, Huang,
  Sifre, Cai, Agapiou, Jaderberg, Vezhnevets, Leblond, Pohlen, Dalibard,
  Budden, Sulsky, Molloy, Paine, Gulcehre, Wang, Pfaff, Wu, Ring, Yogatama,
  W\"{u}nsch, McKinney, Smith, Schaul, Lillicrap, Kavukcuoglu, Hassabis, Apps,
  and Silver]{starcraft2}
O.~Vinyals, I.~Babuschkin, W.~M. Czarnecki, M.~Mathieu, A.~Dudzik, J.~Chung,
  D.~H. Choi, R.~Powell, T.~Ewalds, P.~Georgiev, J.~Oh, D.~Horgan, M.~Kroiss,
  I.~Danihelka, A.~Huang, L.~Sifre, T.~Cai, J.~P. Agapiou, M.~Jaderberg, A.~S.
  Vezhnevets, R.~Leblond, T.~Pohlen, V.~Dalibard, D.~Budden, Y.~Sulsky,
  J.~Molloy, T.~L. Paine, C.~Gulcehre, Z.~Wang, T.~Pfaff, Y.~Wu, R.~Ring,
  D.~Yogatama, D.~W\"{u}nsch, K.~McKinney, O.~Smith, T.~Schaul, T.~Lillicrap,
  K.~Kavukcuoglu, D.~Hassabis, C.~Apps, and D.~Silver.
\newblock Grandmaster level in {StarCraft II} using multi-agent reinforcement
  learning.
\newblock \emph{Nature}, 2019.

\bibitem[Wang et~al.(2017)Wang, Bapst, Heess, Mnih, Munos, Kavukcuoglu, and
  de~Freitas]{acer}
Z.~Wang, V.~Bapst, N.~Heess, V.~Mnih, R.~Munos, K.~Kavukcuoglu, and
  N.~de~Freitas.
\newblock Sample efficient actor-critic with experience replay.
\newblock In \emph{Proc. ICLR}, 2017.

\bibitem[Wijmans et~al.(2020)Wijmans, Kadian, Morcos, Lee, Essa, Parikh, Savva,
  and Batra]{ddppo}
E.~Wijmans, A.~Kadian, A.~Morcos, S.~Lee, I.~Essa, D.~Parikh, M.~Savva, and
  D.~Batra.
\newblock {DD-PPO}: Learning near-perfect pointgoal navigators from 2.5 billion
  frames.
\newblock In \emph{Proc. ICLR}, 2020.

\bibitem[Wu et~al.(2017)Wu, Mansimov, Liao, Grosse, and Ba]{acktr}
Y.~Wu, E.~Mansimov, S.~Liao, R.~Grosse, and J.~Ba.
\newblock Scalable trust-region method for deep reinforcement learning using
  {Kronecker}-factored approximation.
\newblock In \emph{Proc. NeurIPS}, 2017.

\end{thebibliography}
}

\clearpage

\appendix
\renewcommand{\thetable}{A\arabic{table}}
\setcounter{table}{0}
\setcounter{figure}{0}
\setcounter{equation}{0}
\renewcommand{\thetable}{A\arabic{table}}
\renewcommand\thefigure{A\arabic{figure}}
\renewcommand{\theHtable}{A.Tab.\arabic{table}}
\renewcommand{\theHfigure}{A.Abb.\arabic{figure}}
\renewcommand\theequation{A\arabic{equation}}
\renewcommand{\theHequation}{A.Abb.\arabic{equation}}

\newcommand*{\dictchar}[1]{
    \clearpage
    \twocolumn[
    \centerline{\parbox[c][3cm][c]{\textwidth}{
            \centering
            \fontsize{14}{14}
            \selectfont
            {#1}}}]
}

\onecolumn
{\centering \Large \textbf{Appendix: High-Throughput Synchronous Deep RL}}

In this appendix we first provide the proofs for \clmref{clm:time} (\secref{sec:supp_proof_time}) and \clmref{clm:queue} (\secref{sec:supp_proof_queue}). We then discuss delayed gradient updates (\secref{sec:supp_dg}), additional ablation studies (\secref{sec:supp_abla}), comparison with additional baselines(\secref{sec:supp_baseline}), implementation details (\secref{sec:supp_implement}), metrics (\secref{sec:supp_metrics}) and provide all the training curves (\secref{sec:supp_curves}).

\section{Proof of \clmref{clm:time}}
\label{sec:supp_proof_time}

\time*
\begin{proof}[Proof]
As the environments synchronize every $\alpha$ steps, we need $\frac{K}{n\alpha }$ synchronizations to finish the $K$ steps. Let $T_{l}$ denote the time required for the $l^{\text{th}}$ synchronization. We have  
$\expectation[T_\text{total}^{n,K}] = \sum_{l=1}^{\frac{K}{\alpha n}}\expectation[T_l]$. 
Note $\expectation[T_l]  = \expectation[\max_{j} \sum_{i=1}^{\alpha}X_i^{(j)}] {+\alpha c} ~\forall~l$. 
By assumption we know that $Y_j \triangleq \sum_{i=1}^{\alpha}X_i^{(j)}$ follows a gamma distribution with shape $\alpha$ and rate $\beta$. 
By extreme value theory~\cite{Fisher28, Hann06}, suppose $X_l^{(j)} \sim \text{Gamma}(\alpha, \beta)$, then $\expectation[\max_{j} X_l^{(j)}] \simeq \frac{\gamma}{\beta} (1 + \frac{\alpha - 1}{\beta F^{-1}(1 - \frac{1}{n})}
)  + F^{-1}(1 - \frac{1}{n})$, where $F^{-1}=\inf\{x \in {\cal R}: F(x) \geq q\}$. $F(x)$ is the CDF of Gamma$(\alpha, \beta)$, and $\gamma$ is the Euler-Mascheroni constant.  By plugging the obtained approximation into $\sum_{l=1}^{\frac{K}{\alpha n}}\expectation[T_l]$, the result follows.
\end{proof}

\begin{figure}[b]
\centering
\includegraphics[width=0.4\textwidth]{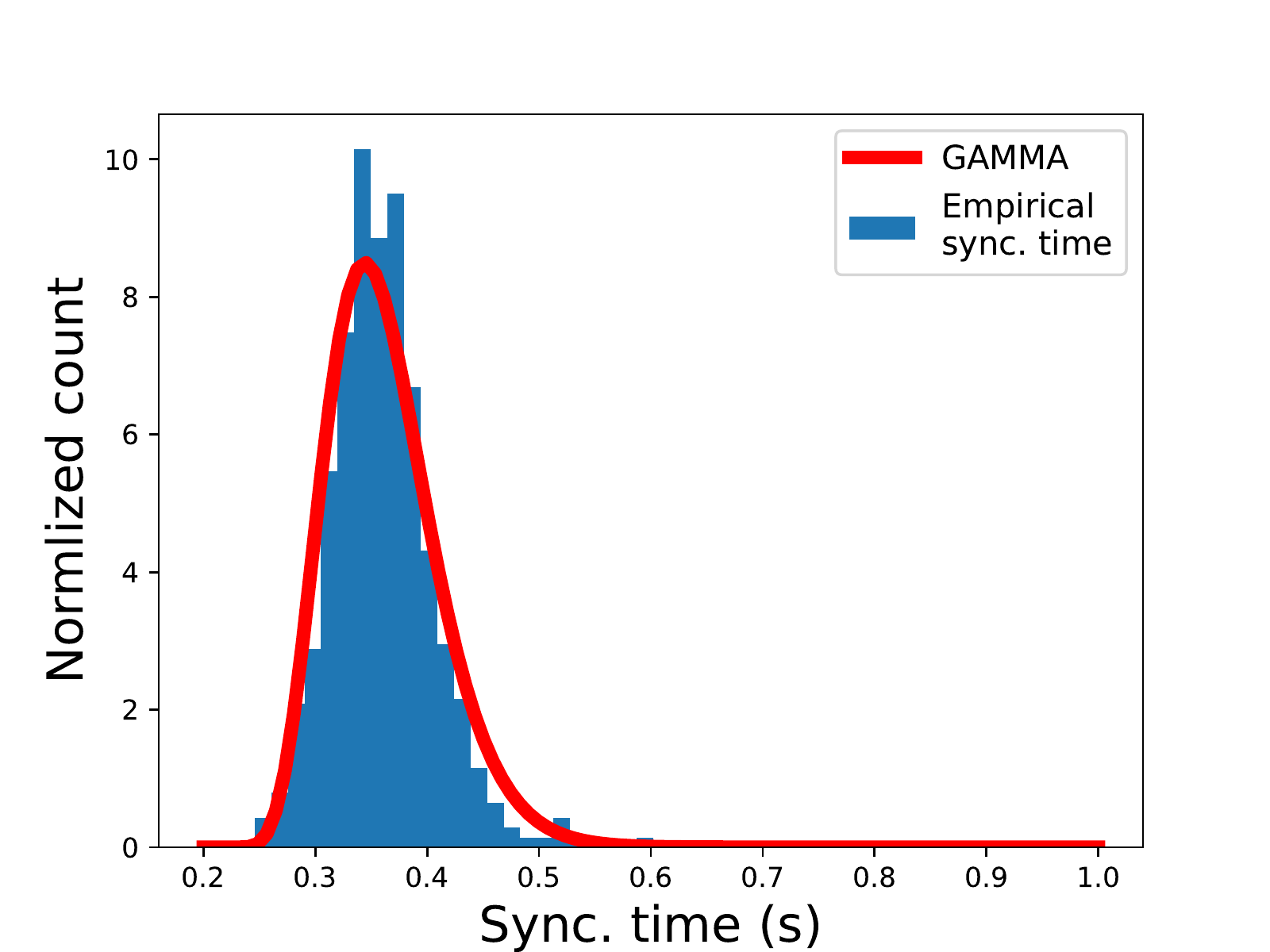}
\caption{Empirical synchronization time.}
\label{fig:exp_plot}
\end{figure}

{In~\clmref{clm:time}, we assume the sum of steptimes (synchronization time) to follow a Gamma distribution. We empirically verify this assumption. In~\figref{fig:exp_plot}, we show the histogram of synchronization time (sum of every 100 step times) on `3 \vs 1 w/ keeper' Furthermore, we perform a Kolmogorov-Smirnov goodness-of-fit test, with a significance-level of 0.05 and D-statistics of 0.04. We find  the empirical data is consistent with the assumed Gamma distribution.
}

\section{Proof of \clmref{clm:queue}}
\label{sec:supp_proof_queue}
\queue*
\begin{proof}
Observe that the latency $L$ is equal to the length of the data queue. $n$ actors send data to the queue with rate $n\lambda_0$ in total. Let $P_i$ denote the probability that there are $i$ data points in the queue. To be stable, the system must satisfy the balance equations
\begin{eqnarray}
n\lambda_0P_0 &=& \mu P_1 \label{eq:p0}\\
(n\lambda_0 + \mu)P_j &=& n\lambda_0P_{j-1} + \mu P_{j+1}, \quad\quad j \geq 1\label{eq:pj}.
\end{eqnarray}
Note~\equref{eq:p0} and~\equref{eq:pj} reduce to 
\begin{equation}
(n\lambda_0)P_j = \mu P_{j+1}, j \geq 0,
\end{equation}
or $P_{j+1} = n\rho_0P_j, j \geq 0$ from which we recursively obtain 
\begin{equation}
P_j = (n\rho_0)^jP_0.
\end{equation}
Using the fact that $1 = \sum_{j=0}^{\infty}P_j = P_0\sum_{j=0}^{\infty}(n\rho_0)^j$, we observed that there is a solution if and only if $n\rho_0 < 1$, in which case $1=P_0(1-n\rho_0)^{-1}$, or
\begin{equation}
P_0 = 1 - n\rho_0.
\end{equation}
Therefore, we have 
\begin{equation}
P_j = (n\rho_0)^j(1 - n\rho_0),
\end{equation}
which follows a geometric distribution with success probability $(1 - n\rho_0)$.  Therefore, we have $\expectation[L] = \frac{n\rho_0}{1 - n\rho_0}$ which concludes the proof.
\end{proof}

\section{Delayed Gradient}
\label{sec:supp_dg}
A delayed stochastic gradient descent performs the following update:
$\theta_t = \theta_{t-1} - \alpha_t \nabla \ell(x_{t-\tau}; \theta_{t-\tau})$. The algorithm is identical to the standard stochastic gradient descent, except that gradients are delayed by a time step of $\tau$. 

Consider a loss function of the form
\begin{equation}
\cL(\theta) \triangleq \sum_{t=1}^T \ell(x_t; \theta).
\end{equation}
We are interested in analyzing the convergence rate of $\theta$ to the optimal parameters $\theta^* \triangleq \arg\min_\theta \cL(x_t; \theta)$. 
Following~\citet{langford2009slow}, we assume the following:
(a) $\ell$ is convex, (b) $L$-Lipschitz, \ie, $\norm{\nabla(\ell(x, \theta))} \leq L$, (c) $x_t$ is drawn i.i.d. following a uniform distribution from a finite set $X$, (d) $\max_{x, x' \in X} \frac{1}{2}\|x-x'\|^2 \leq F^2$, where $F$ is a constant, and (e) the learning rate of the delayed stochastic gradient descent is $\frac{\sigma}{\sqrt{t-\tau}}$, where $\sigma^2 = \frac{F^2}{2 \tau L^2}$,
then
\begin{equation}\sum_{t=1}^T \ell(x_t, \theta_t) - \ell(x_t, \theta^*) \leq 4FL\sqrt{\tau T}.
\end{equation}

Dividing both sides by $T$, we have
\begin{equation}
\frac{1}{T}\sum_{t=1}^T \ell(x_t, \theta_t) - \ell(x_t, \theta^*) \leq 4FL\sqrt{\frac{\tau}{T}}.
\end{equation}
Stated differently, the convergence rate is $O(\sqrt{\frac{\tau}{T}})$. For HTS-RL with on-policy RL algorithms, the delay is guaranteed to be one, \ie, $\tau = 1$. Therefore, the convergence rate is $O(\sqrt{\frac{1}{T}})$. {Note that in practice the aforementioned assumptions are typically not met due to the use of deep nets.}

{
\section{Ablation Study}
\label{sec:supp_abla}
\
\begin{table}[]
\centering
\setlength{\tabcolsep}{4pt}
{\small
\begin{tabular}{l|ccc}
\specialrule{.15em}{.05em}{.05em}
       &  Our Delayed Gradient  &   Truncated I.S. & No Correction \\
\hline
\hline

BankHeist & {\bf 987} & 881 & 877 \\
Breakout & {\bf 415} & 390 & 402 \\
Seaquest & {\bf 1831} & 1827& 1784 \\

 \specialrule{.15em}{.05em}{.05em}
\end{tabular}

\caption{Average episode rewards of our delayed-gradient, truncated importance sampling, and no correction on  Atari games.
}
\label{tab:delayed_grad}
}
\end{table}

\subsection{Delayed Gradient}
In addition to the convergence rate bound of delayed gradient, we verify the effectiveness of delayed gradient empirically. We run HTS-RL with (1) delayed gradient, (2) truncated importance sampling, and (3) no correction on multiple Atari games. The results are summarized in~\tabref{tab:delayed_grad}, where the average rewards of 100 evaluation episodes are reported. Compared with truncated  importance sampling and no correction, the one-step delayed gradient in HTS-RL achieves a higher reward, which underlines the suitability of the delayed gradient strategy. 
} 

\begin{figure}[th]
\centering

\includegraphics[width=0.4\textwidth]{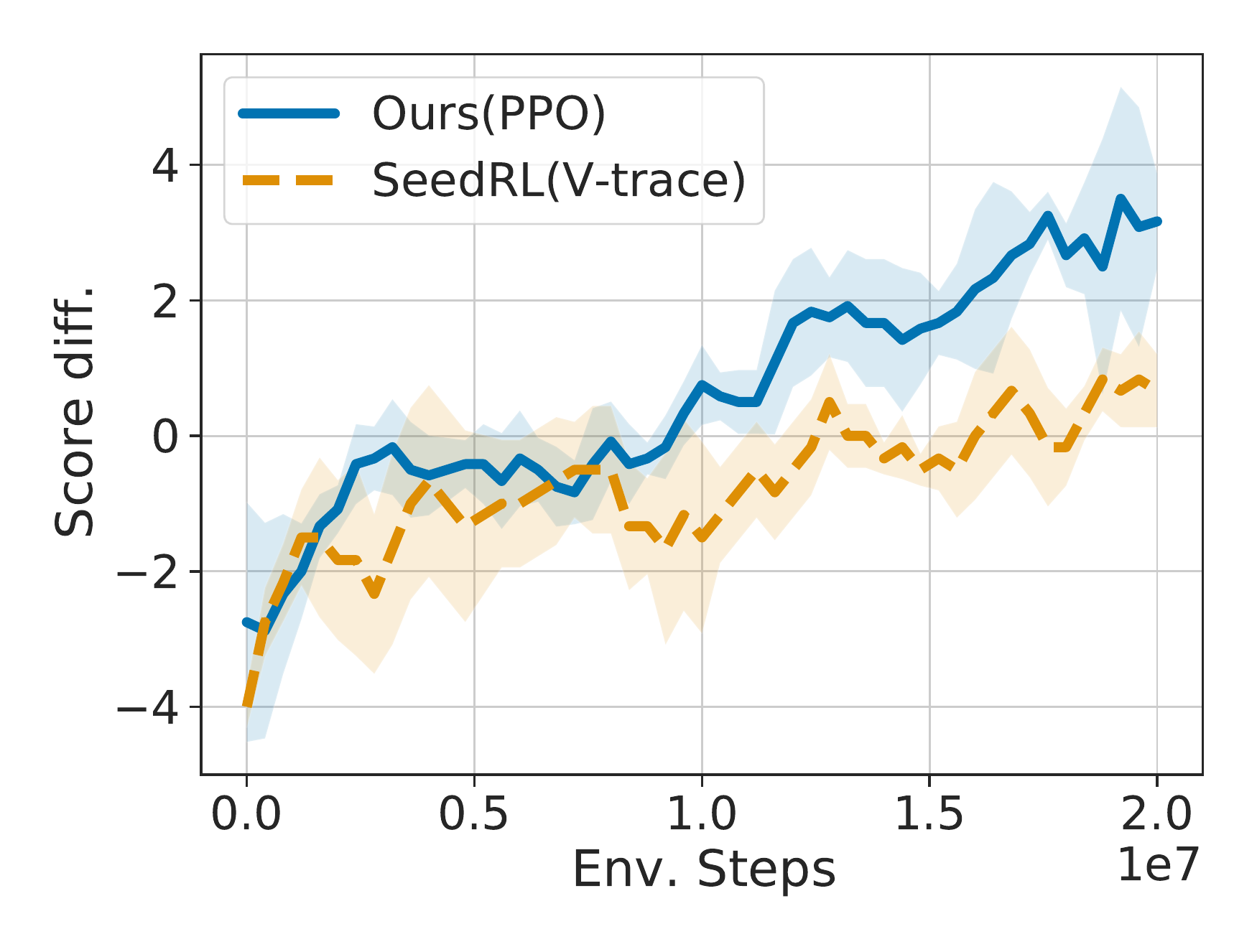}

\caption{HTS-RL and SeedRL on GFootball 11 \vs 11 easy task.}
\label{fig:seedrl}

\end{figure}
\begin{figure}[th]

\centering
\begin{tabular}{l|rrrr}
\specialrule{.15em}{.05em}{.05em}
         Method &  \makecell{\citet{pytorchrl}} &  \makecell{OpenAI\\ Baselines~\cite{a2c}} &  rlpyt~\cite{rlpyt} & Ours\\
\hline
\hline
BankHeist&  1382 $\pm$6  &   991 $\pm$14  &     1737 $\pm$39  &     \textbf{2111 $\pm$21} \\
Beam Rider&  1663 $\pm$14  &   1081 $\pm$18  &     2086 $\pm$32  &     \textbf{2586 $\pm$14} \\
Breakout&   1225 $\pm$12  &    829 $\pm$31  &      1508 $\pm$60  &     \textbf{1885 $\pm$15}  \\
Frostbite&    1337 $\pm$8 &    962 $\pm$15 &      1803 $\pm$17  &     \textbf{1973 $\pm$24}  \\
Jamesbond&    1353 $\pm$5 &    1014 $\pm$1 &      1991 $\pm$24 &     \textbf{2139 $\pm$31}  \\
Krull &  1443 $\pm$6  &   1057 $\pm$11  &    2001 $\pm$29 &     \textbf{2657 $\pm$16}   \\
KFMaster&  1532 $\pm$15  &  1056 $\pm$8  &    1979 $\pm$55 &    \textbf{ 2483 $\pm$15}   \\
MsPacman&   1574 $\pm$9  &   1052 $\pm$3  &    1972 $\pm$13  &     \textbf{2364 $\pm$5}  \\
Qbert &  1232 $\pm$13  &   953 $\pm$7  &   1621 $\pm$43 &     \textbf{1860 $\pm$6}   \\
Seaquest&   1593 $\pm$10  &   946 $\pm$21 &     1918 $\pm$25 &     \textbf{2633 $\pm$32}  \\
S. Invader &   1514 $\pm$20 &    1010 $\pm$7  &      1899 $\pm$32 &     \textbf{2318 $\pm$12}   \\
Star Gunner&  1622 $\pm$19  &  1110 $\pm$5  &    2066 $\pm$24 &    \textbf{ 2616 $\pm$25}   \\
\specialrule{.15em}{.05em}{.05em}
\end{tabular}
\captionof{table}{SPS of different implementations of A2C.}
\label{tab:a2c}

\end{figure}%
{
\section{Baselines}
\label{sec:supp_baseline}
\subsection{A2C}
To ensure the A2C implementation~\cite{pytorchrl} we use is a strong baseline, we compare the speed of different versions of A2C, including our HTS-RL, ~\citet{pytorchrl}, OpenAI baselines~\cite{a2c}, and rlpyt~\cite{rlpyt}, on Atari games. For a fair comparison, all methods use 16 parallel environment processes for data collection, and one GPU for model training/forwarding. For rlpyt, we use the most efficient `parallel-GPU' mode. As shown in~\tabref{tab:a2c}, Kostrikov's A2C is a strong baseline, which achieves $1.4\times$ higher SPS than OpenAI baselines. 
Also, HTS-RL consistently achieves higher SPS than rlpyt.  
\subsection{SeedRL}
SeedRL~\cite{seed_rl} is a recent work that reports results on GFootball `11 \vs 11 easy' task. We compare HTS-RL with Seed RL (V-trace)~\cite{seed_rl} on  Gfootball `11 \vs 11 easy.' For a fair comparison, both HTS-RL and Seed RL use 16 parallel environment processes and one GPU. HTS-RL achieves 829 environment steps per second (SPS) while Seed RL achieves 609 SPS. After 20M steps of training, HTS-RL and Seed RL achieve a $3.55\pm0.3$ and $1.50\pm 0.7$ score difference, respectively. The training curve is shown in~\figref{fig:seedrl}.
}

\section{Implementation Details}
\label{sec:supp_implement}

\subsection{Atari Game Experiments}
In Atari experiments, we use the same neural network architectures as~\citet{impala, torchbeast} for all three methods (IMPALA, A2C, Ours). The network has four hidden layers. The first layer is a convolutional layer with $32$ filters of size $8\times 8$ and stride $4$. The second layer is a convolutional layer with $64$ filters of size $4\times 4$ and stride $2$. The third layer is a
convolutional layer with $64$ filters of size $3\times 3$ and stride 1.  The fourth layer is a fully connected layer with $512$ hidden units. Following the hidden units are two sets of output. One provides a probability distribution over all valid actions. The other one provides the estimated value function. For ours (A2C) and A2C baseline, we use the same hyper-parameters as~\citet{pytorchrl}. For IMPALA, we use the same hyper-parameters as~\citet{torchbeast, impala}. We summarize the hyper-parameters in~\tabref{tab:atari_hyper}. 
{
Note~\citet{torchbeast} deploy distributed IMLALA with 48 actors. However, in this work we target single machine parallel computing, and restrict ourselves to 16 parallel environments.  For a fair comparison, we run all experiments with 16 parallel environments on a single machine. Importantly, while being downscaled to one machine, the reported IMPALA results match the results reported in the original paper~\cite{torchbeast}. \tabref{tab:atari_base} summarizes the results of our baseline and that reported by~\citet{torchbeast}. }

\begin{table}[t]
\begin{minipage}{0.48\textwidth}
\centering
\setlength{\tabcolsep}{4pt}
{\small
\begin{tabular}[t]{l|rr}
\specialrule{.15em}{.05em}{.05em}
         &  IMPALA  &  A2C/Ours \\
\hline
\hline

Unroll length & 20 & 5  \\
Batch size & 32 & -  \\
Discount factor & 0.99 & 0.99  \\
Value loss coefficient & 0.5 & 0.5  \\
Entropy loss coefficient & 0.01 & 0.01  \\
RMSProp momentum & 0.00 & 0.00  \\
RMSProp $\epsilon$ & 0.01 & 0.00001  \\
Learning rate & 0.006 & 0.0007  \\
Number of actors & 16 & 4 \\

 \specialrule{.15em}{.05em}{.05em}
\end{tabular}

\caption{Hyper-parameters of IMPALA and A2C/Ours(A2C) in Atari experiments.
}
\label{tab:atari_hyper}
}
\end{minipage}
\hspace{0.2cm}
\begin{minipage}{0.47\textwidth}
\centering
\setlength{\tabcolsep}{4pt}
{\small
\begin{tabular}[t]{l|rr}
\specialrule{.15em}{.05em}{.05em}
         Method &    \makecell{IMPALA\\ 16 actors \\ (Baseline)}   &   \makecell{IMPALA\\ 48 actors\\ \cite{torchbeast}}    \\
\hline
\hline
BankHeist&  339 &  $\sim$300    \\
Beam Rider&  4000  &   $\sim$4000   \\
Breakout&   201  &    $\sim$130   \\
Frostbite&    73  &    $\sim$70  \\
Jamesbond&    82  &    $\sim$80   \\
Krull &  2546  &  $\sim$2500   \\
KFMaster&  9516  &  $\sim$8000  \\
MsPacman&   807  &   $\sim$1300   \\
Qbert &  4116  &   $\sim$4000  \\
Seaquest&   458   &   $\sim$420  \\
S. Invader &   {1142} &    $\sim$2000   \\
Star Gunner&  8560  &  $\sim$6000   \\
\specialrule{.15em}{.05em}{.05em}
\end{tabular}
\caption{Atari@20M environment steps. Since~\citet{torchbeast} don't report exact scores at 20M environment steps, we obtain their numbers from their plots and indicate that with a $\sim$ symbol.}
\label{tab:atari_base}
}
\end{minipage}
\end{table}

\subsection{GFootball Experiments}
In GFootball  experiments, we use the CNN architecture of~\citet{gfootball} for all three methods (IMPALA, PPO, Ours). The network has four hidden layers. The first layer is a convolutional layer with 32 filters of size $8\times 8$ 
and stride 4. The second layer is a convolutional layer with 64 filters of size $4\times 4$ and stride 2. The third layer is a
convolutional layer with 64 filters of size $3 \times 3$ and stride 1.  The fourth layer is a fully connected layer with 512 hidden units. Following the hidden units are two sets of output. One provides a probability distribution over all valid actions. The other one provides the estimated value function. To be consistent with the official torch beast implementation~\cite{torchbeast}, we use RMSProp for all methods. Regarding hyper-parameters, for ours (PPO) and PPO baseline, we mostly use the same hyper-parameters as~\citet{gfootball}. The only difference is that, instead of 512 steps, we unroll for 128 steps, which we found to give better results. For IMPALA,~\citet{gfootball} deploy distributed training with 500 actors. However, in this work we target single machine parallel computing, and restrict ourselves to 16 parallel environments. Therefore, for a fair comparison, we mostly follow the hyper-parameter settings of~\citet{gfootball}, but decrease the number of actors and batch size. With only 16 actors and a smaller batch size, our baseline results match the results of IMPALA on GFootball environments reported by~\citet{gfootball}. 
{\tabref{tab:gfootball_base} summarizes the results of our baseline and that reported by~\citet{gfootball}.}
The hyper-parameters are summarized in~\tabref{tab:gfootball_hyper}.

\section{Final Time and Required Time Metrics}
\label{sec:supp_metrics}
The results of all Atari experiments in \emph{final time metric} and \emph{required time metric} are summarized in~\tabref{tab:atari_final} and~\tabref{tab:atari_required}.  For \emph{final time metric}, the time limit for each experiment is set to the time when IMPALA finishes training for 20M steps. For \emph{required time metric}, we report the time to achieve average episode rewards of $40\%$ and $80\%$ of the A2C baseline episode rewards reported by~\citet{a2c}.
The results of all GFootball experiments in \emph{final time metric} and \emph{required time metric} are summarized in~\tabref{tab:football_final} and~\tabref{tab:football_required}. 
For \emph{final time metric}, the time limit for each experiment is set to the time when IMPALA finishes training for 5M steps. For \emph{required time metric}, we report the time to achieve an average score of 0.4 and 0.8. As shown in~\tabref{tab:atari_final} and~\tabref{tab:football_final}, given the same amount of time, HTS-RL consistently achieves higher average rewards/scores than IMPALA and synchronous baselines. Moreover, as shown in~\tabref{tab:atari_required} and~\tabref{tab:football_required}, to achieve a target reward/score, HTS-RL consistently needs less time.

\begin{table}[]
\centering
\setlength{\tabcolsep}{4pt}
{\small
\begin{tabular}{l|rr}
\specialrule{.15em}{.05em}{.05em}
         Method &    \makecell{IMPALA\\ 16 actors \\ (Baseline)}   &   \makecell{IMPALA\\ 500 actors\\ \cite{gfootball}}    \\
\hline
\hline
{Empty goal close} & 1.0&    $\sim$0.99   \\
                         
{Empty goal}  & 1.0&    $\sim$0.85   \\
                           
{Run to score}  & 0.80&   $\sim$ 0.80  \\             
                                       
{RSK}  & 0.05&    $\sim$0.22  \\

{PSK}  & 0.20&    $\sim$0.18   \\

{RPSK} & 0.82&    $\sim$0.41   \\
     
{3 vs 1 w/ keeper}  & 0.21&    $\sim$0.20 \\
    
{Corner}  & 0.0&   $\sim$-0.1  \\
                   
{Counterattack easy} & 0.0&    $\sim$0.0  \\
                        
{Counterattack hard}  & 0.50&    $\sim$0.0  \\
{11 vs 11 w/ lazy Opp.} & 0.71&    $\sim$0.38 \\
\specialrule{.15em}{.05em}{.05em}
\end{tabular}
\caption{GFootball Academy@5M environment steps. Since~\citet{gfootball} don't report exact scores at 5M environment steps, we obtain their numbers from their plots and indicate that with a $\sim$ symbol.}
\label{tab:gfootball_base}
}
\end{table}

\begin{table}[]
\centering
\setlength{\tabcolsep}{4pt}
{\small
\begin{tabular}{l|rr}
\specialrule{.15em}{.05em}{.05em}
         &  IMPALA  &      PPO / PPO(Ours) \\
\hline
\hline

Unroll length & 32 & 128  \\
Batch size & 8 & 16  \\
Discount factor & 0.993 & 0.993  \\
Value loss coefficient & 0.5 & 0.5  \\
Entropy loss coefficient &  0.00087453 & 0.003  \\
RMSProp momentum & 0.00 & 0.00  \\
RMSProp $\epsilon$ & 0.01 & 0.00001  \\
Learning rate & 0.00019896 & 0.000343  \\
Number of actors & 16 & 4 \\

 \specialrule{.15em}{.05em}{.05em}
\end{tabular}

\caption{Hyper-parameters of IMPALA and PPO/Ours(PPO) in GFootball experiments.
}
\label{tab:gfootball_hyper}
}
\end{table}

\begin{table}[]
\centering
\setlength{\tabcolsep}{4pt}
{\small
\begin{tabular}{l|rrr}
\specialrule{.15em}{.05em}{.05em}
         Method &  IMPALA  &      A2C &  Ours (A2C)\\
\hline
\hline
BankHeist&  339 $\pm$ 10 &   775 $\pm$ 166 &     \bf{942 $\pm$ 100}   \\
Beam Rider&  4000 $\pm$ 690 &   4392 $\pm$ 134 &     \bf{6995} $\pm$ 420  \\
Breakout&   201 $\pm$ 133 &    362 $\pm$ 29 &      \bf{413 $\pm$ 37}  \\
Frostbite&    73 $\pm$ 2 &    272 $\pm$ 14&      \bf{315 $\pm$ 12}  \\
Jamesbond&    82 $\pm$ 10 &    438 $\pm$ 59 &      \bf{474 $\pm$ 88}  \\
Krull &  2546 $\pm$ 551 &   7560 $\pm$ 892 &     \bf{7737 $\pm$ 609}  \\
KFMaster&  9516 $\pm$ 3311 &  \bf{30752 $\pm$ 6641} &    30020 $\pm$ 3559  \\
MsPacman&   807 $\pm$ 170 &   1236 $\pm$ 292 &     \bf{1675 $\pm$ 459}  \\
Qbert &  4116 $\pm$ 610 &   12479 $\pm$ 1965 &    \bf{13682 $\pm$ 1873}  \\
Seaquest&   458 $\pm$ 2  &   \bf{1833 $\pm$ 6} &     1831 $\pm$ 7 \\
S. Invader &   \bf{1142} $\pm$ 207&    596 $\pm$ 69 &      {731} $\pm$ 80  \\
Star Gunner&  8560 $\pm$ 918 &  41414 $\pm$ 3826 &    \bf{52666 $\pm$ 5182}  \\

\specialrule{.15em}{.05em}{.05em}
\end{tabular}
\caption{{Atari experiment in \emph{final time metrics}: }Average evaluation rewards achieved given limited training time. }
\label{tab:atari_final}
}
\end{table}

\begin{table}[]
\centering
\setlength{\tabcolsep}{4pt}
{\small
\begin{tabular}{l|rrr}
\specialrule{.15em}{.05em}{.05em}
         Method (target reward 1 / target reward 2)&  IMPALA  &      A2C &  Ours (A2C)\\
\hline
\hline
BankHeist (480 / 960)&  -/- &    28.9/{\bf62.1} &      {\bf18.9}/116.8 \\
Beam Rider (1600 / 3200)&  36.4/60.4 &    32.1/54.5 &      {\bf10.3/30.9}  \\
Breakout (160 / 320)&   77.8/- &    21.7/43.5 &      {\bf17.7/38.9} \\
Frostbite (104 / 208)&    -/- &    5.0/10.0 &      {\bf3.4/6.8}  \\
Jamesbond (200 / 400)&    -/- &    39.4/49.2 &      {\bf21.8/31.1} \\
Krull (3600 / 7200)& -/- &    13.9/37.0 &      {\bf7.5/37.6}  \\
KFMaster (15200 / 30420)&  -/- &   39.5/192.6 &      {\bf18.8/118.1}  \\
MsPacman (880 / 1760)&  75.8/- &    49.3/160.3 &      {\bf22.9/94.3}  \\
Qbert (4000 / 8000)&  94.2/- &    53.3/83.8 &      {\bf52.2/67.7}  \\
Seaquest (640 / 1280)&   -/- &    6.7/28.4 &      {\bf4.0/17.2} \\
Space. (240 / 480)&   {\bf9.65/19.9} &    14.1/26.4 &      {6.9/21.8}  \\
Star Gunner (8400 / 16800)&  47.1/- &    28.7/41.1 &     {\bf 17.8/25.4}  \\

\specialrule{.15em}{.05em}{.05em}
\end{tabular}
\caption{{Atari experiment in \emph{required time metrics}: } Required time (minutes) to achieve goal episode rewards (time required to achieve 40\% rewards reported by~\citet{a2c} / time required to achieve 80\% rewards reported by~\citet{a2c}).  `-' indicates that the method did not achieve the desired reward after 20M environment step training.
Space.: Space invaders, KFMaster: KungFu Master.
}
\label{tab:atari_required}
}
\end{table}

\begin{table}[]
\centering
\setlength{\tabcolsep}{4pt}
{\small
\begin{tabular}{l|rrr}
\specialrule{.15em}{.05em}{.05em}
         Method & IMPALA  &      PPO &  \makecell{Ours (PPO)}\\
\hline
\hline
{Empty goal close} & {\bf{1.00 $\pm$ 0.00}}&    {\bf{1.00 $\pm$ 0.00}} &      {\bf{1.00 $\pm$ 0.00}}  \\
                         
{Empty goal}  &{\bf{1.00 $\pm$ 0.00}}&    0.89 $\pm$0.05 &      0.94 $\pm$0.04  \\
                           
{Run to score}  & 0.65$\pm$0.42 &    0.89$\pm$0.05 &      \bf{0.93$\pm$0.03}  \\             
                                       
{RSK}  & 0.03$\pm$0.03 &    0.52$\pm$0.21 &      \bf{0.88$\pm$0.06}  \\  

{PSK}  & 0.00$\pm$0.00 &    0.05$\pm$0.04 &      \bf{0.41$\pm$0.02}  \\  

{RPSK} & 0.67$\pm$0.06 &    0.49$\pm$0.06 &      \bf{0.80$\pm$0.03}  \\  
     
{3 vs 1 w/ keeper}  &0.23$\pm$0.01 &    0.20$\pm$0.09 &      \bf{0.81$\pm$0.02}  \\  
    
{Corner}  & -0.10$\pm$0.33 &    -0.06$\pm$0.08 &      \bf{0.03$\pm$0.10}  \\  
                   
{Counterattack easy} & 0.00$\pm$0.00 &    0.01$\pm$0.01 &      \bf{0.39$\pm$0.02}  \\  
                        
{Counterattack hard}  & 0.00$\pm$0.00 &    0.01$\pm$0.02 &      \bf{0.53$\pm$0.09}  \\  
{11 vs 11 w/ lazy Opp.}  &0.46$\pm$0.21 &    0.33$\pm$0.07 &      \bf{0.72$\pm$0.09}  \\  
\specialrule{.15em}{.05em}{.05em}
\end{tabular}
\caption{GFootball experiments in \emph{final time metrics:}  Average evaluation scores achieved given limited training time. The time limit for each experiment is set to the time when IMPALA finishes training for 5M steps RSK: run to score w/ keeper, PSK: pass, shoot, w/ keeper, RPSK: run, pass, shoot, w/ keeper.
}
\label{tab:football_final}
}
\end{table}

\begin{table}[]
\centering
\setlength{\tabcolsep}{4pt}
{\small
\begin{tabular}{l|rrr}
\specialrule{.15em}{.05em}{.05em}
         Method & IMPALA  &      PPO &  \makecell{Ours (PPO)}\\
\hline
\hline
{Empty goal close} & 1.7/2.6&    5.4/15.5 &      \bf{1.0/2.0}  \\
                         
{Empty goal}  & 8.4/11.7&    12.8/19.2 &      \bf{2.0/3.9}  \\
                           
{Run to score}  & 27.0/34.6&    16.2/32.5 &      \bf{6.3/11.4}  \\             
                                       
{RSK}  & 52.3/-&    51.2/68.2 &      \bf{11.5/18.8} \\

{PSK}  & -/-&    70.0/- &      {\bf{38.8}}/-  \\

{RPSK} & 22.3/\bf{25.4}&    45.2/90.8 &      {\bf{13.5}}/27.1  \\
     
{3 vs 1 w/ keeper}  & -/-&    67.4/144.2 &      \bf{15.9/25.6}\\
    
{Corner}  & -/-&    -/- &      -/- \\
                   
{Counterattack easy} & -/-&    223.2/- &      {\bf{91.3}}/- \\
                        
{Counterattack hard}  & -/-&    383.4/- &      \bf{61.8/-}  \\
{11 vs 11 w/ lazy Opp.}  &58.2/-&    95.8/260.9 &      \bf{14.4/72.1} \\
\specialrule{.15em}{.05em}{.05em}
\end{tabular}
\caption{GFootball experiments in \emph{required time metrics:} required time (minutes) to achieve goal scores (time required to achieve score 0.4 / time required to achieve score 0.8).  `-' indicates that the method did not achieve the desired score after 5M environment step training. RSK: run to score w/ keeper, PSK: pass, shoot, w/ keeper, RPSK: run, pass, shoot, w/ keeper.
}
\label{tab:football_required}
}
\end{table}

\section{Training Curves}
\label{sec:supp_curves}
The training curves of all Atari and GFootball experiments in terms of time and number of environment steps are shown in~\figref{fig:atari_time_plot}, \figref{fig:atari_step_plot}, \figref{fig:football_time_plot}, and \figref{fig:football_step_plot}. As shown in~\figref{fig:atari_step_plot} and~\figref{fig:football_step_plot}, HTS-RL does not trade data efficiency for higher throughput. While achieving much higher throughput, HTS-RL still maintains a  data efficiency similar to synchronous baselines. As a result, HTS-RL consistently achieves higher rewards in shorter time than IMPALA and synchronous baselines across different environments (\figref{fig:atari_time_plot}, \figref{fig:football_time_plot}).
\begin{figure}[t]
\centering
\begin{tabular}{ccc}
\includegraphics[width=0.33\textwidth]{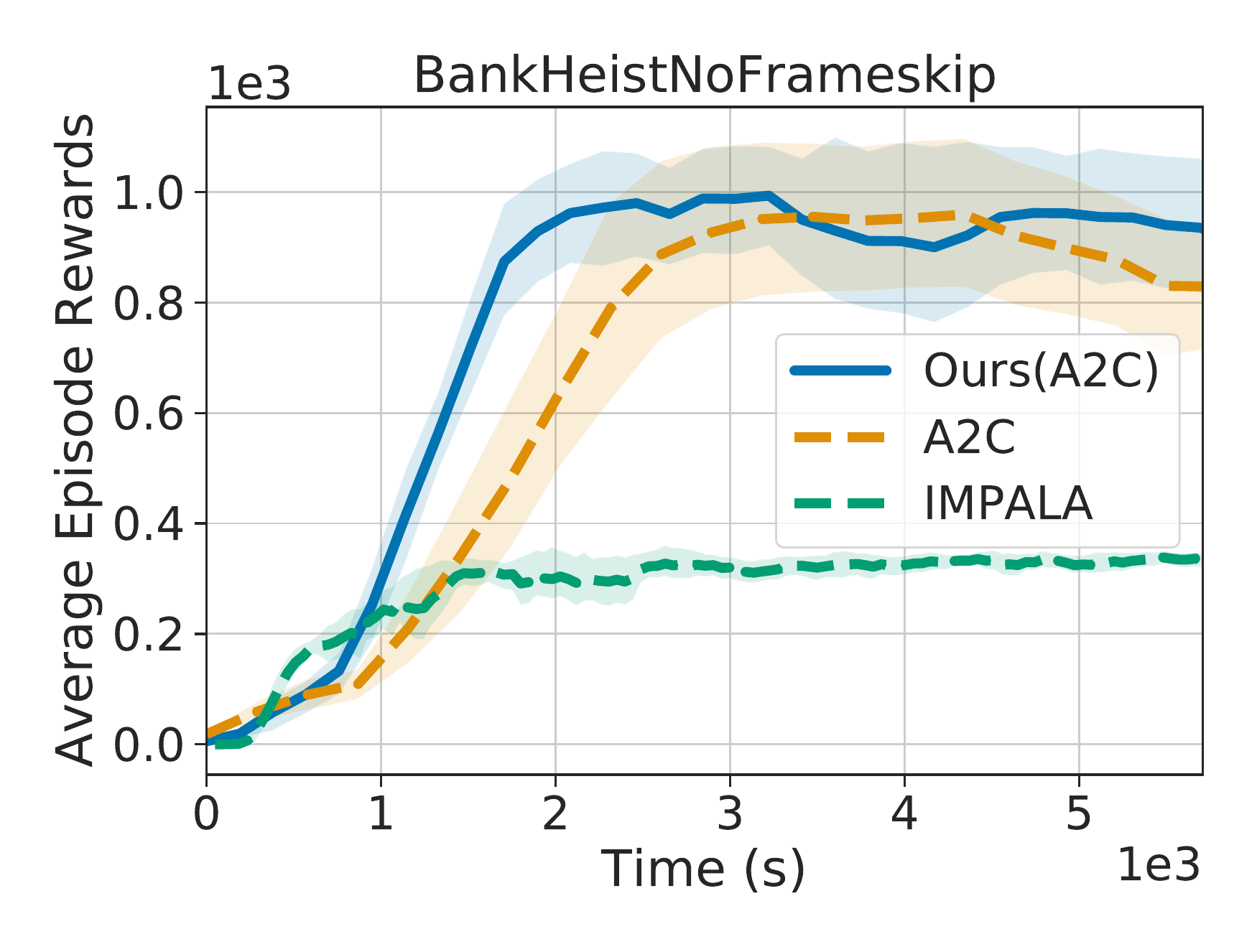}
&\hsp{-0.6cm}
\includegraphics[width=0.33\textwidth]{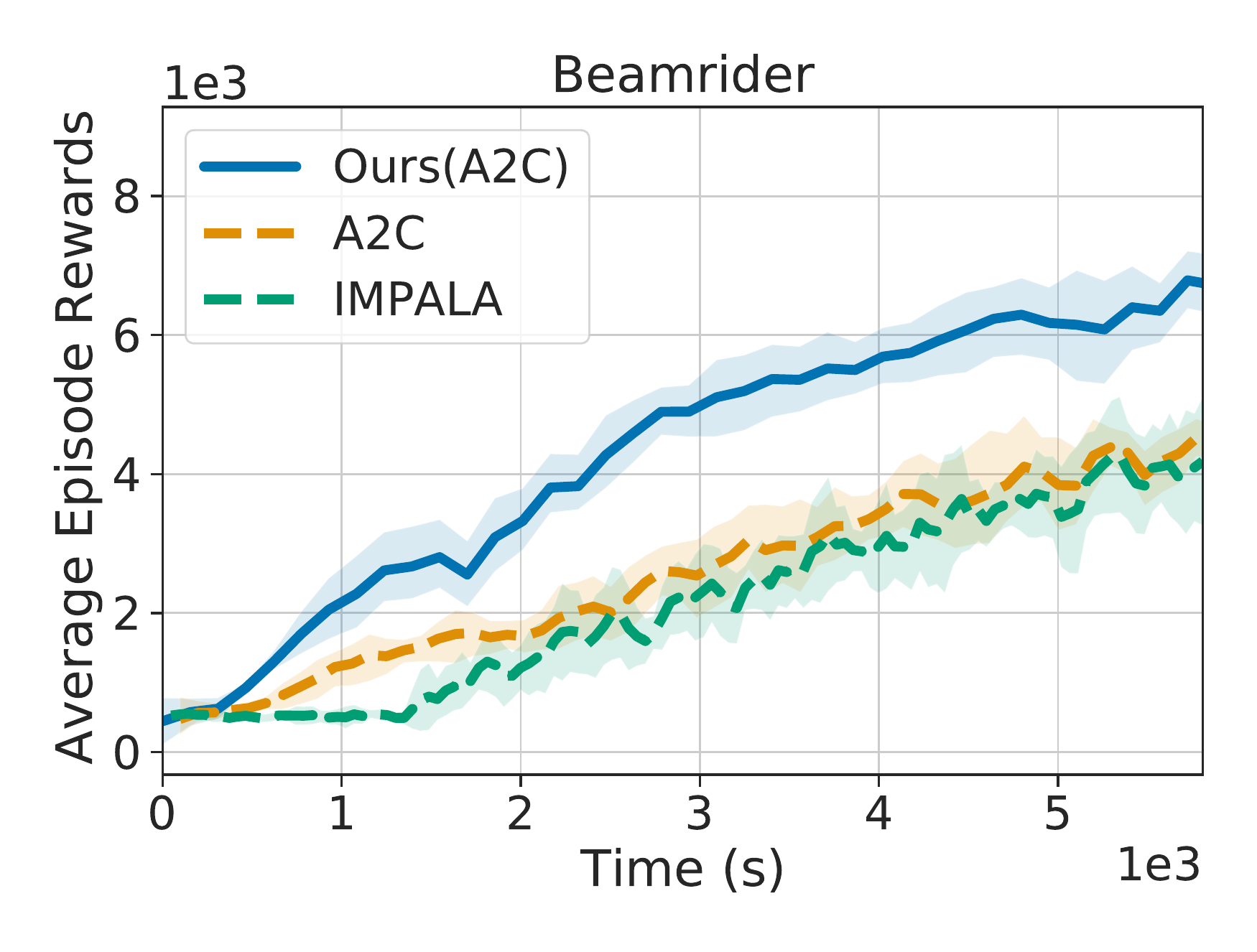}
&\hsp{-0.6cm}
\includegraphics[width=0.33\textwidth]{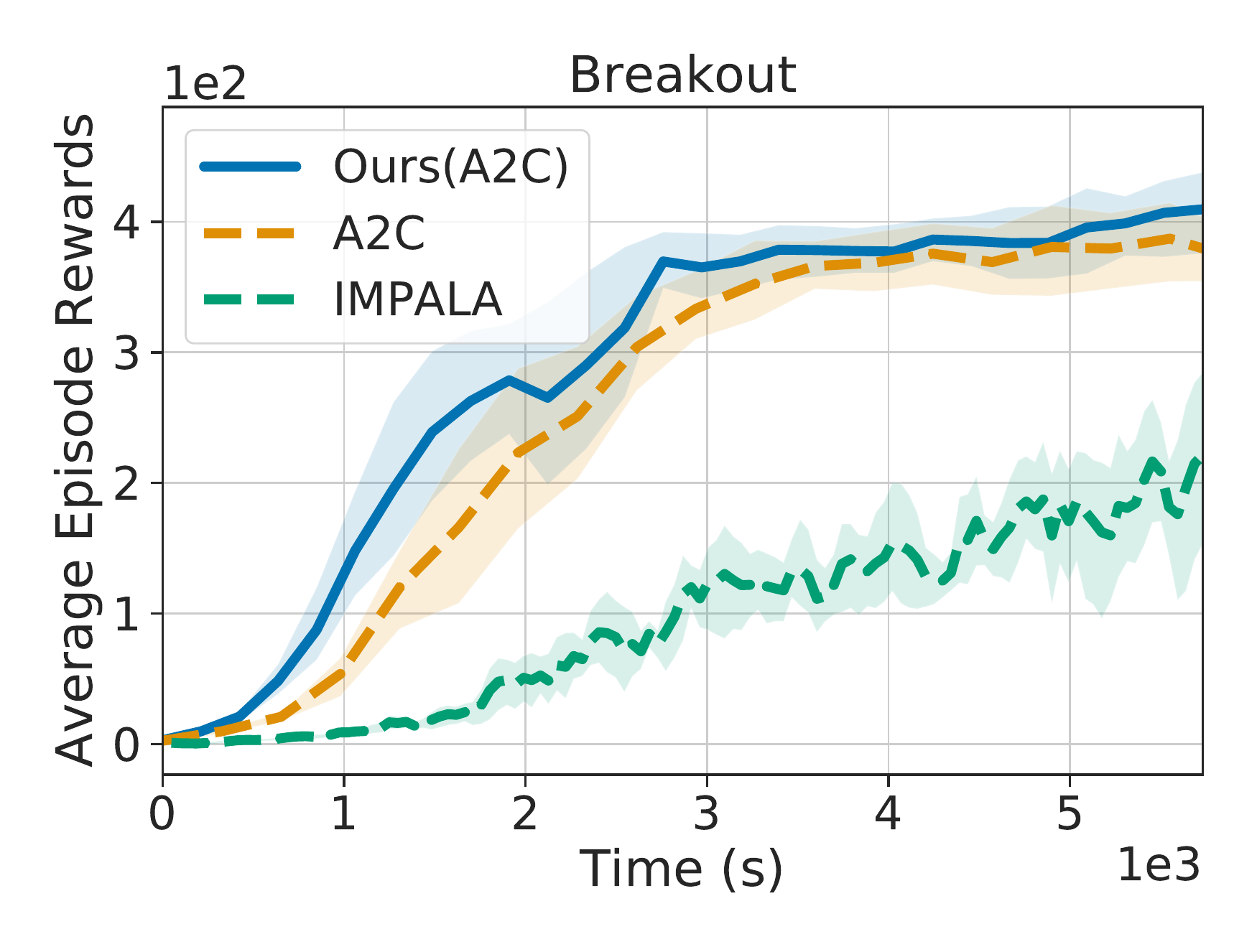}\\

\includegraphics[width=0.33\textwidth]{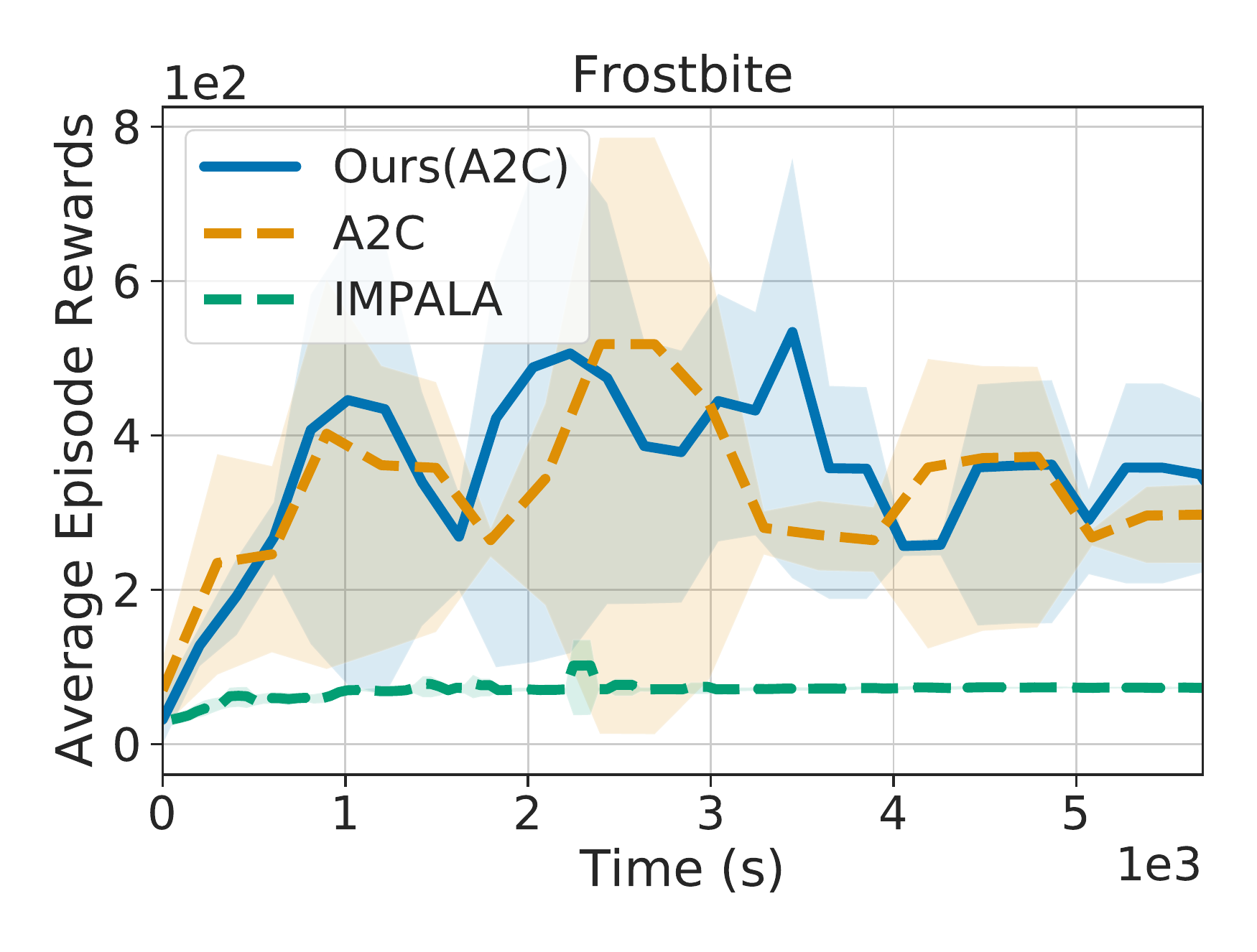}
&\hsp{-0.6cm}
\includegraphics[width=0.33\textwidth]{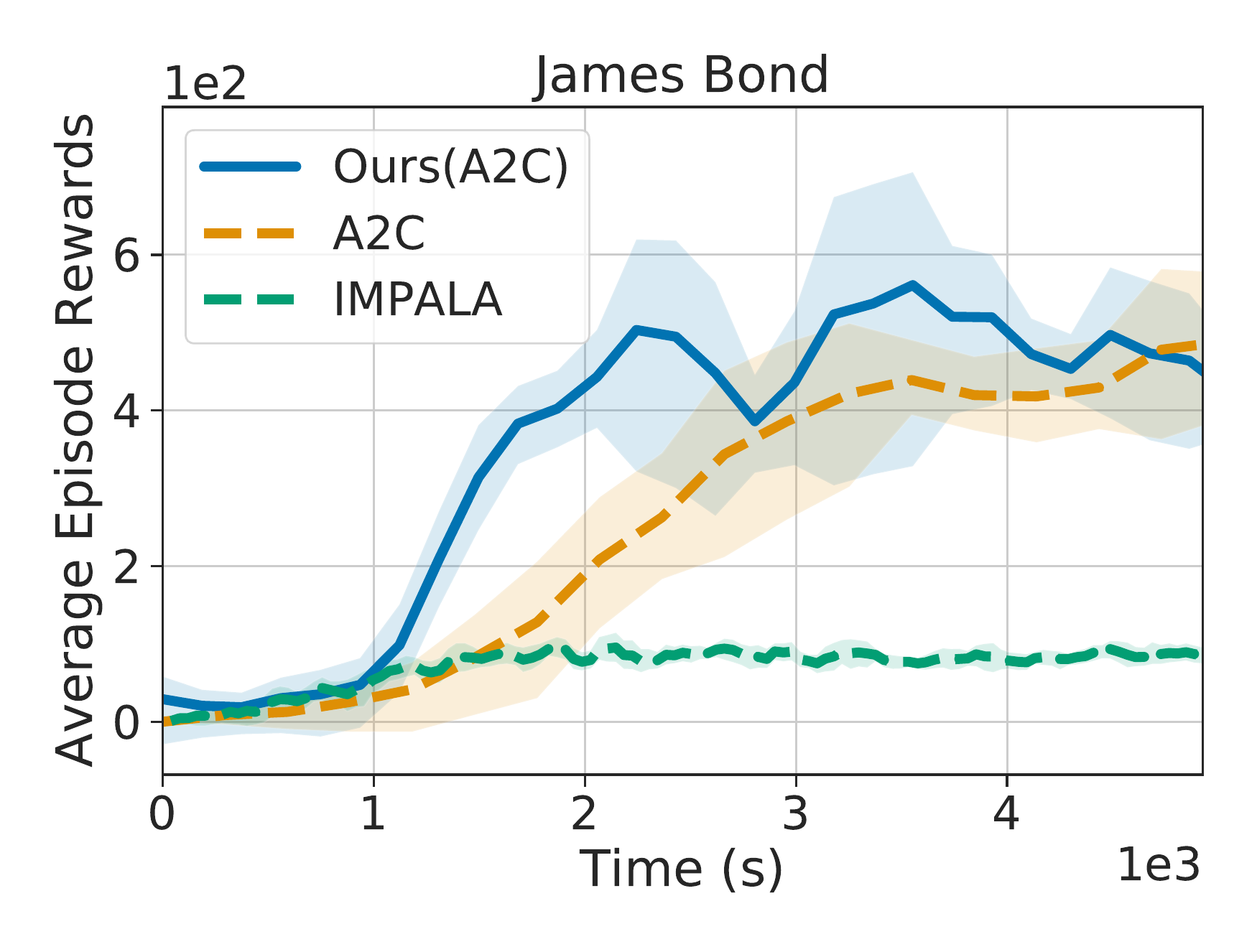}
&\hsp{-0.6cm}
\includegraphics[width=0.33\textwidth]{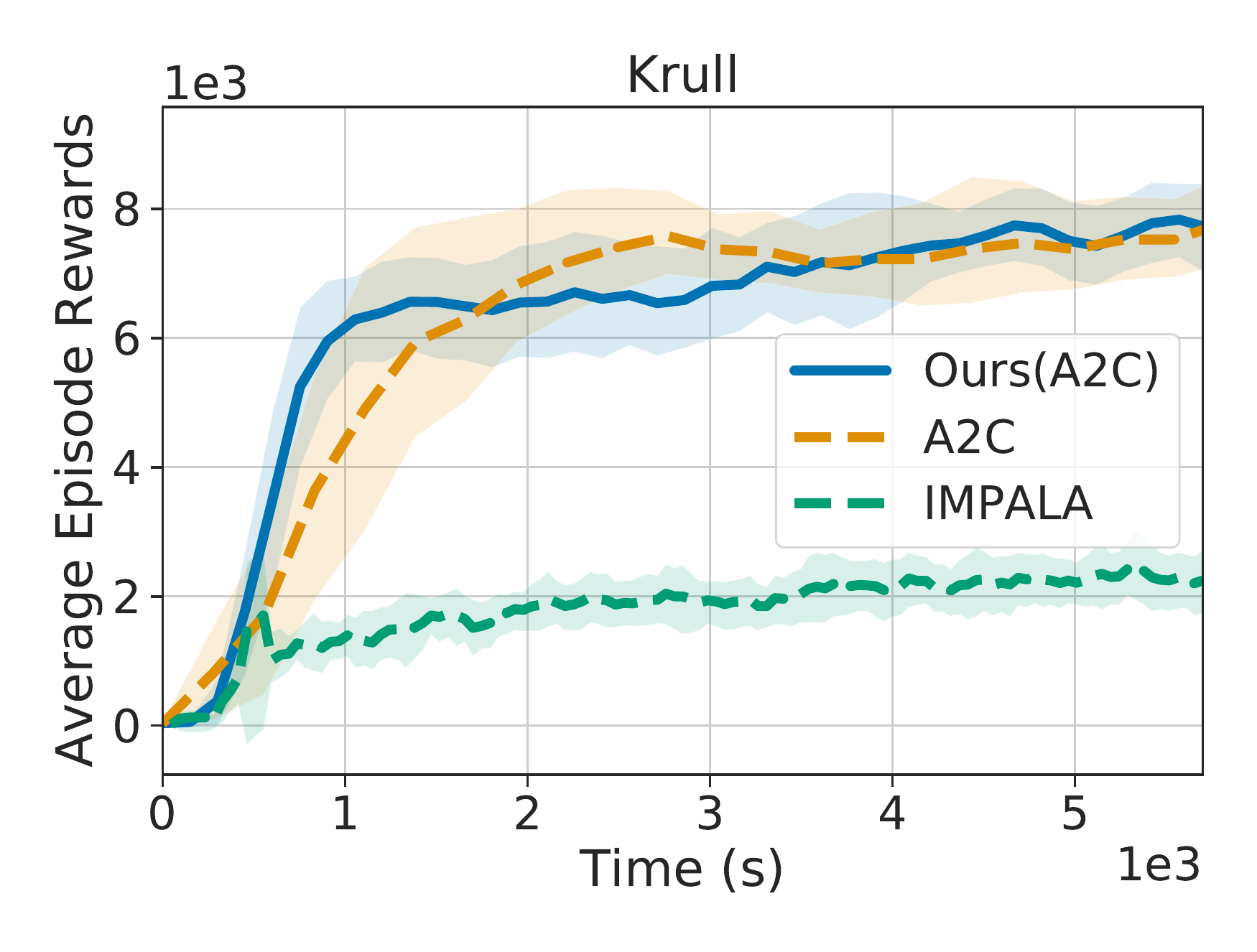}\\

\includegraphics[width=0.33\textwidth]{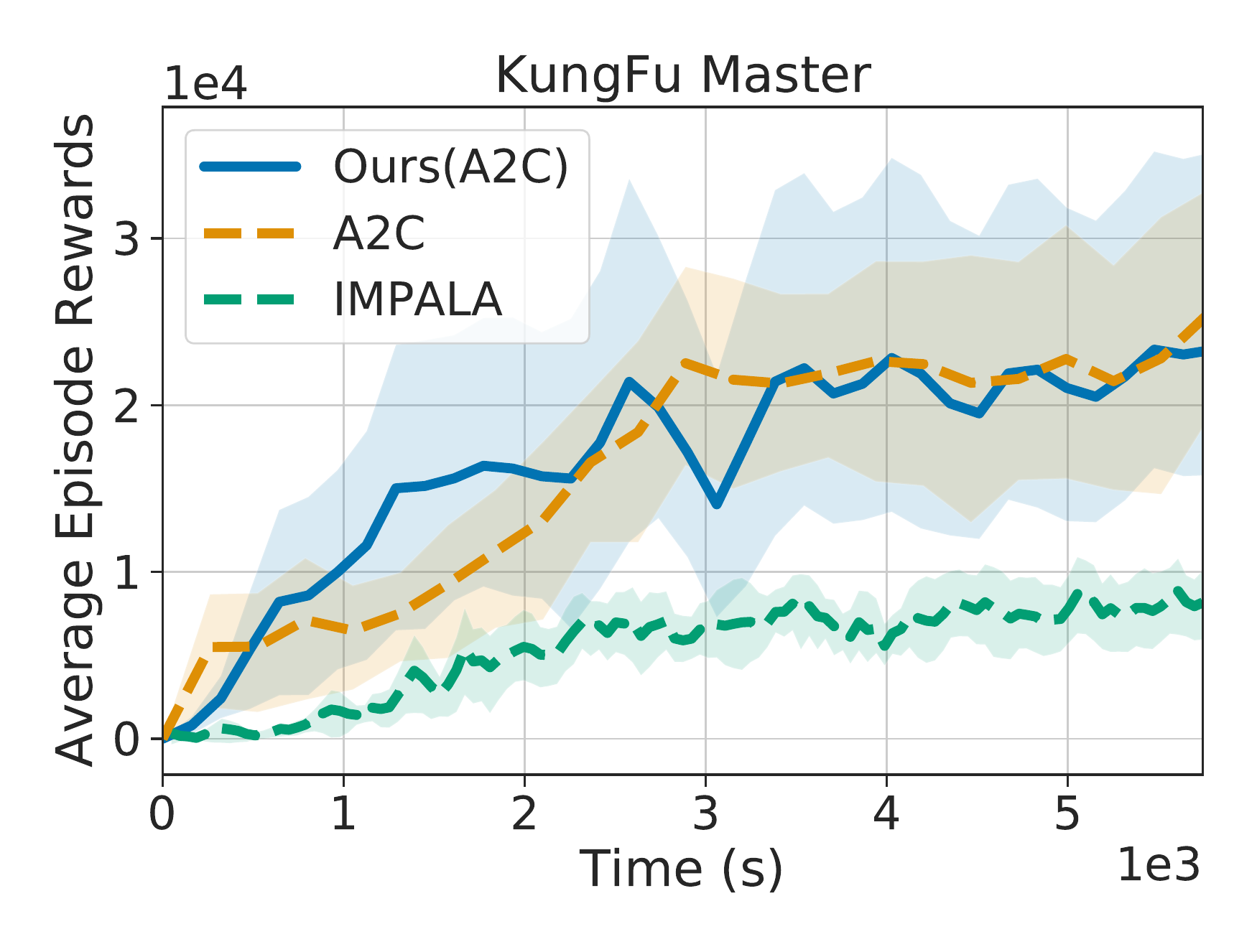}
&\hsp{-0.6cm}
\includegraphics[width=0.33\textwidth]{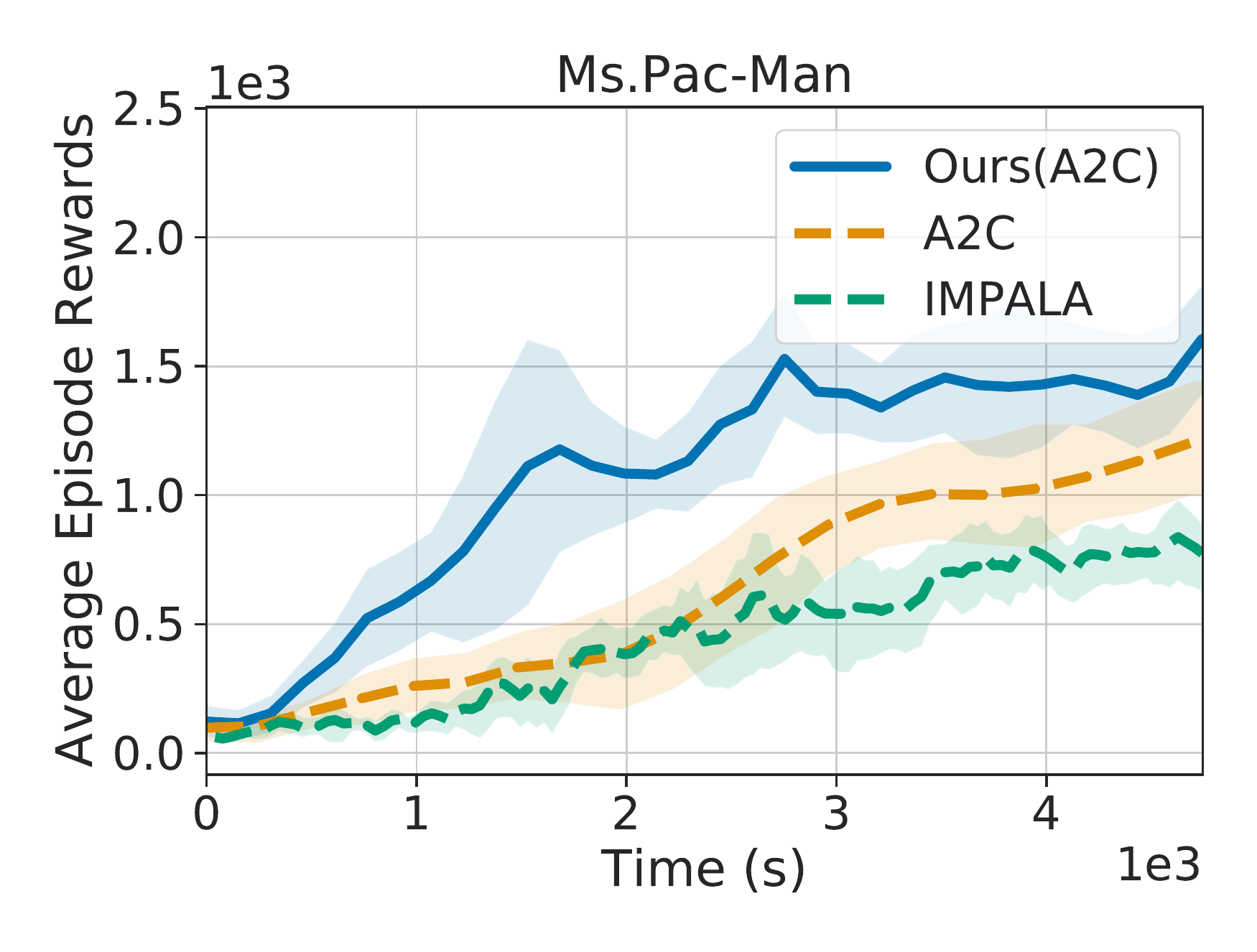}
&\hsp{-0.6cm}
\includegraphics[width=0.33\textwidth]{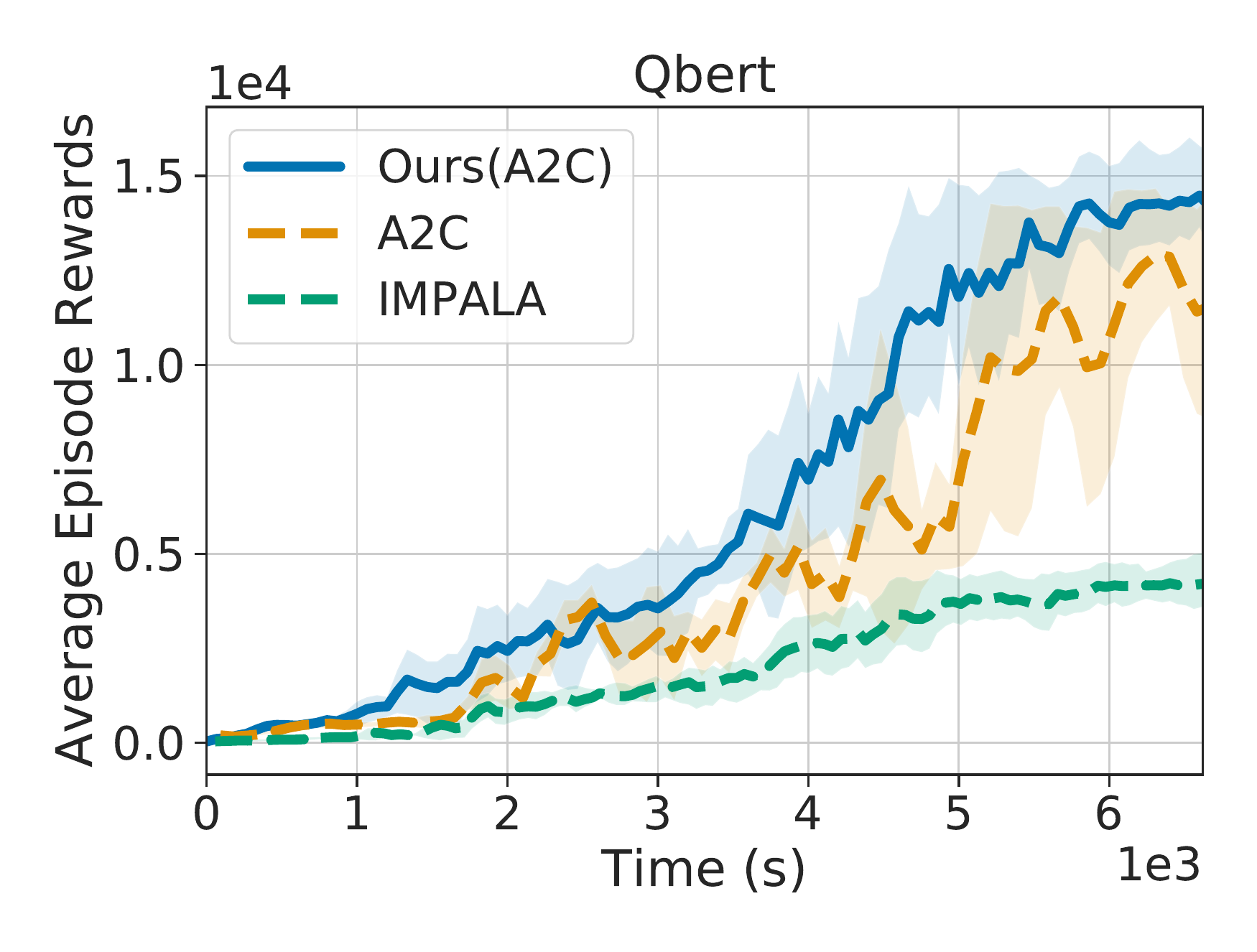}\\

\includegraphics[width=0.33\textwidth]{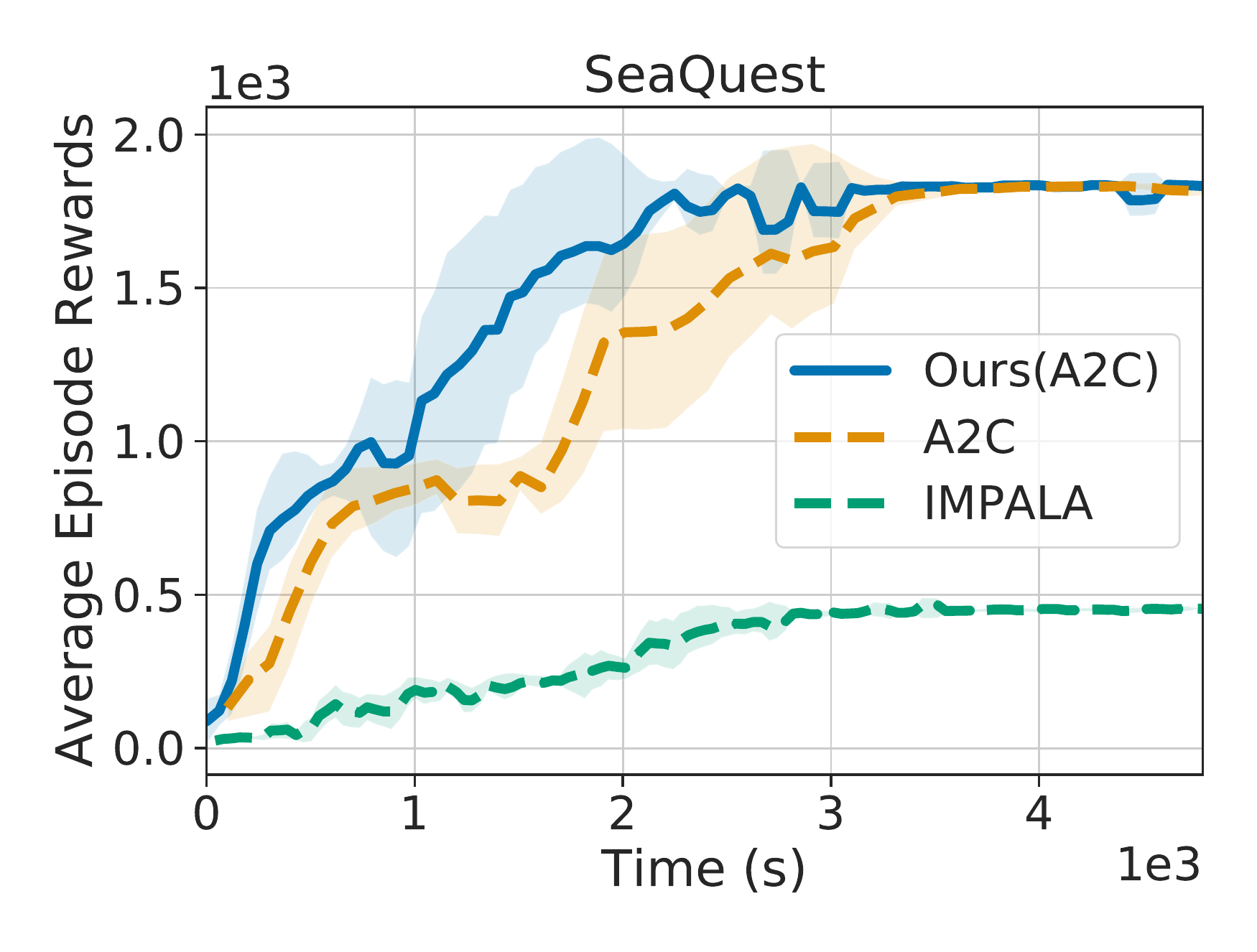}
&\hsp{-0.6cm}
\includegraphics[width=0.33\textwidth]{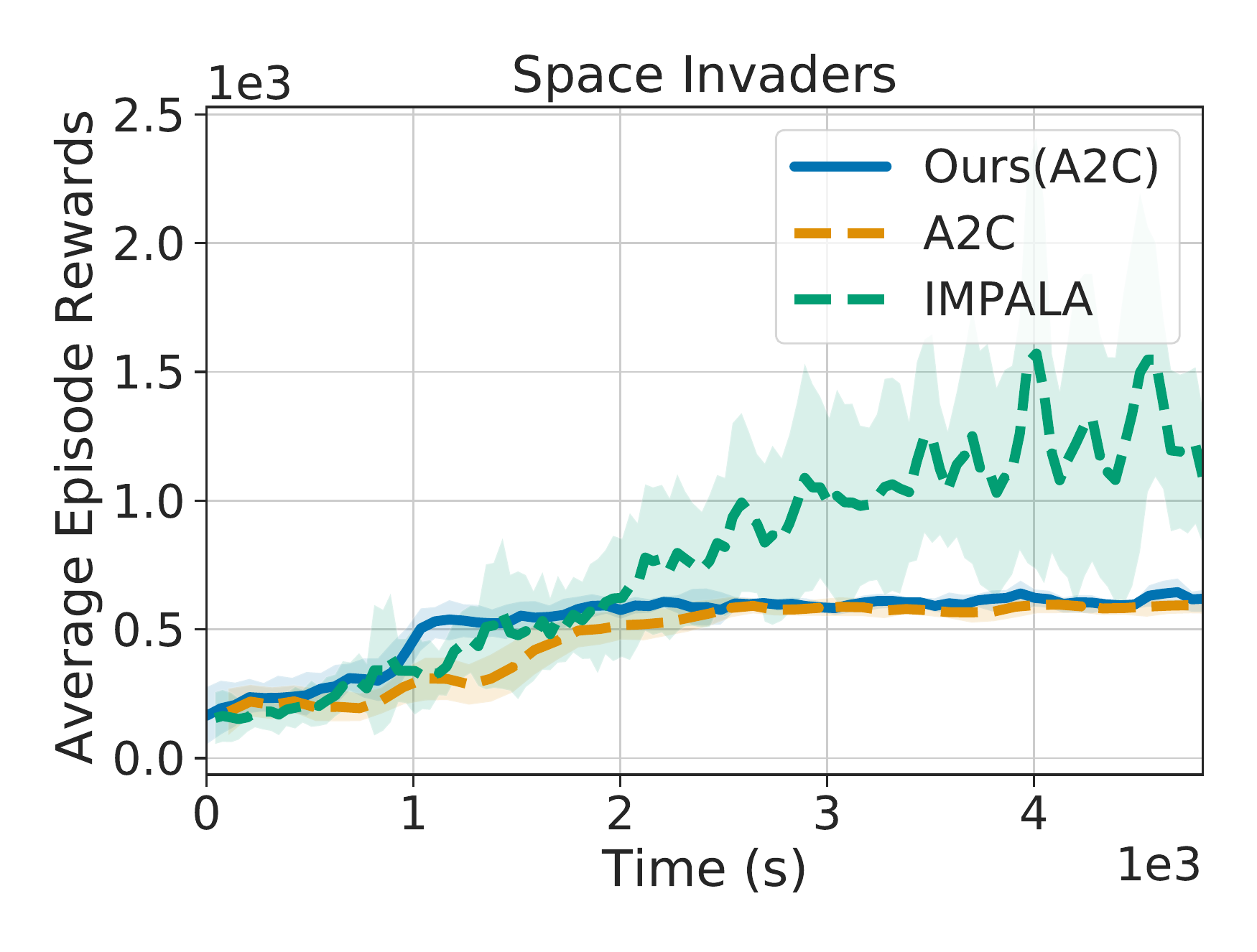}
&\hsp{-0.6cm}
\includegraphics[width=0.33\textwidth]{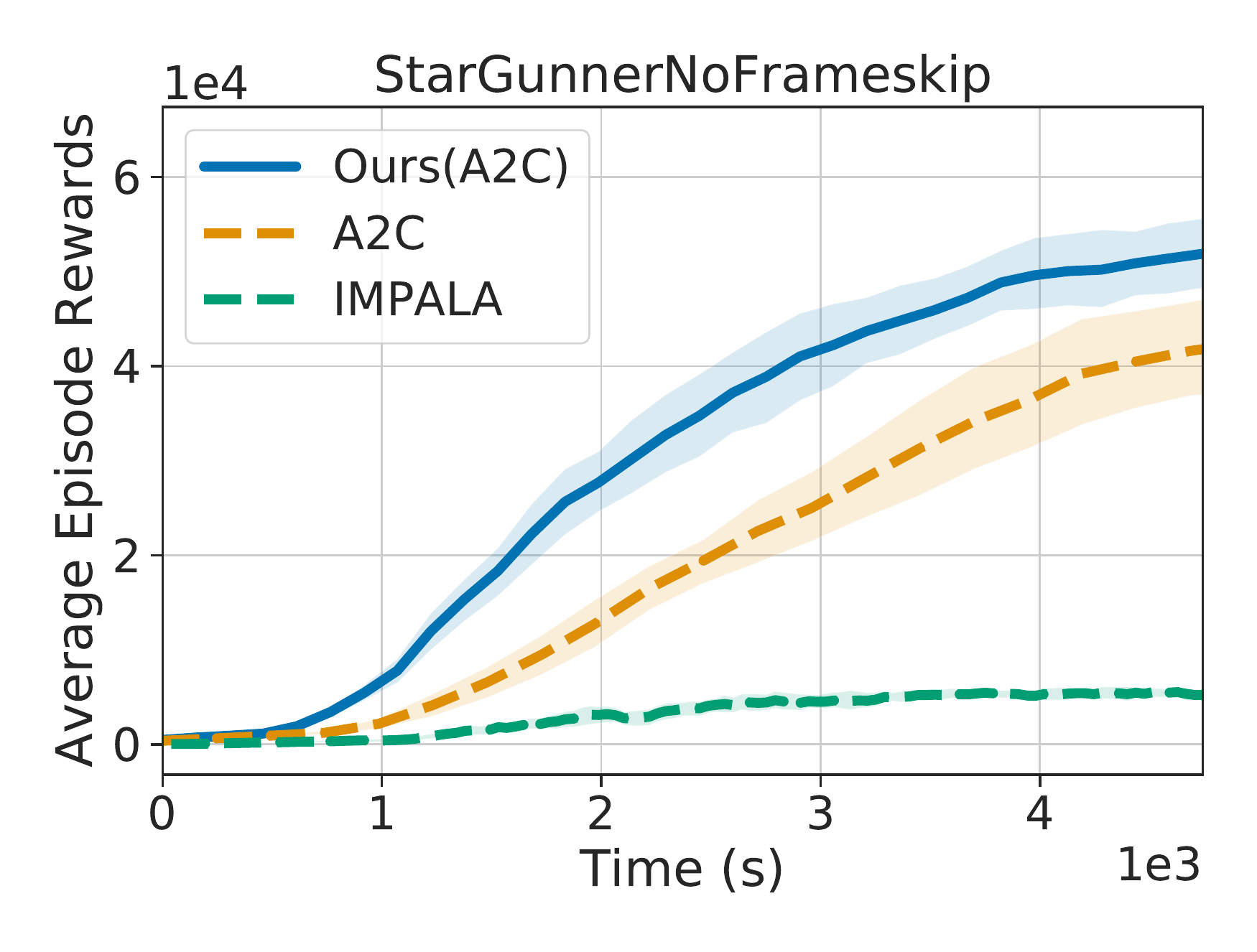}
\end{tabular}
\caption{Atari: {\textbf{Time}} versus reward.} 
\label{fig:atari_time_plot}
\end{figure}

\begin{figure}[t]
\centering
\begin{tabular}{ccc}
\includegraphics[width=0.33\textwidth]{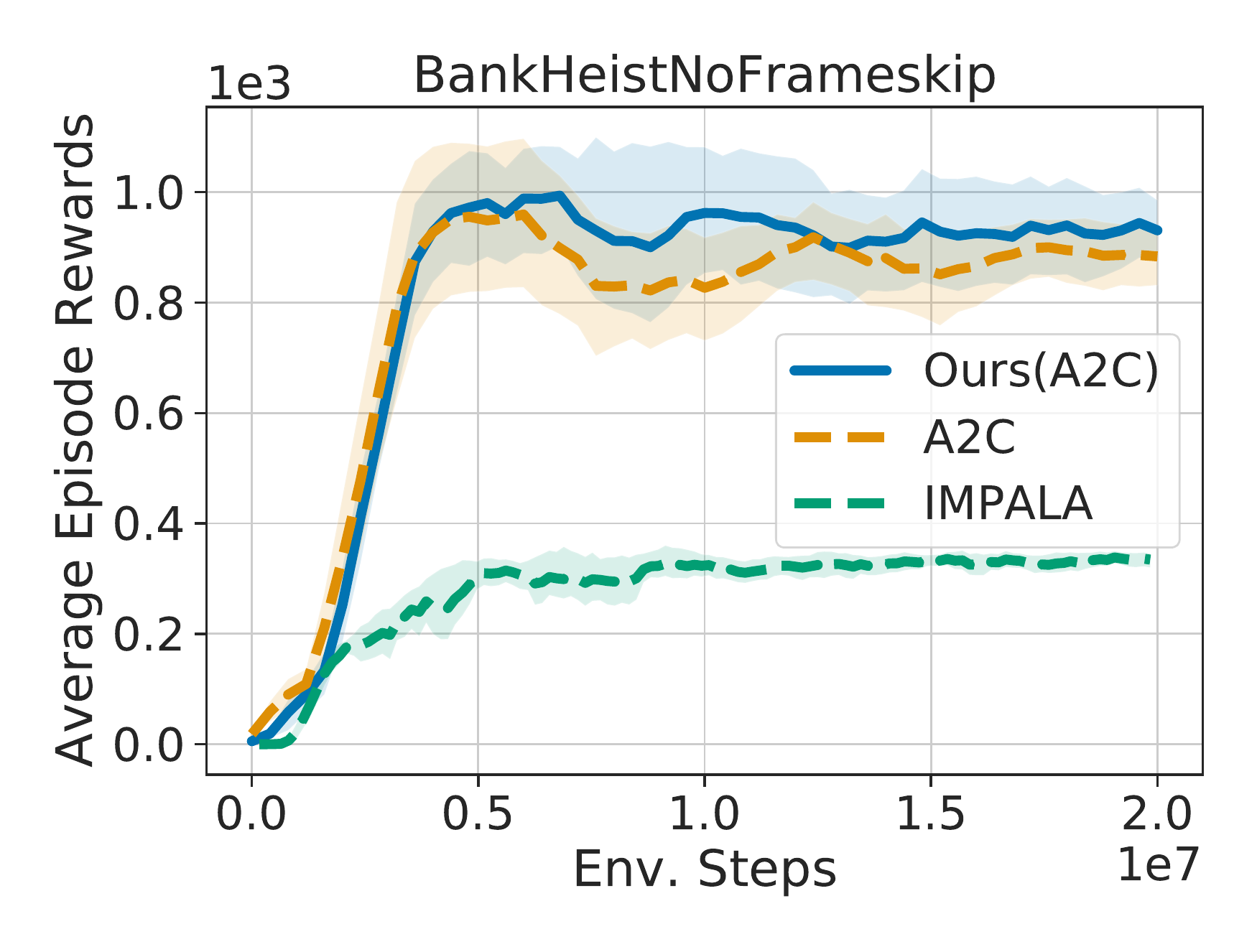}

&\hsp{-0.6cm}
\includegraphics[width=0.33\textwidth]{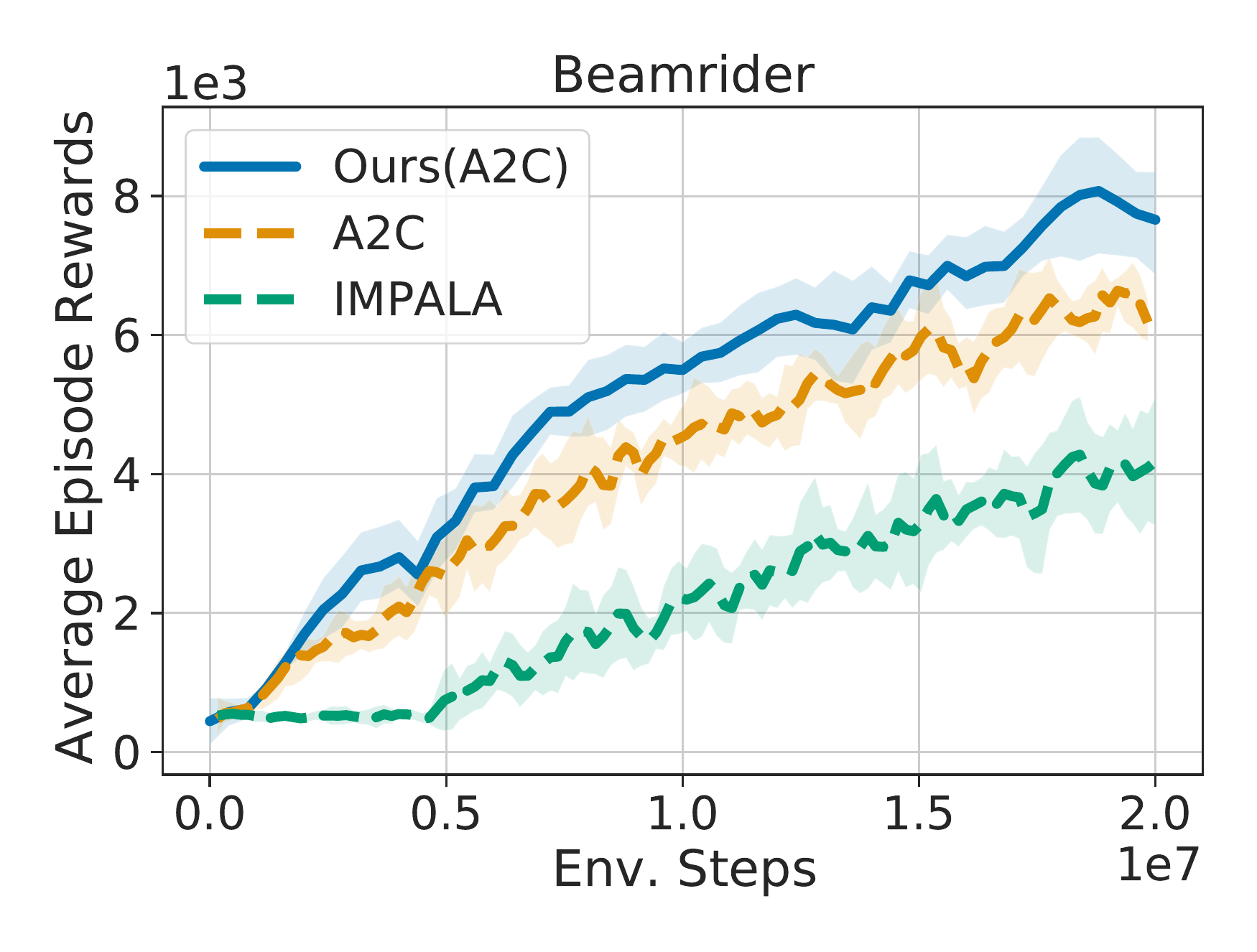}

&\hsp{-0.6cm}
\includegraphics[width=0.33\textwidth]{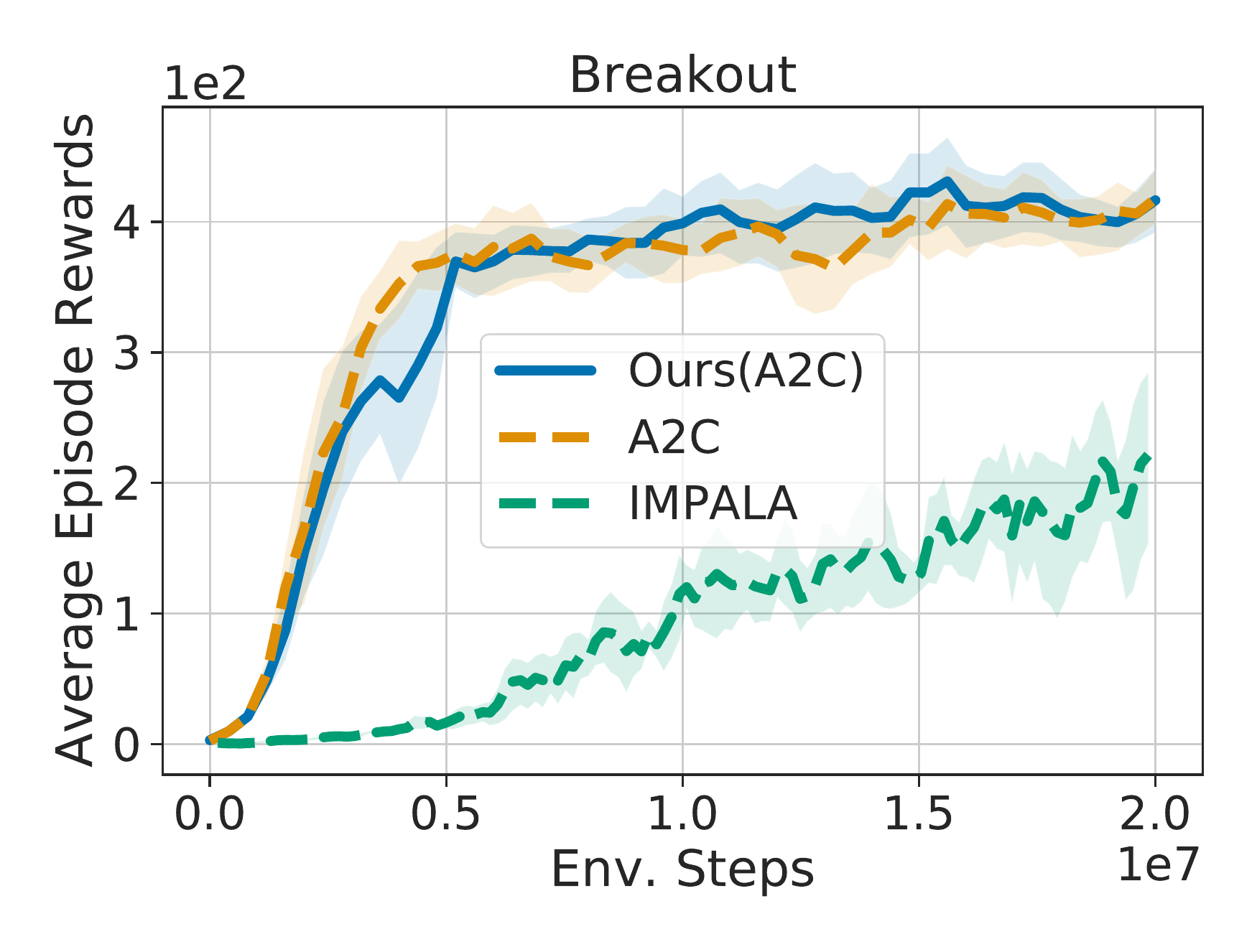}
\\

\includegraphics[width=0.33\textwidth]{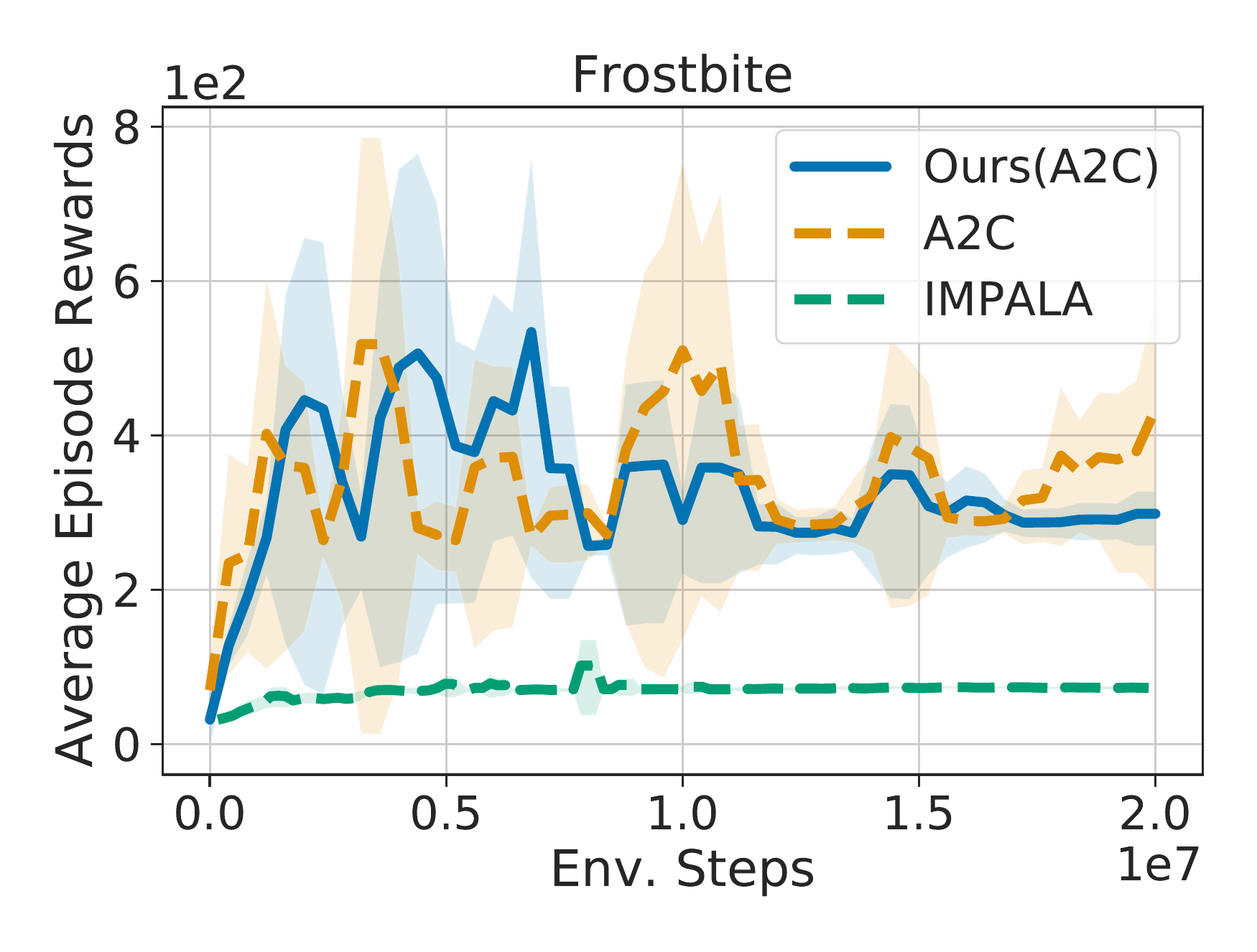}
&\hsp{-0.6cm}
\includegraphics[width=0.33\textwidth]{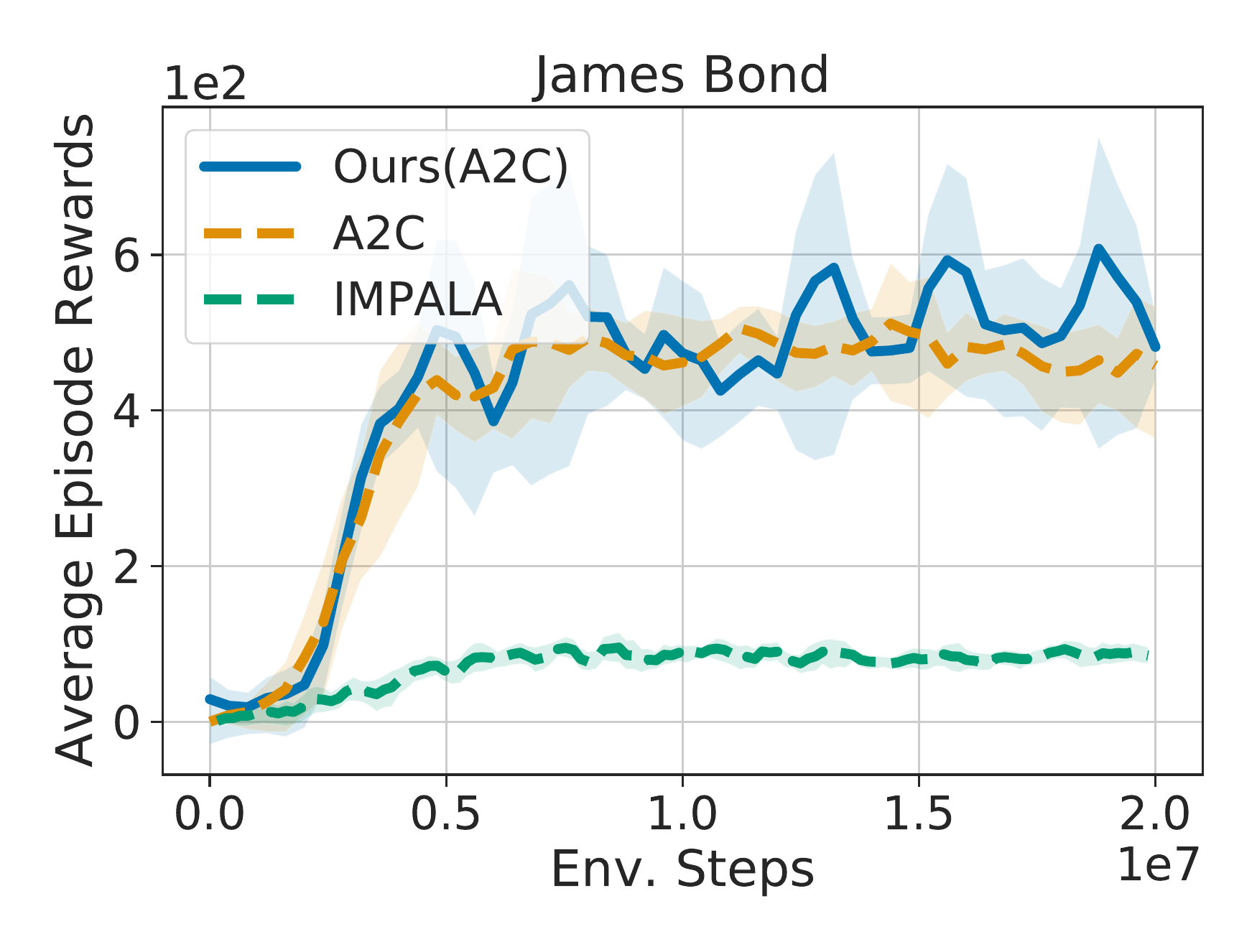}
&\hsp{-0.6cm}
\includegraphics[width=0.33\textwidth]{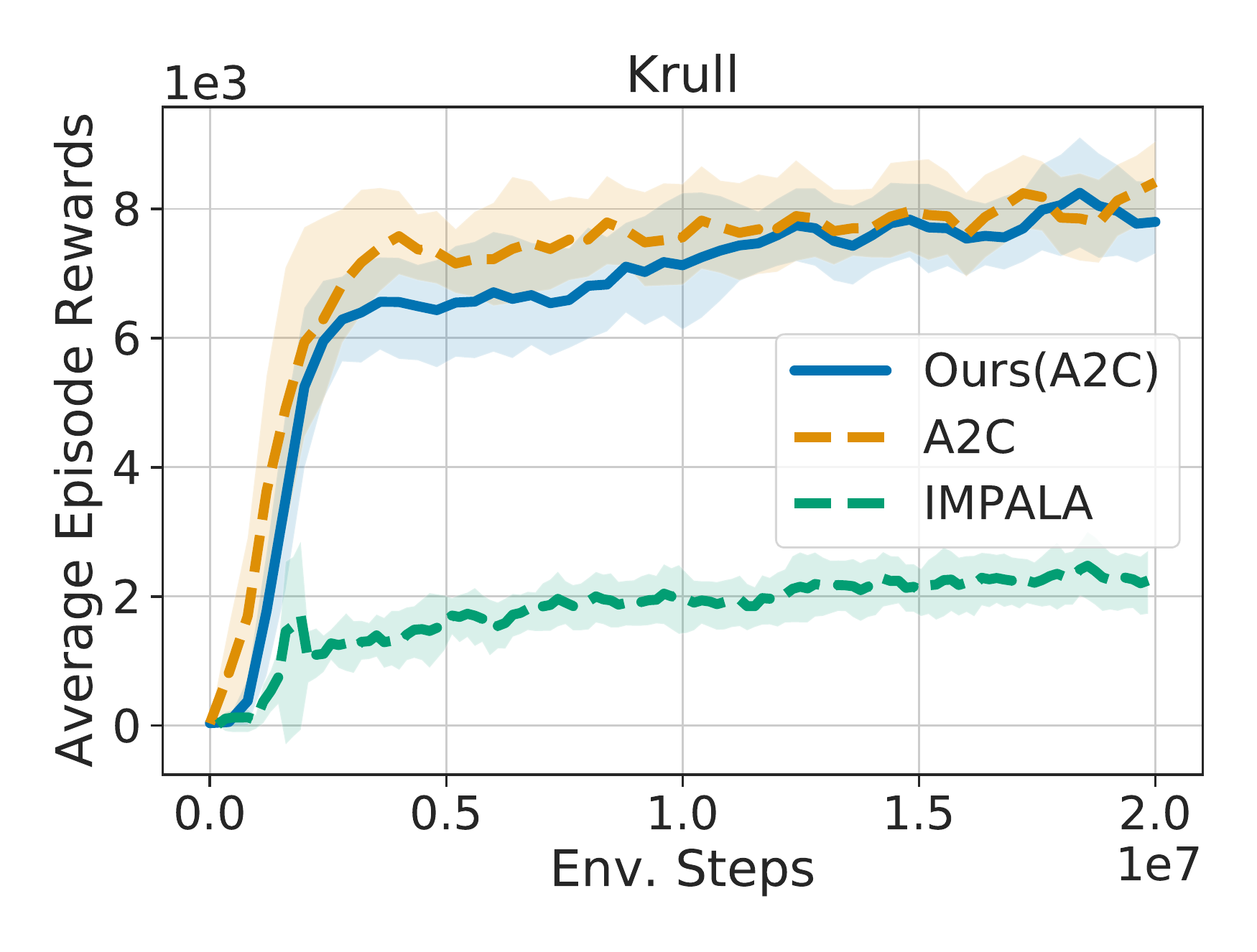}\\

\includegraphics[width=0.33\textwidth]{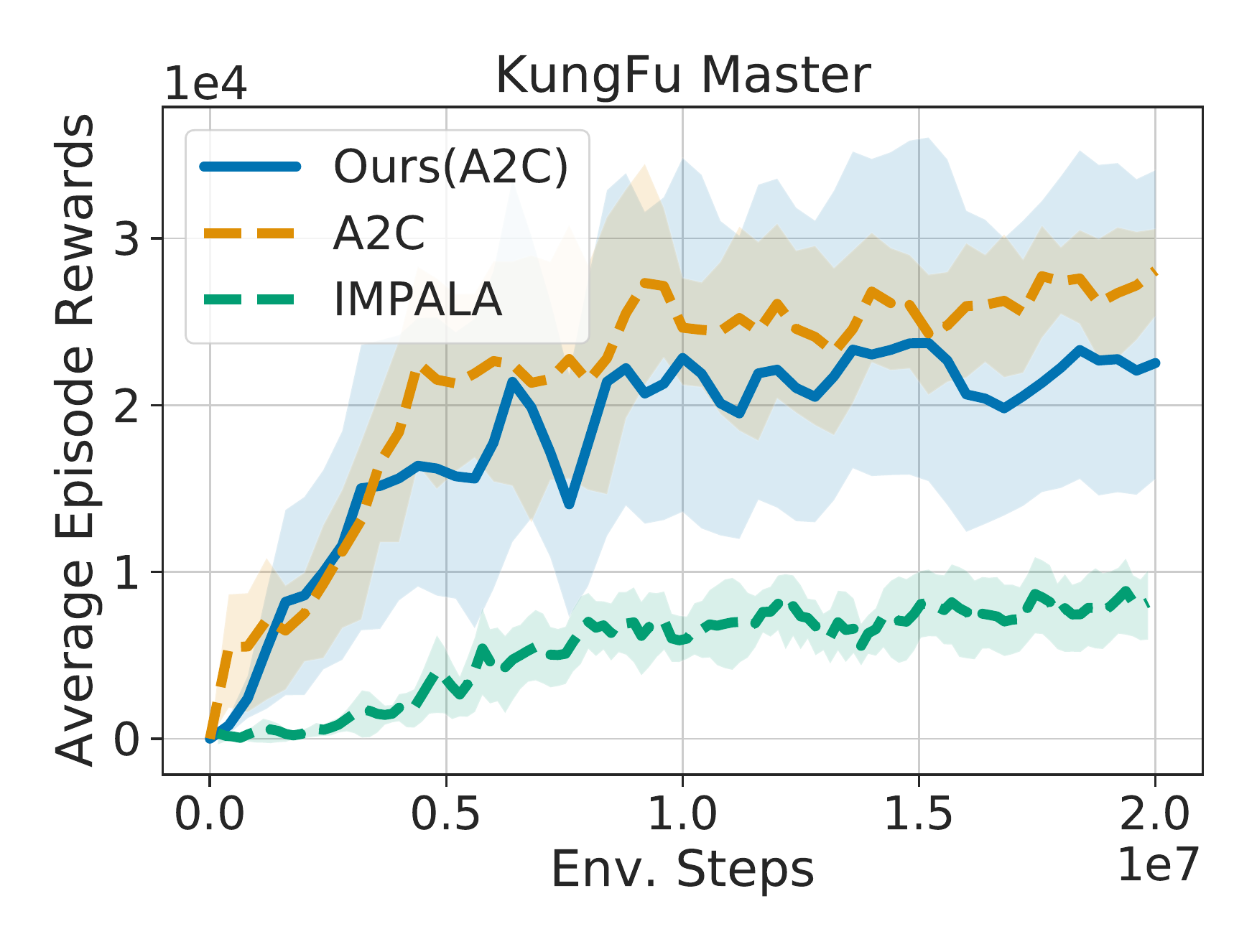}
&\hsp{-0.6cm}
\includegraphics[width=0.33\textwidth]{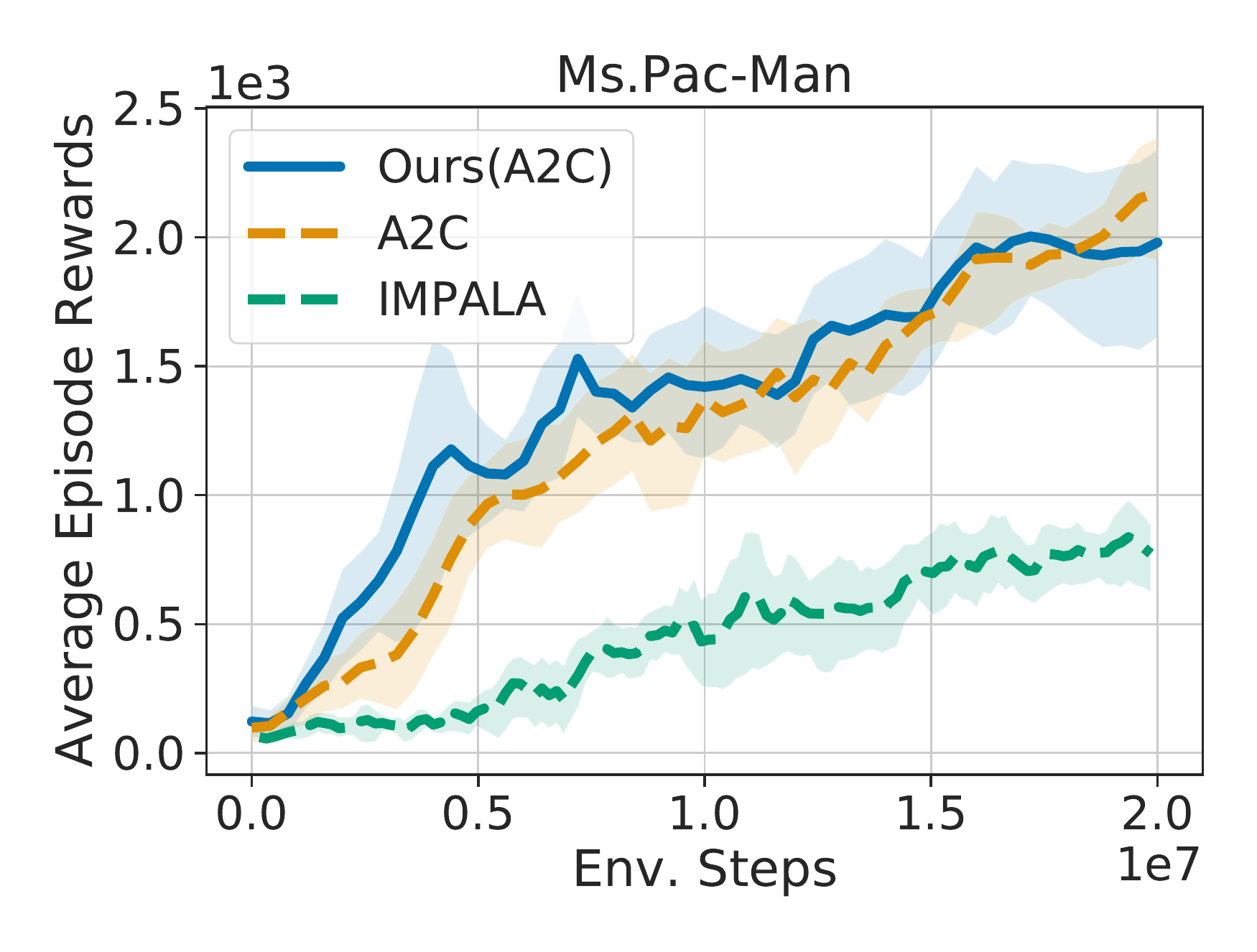}
&\hsp{-0.6cm}
\includegraphics[width=0.33\textwidth]{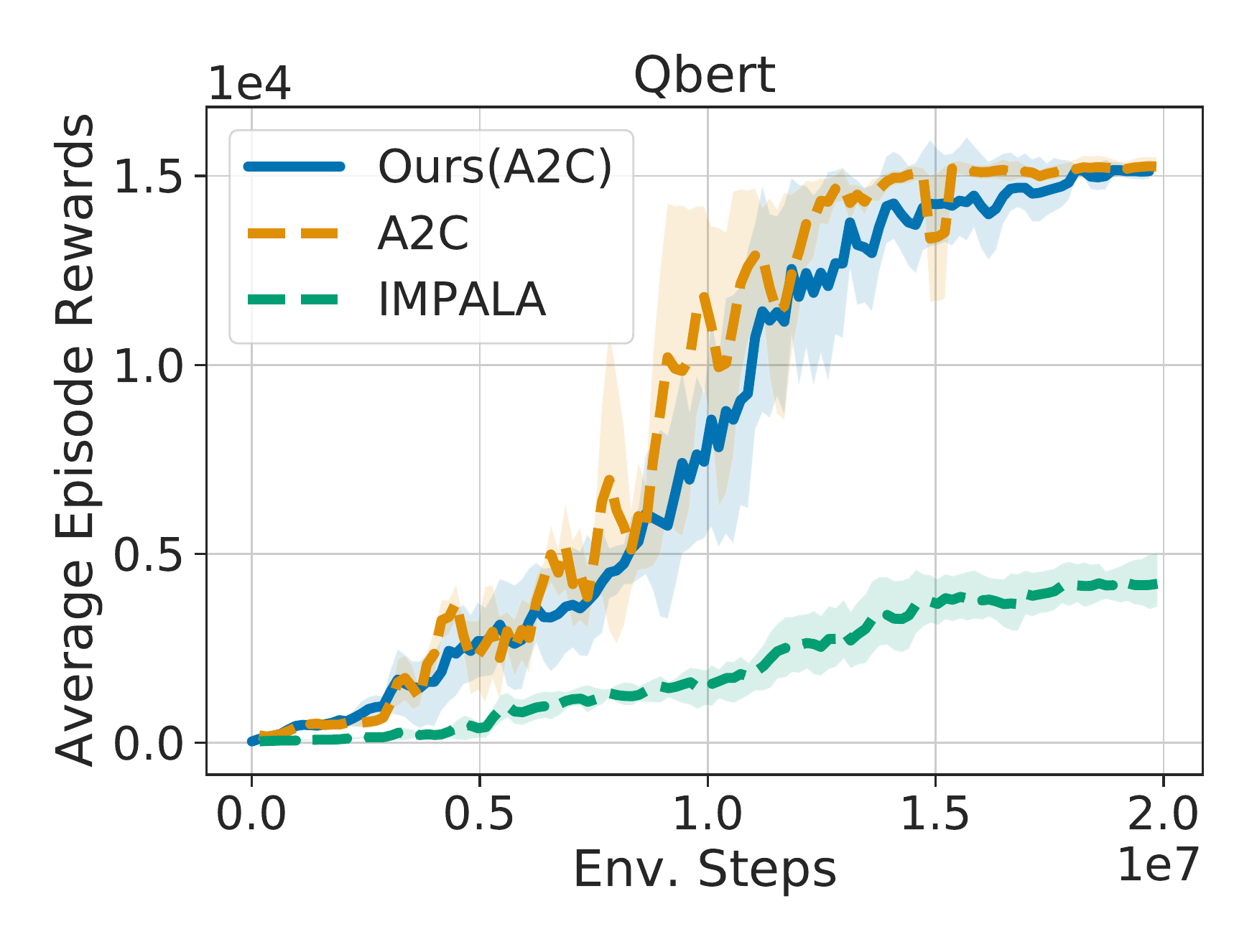}\\

\includegraphics[width=0.33\textwidth]{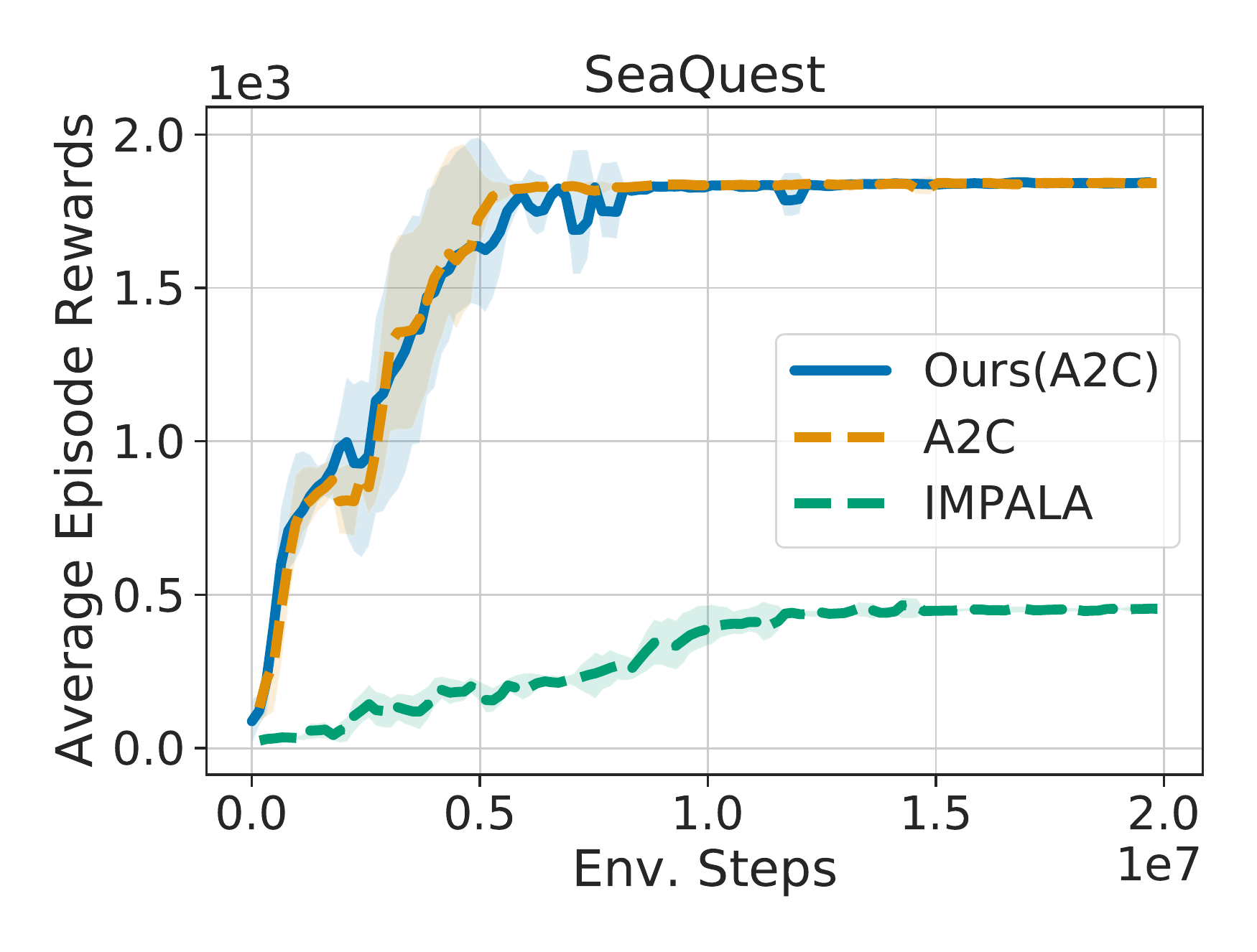}
&\hsp{-0.6cm}
\includegraphics[width=0.33\textwidth]{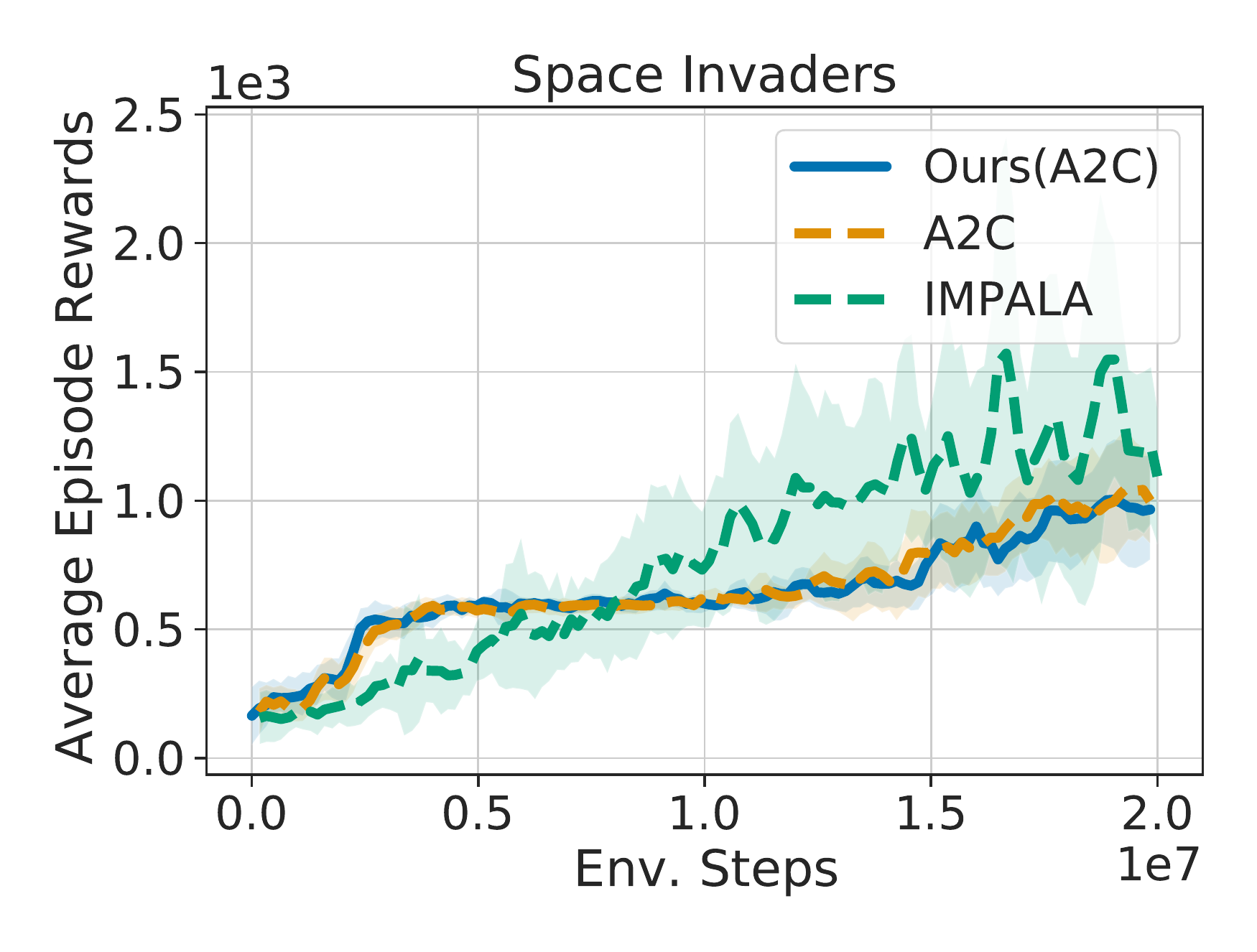}
&\hsp{-0.6cm}
\includegraphics[width=0.33\textwidth]{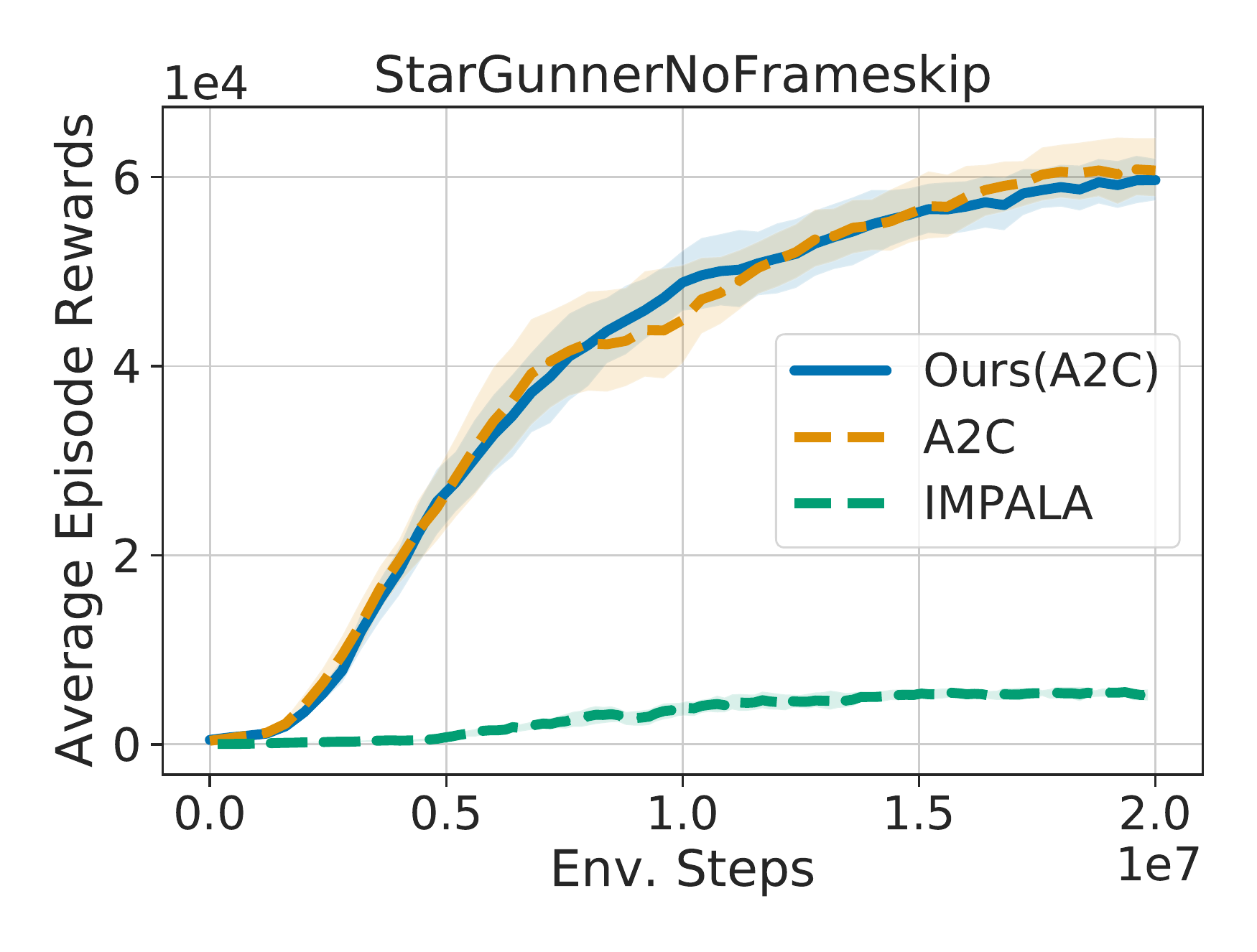}
\end{tabular}
\caption{Atari: {\textbf{Environment step}} versus reward.} 
\label{fig:atari_step_plot}

\end{figure}
\begin{figure}[t]
\centering
\begin{tabular}{cccc}
\includegraphics[width=0.33\textwidth]{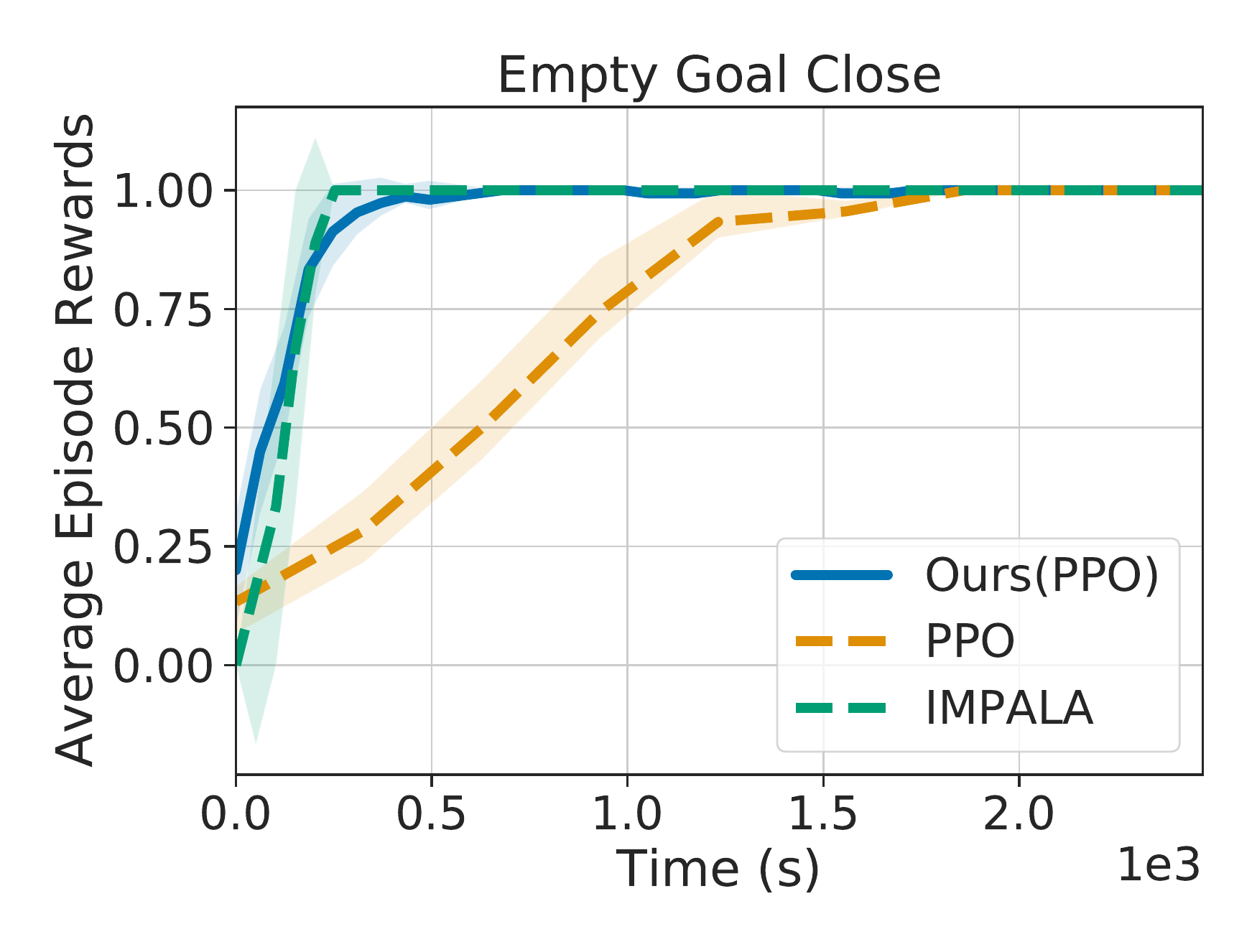}
&\hsp{-0.6cm}
\includegraphics[width=0.33\textwidth]{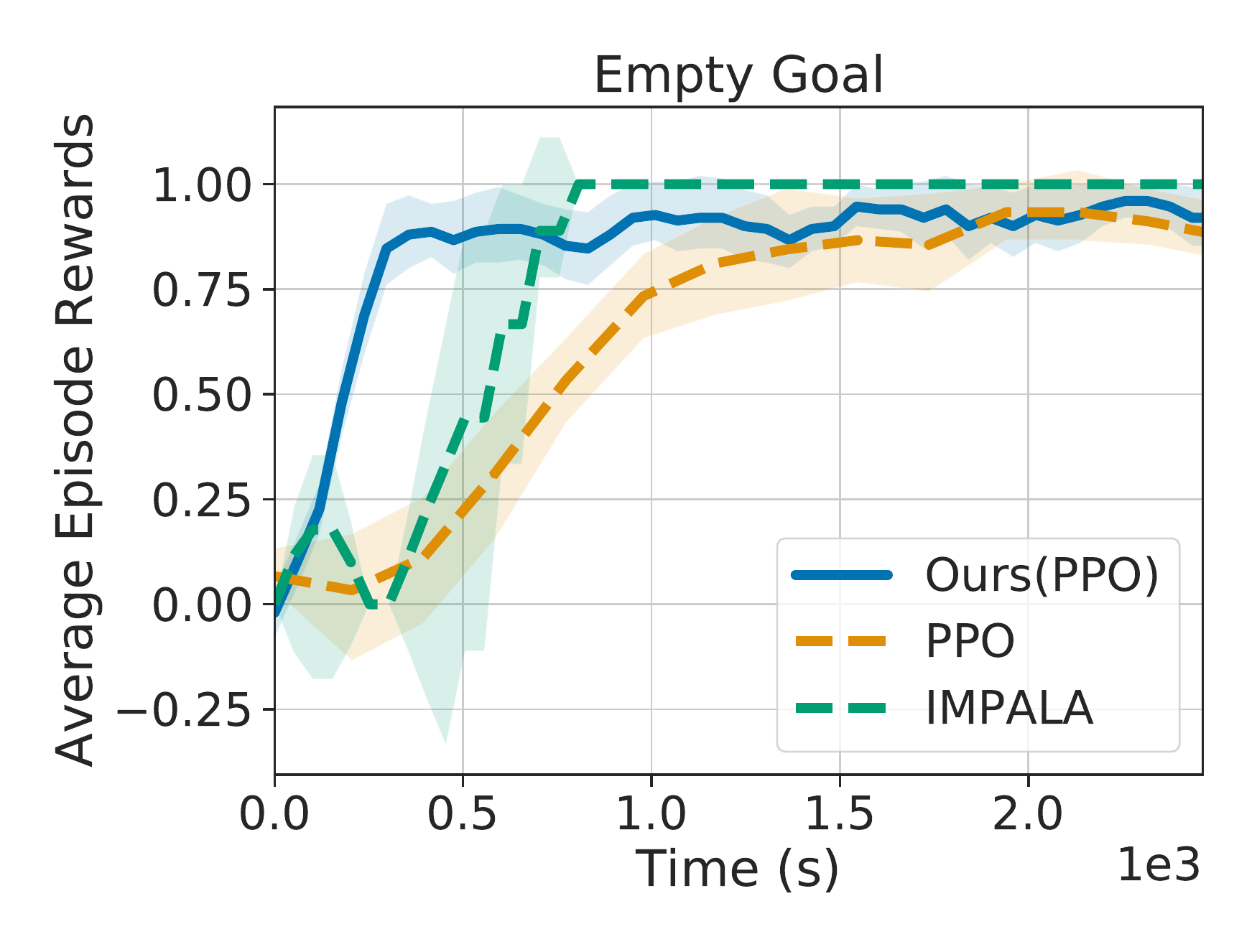}
&\hsp{-0.6cm}
\includegraphics[width=0.33\textwidth]{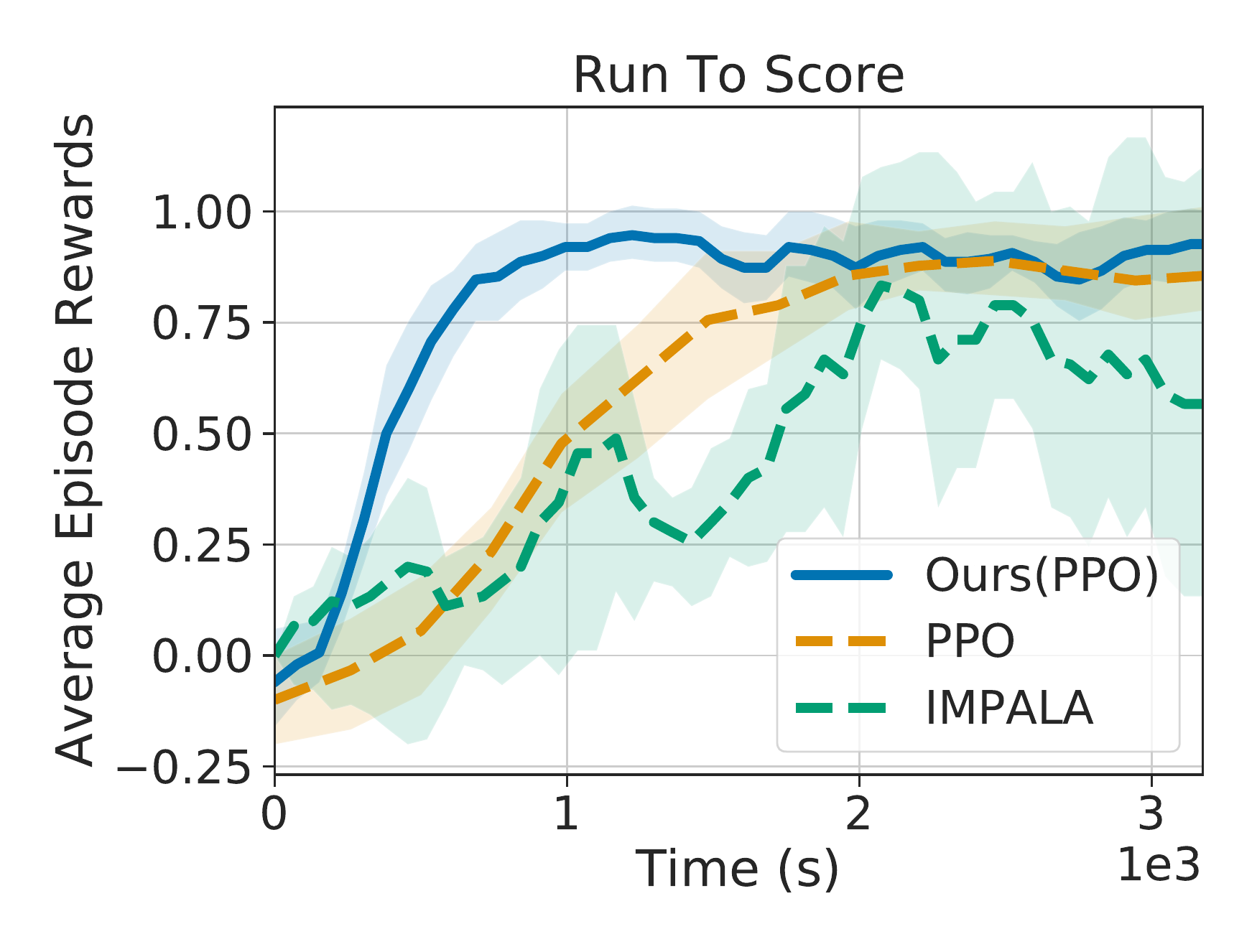}\\

\includegraphics[width=0.33\textwidth]{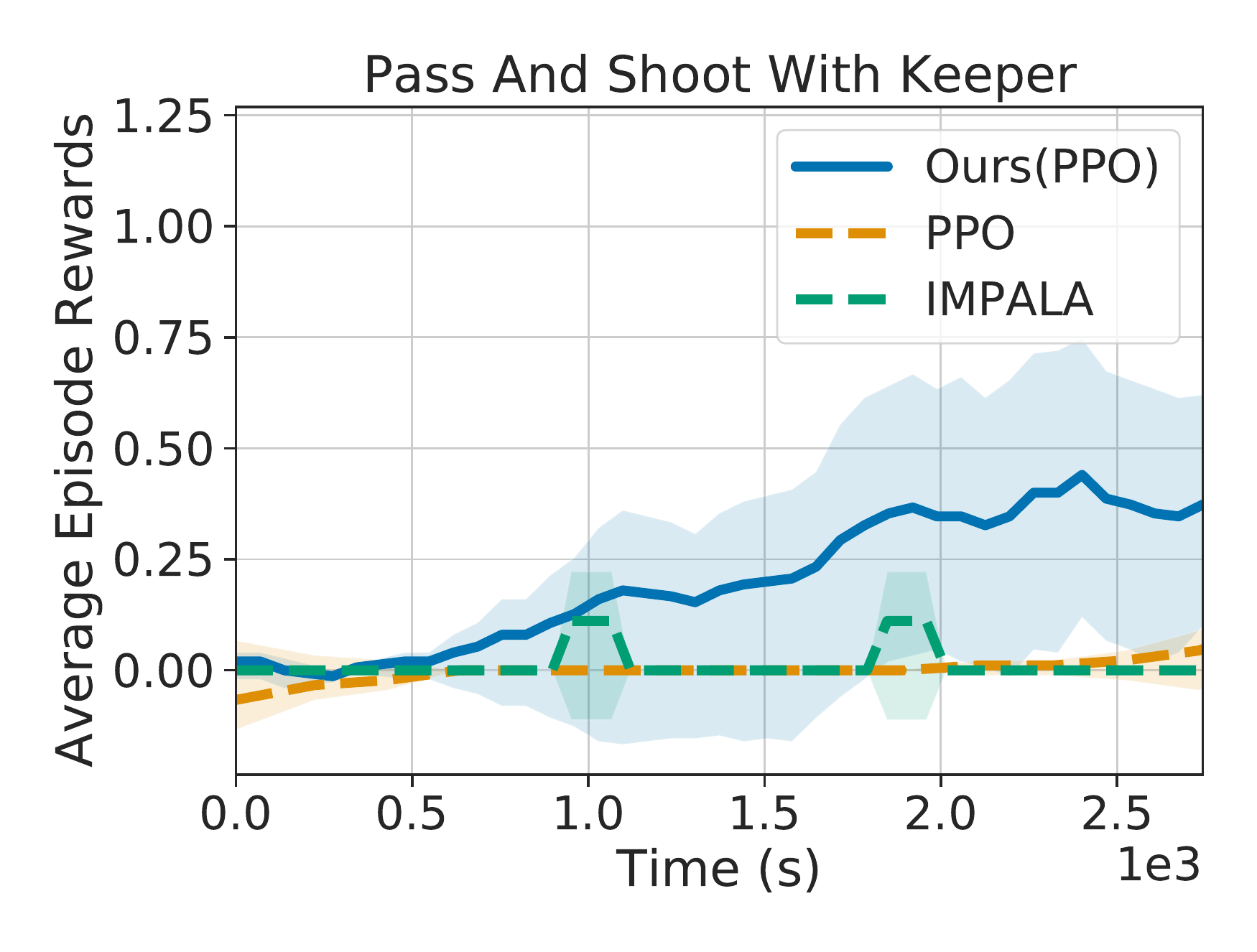}
&\hsp{-0.6cm}
\includegraphics[width=0.33\textwidth]{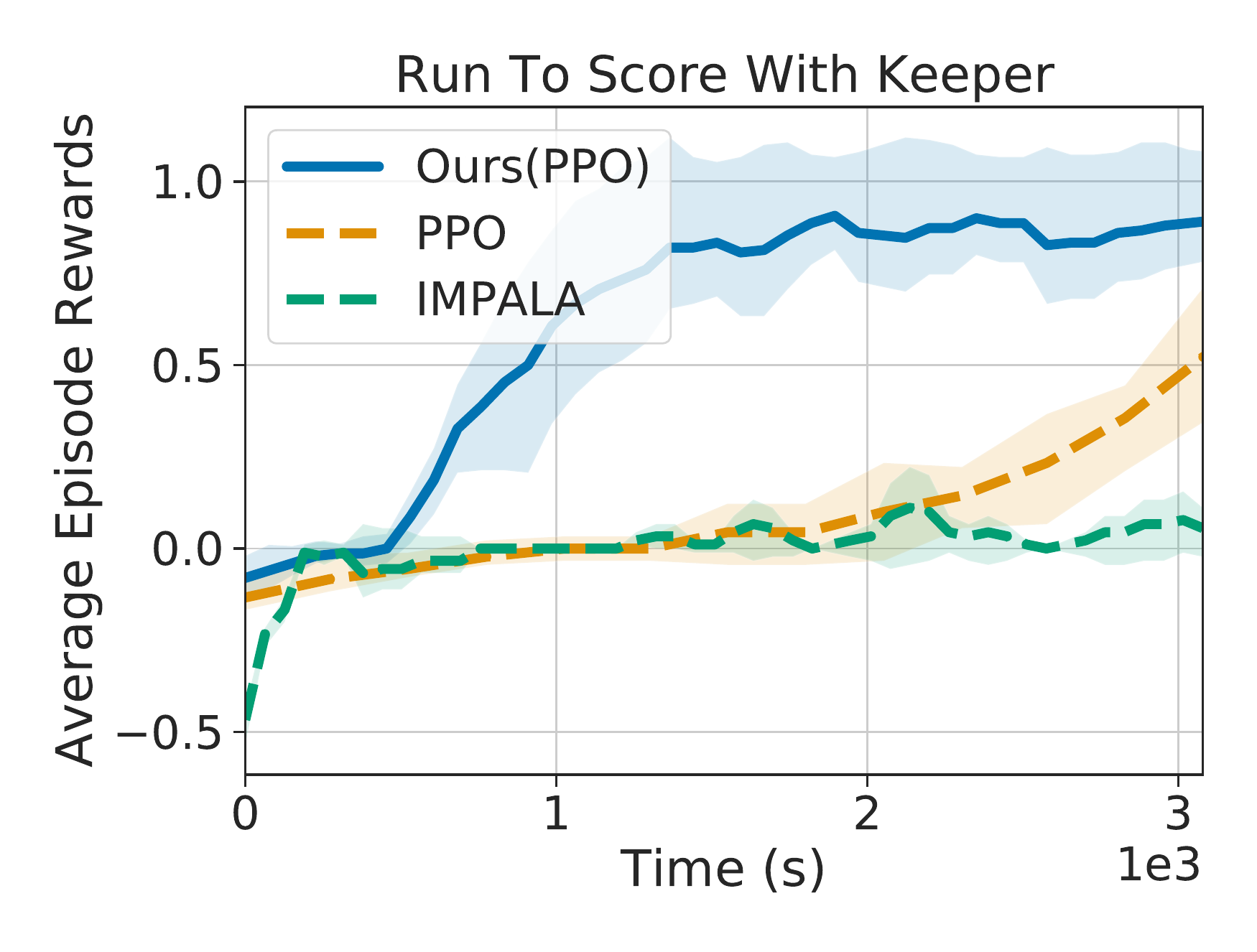}
&\hsp{-0.6cm}
\includegraphics[width=0.33\textwidth]{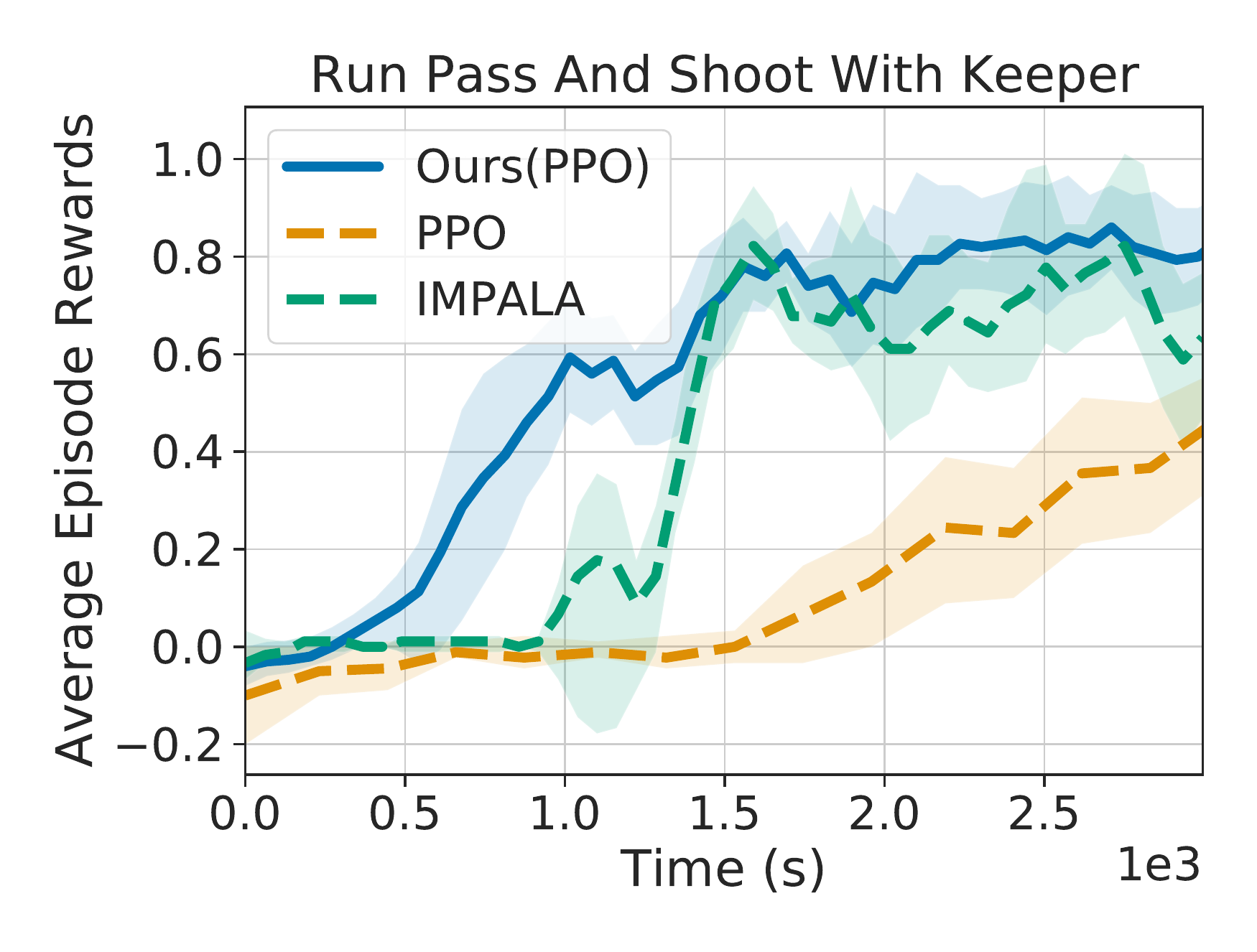}\\

\includegraphics[width=0.33\textwidth]{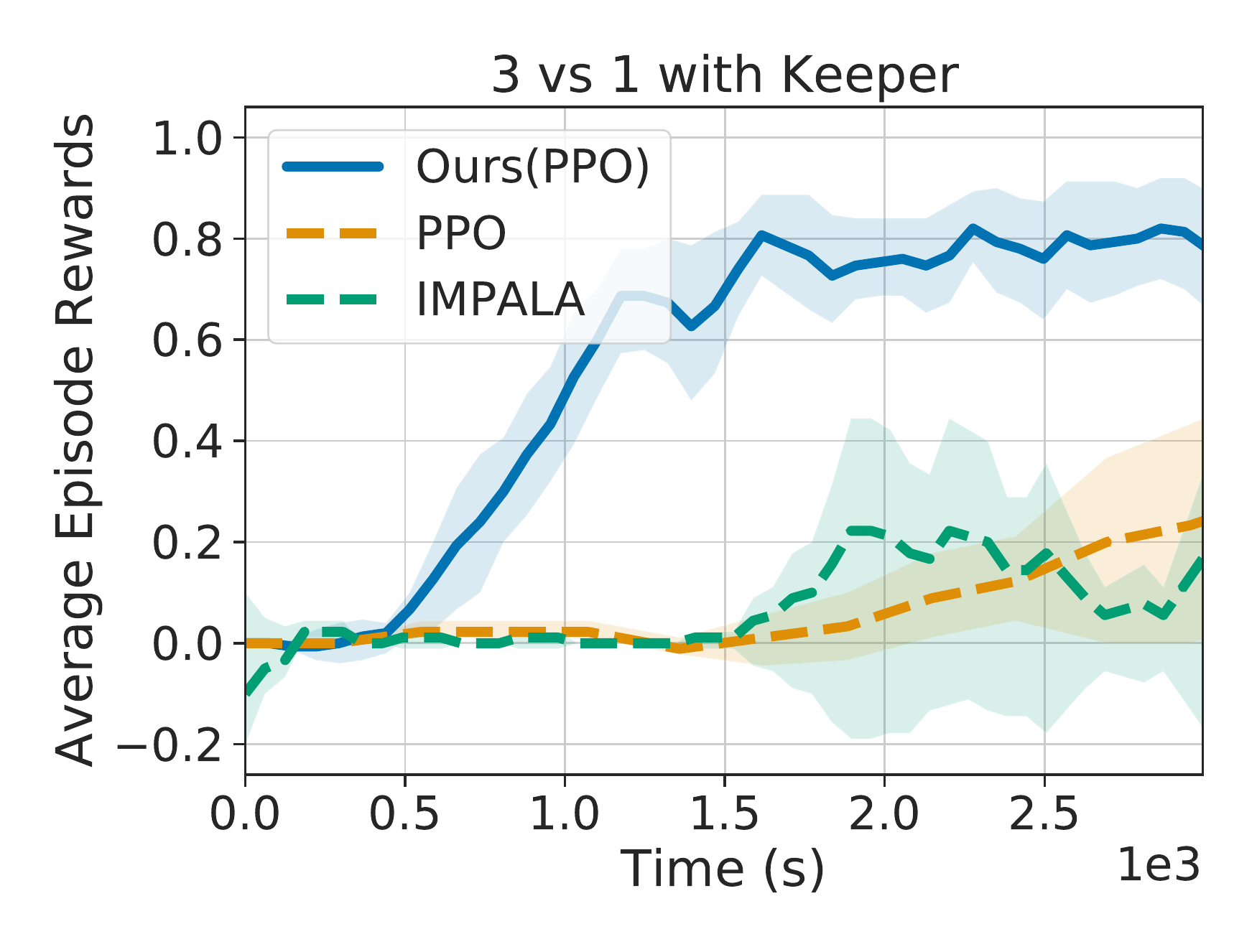}
&\hsp{-0.6cm}
\includegraphics[width=0.33\textwidth]{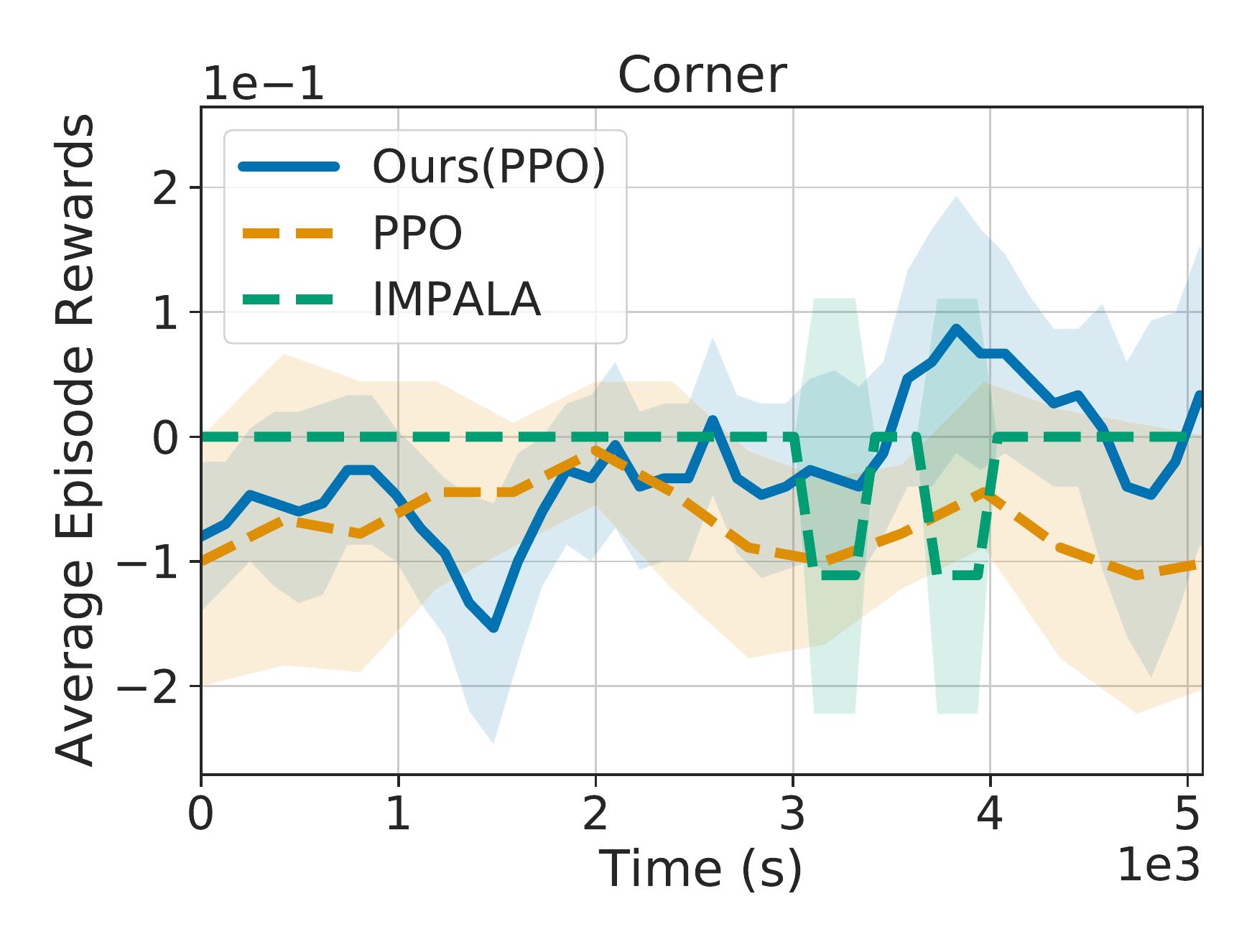}
&\hsp{-0.6cm}
\includegraphics[width=0.33\textwidth]{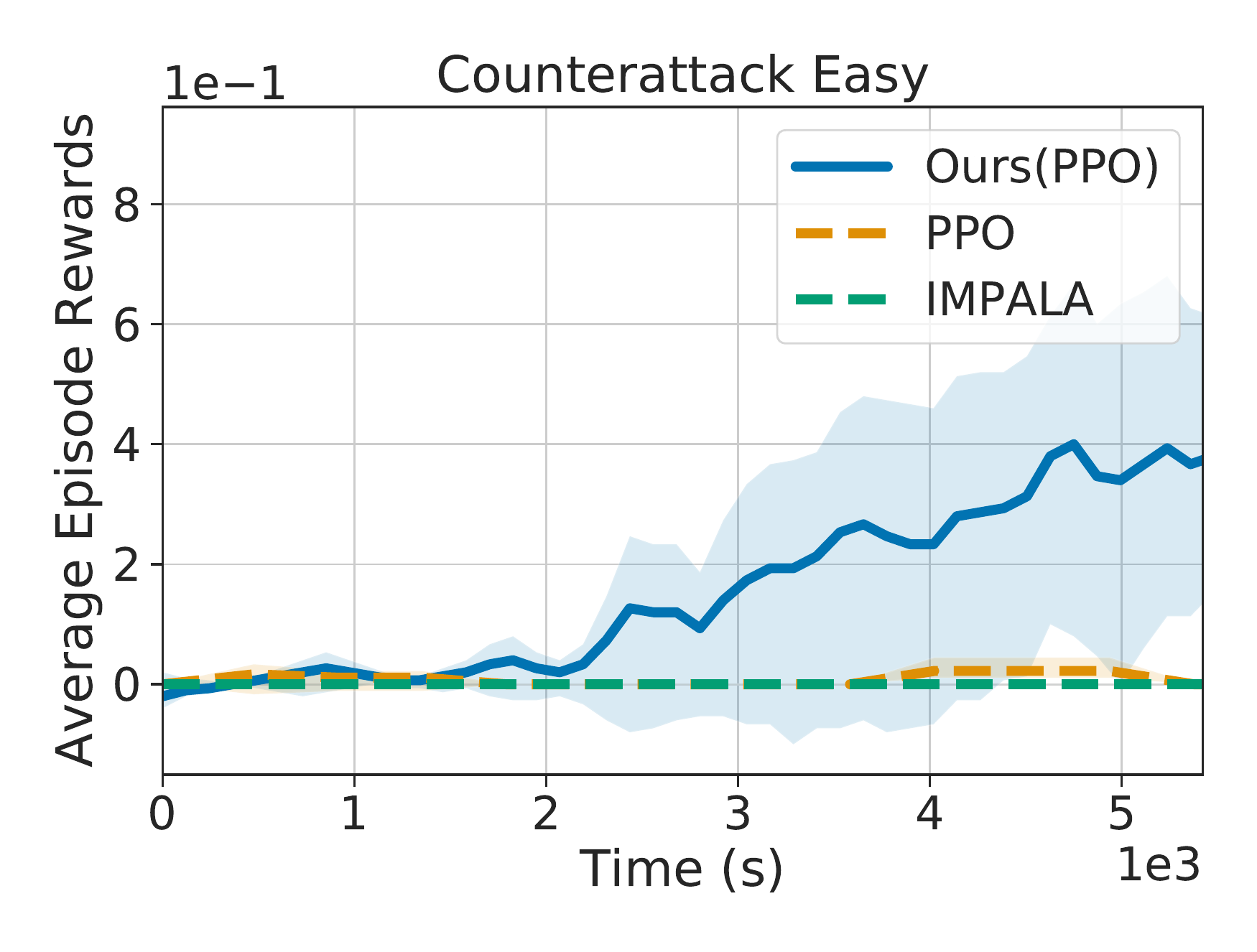}\\

\includegraphics[width=0.33\textwidth]{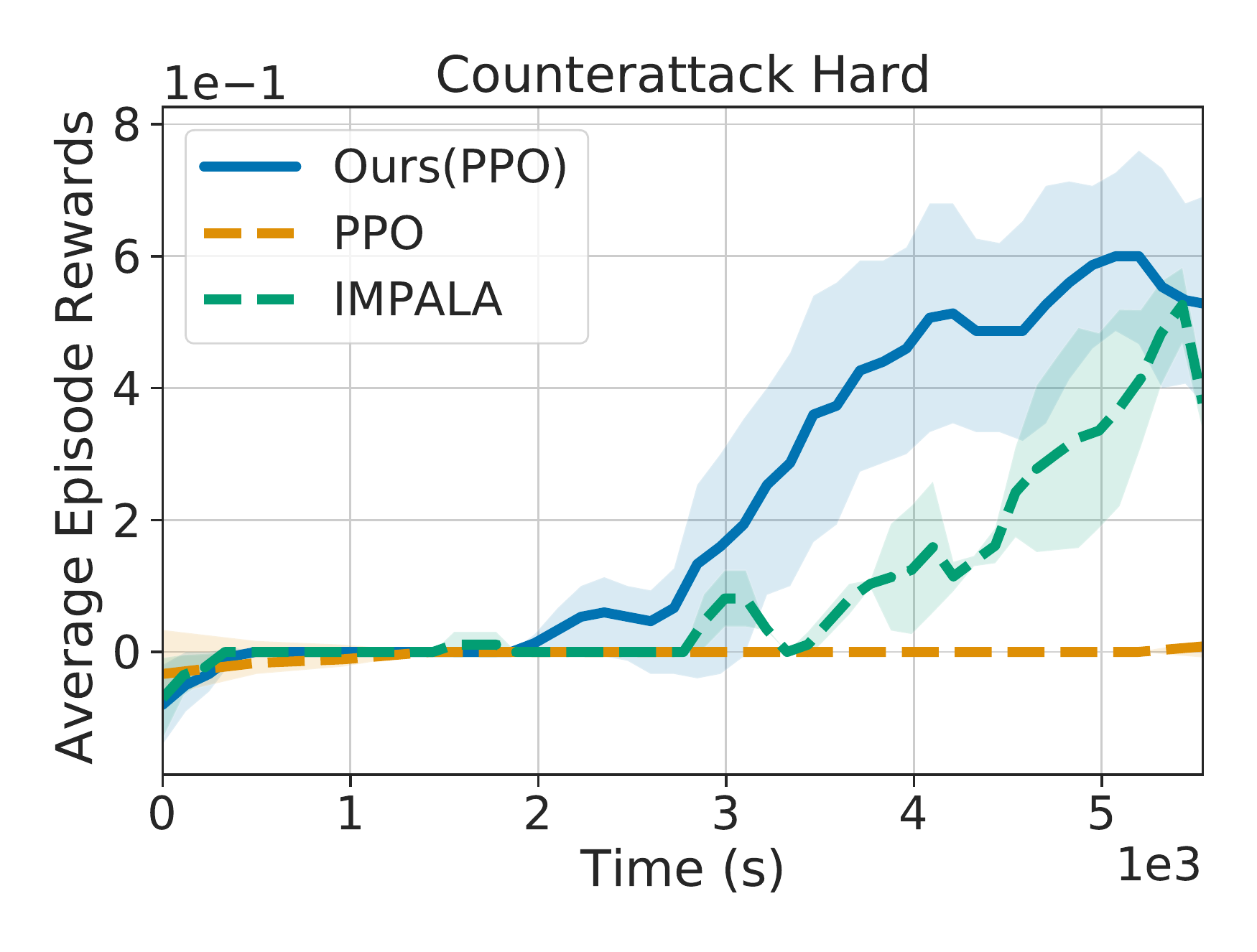}
&\hsp{-0.6cm}
\includegraphics[width=0.33\textwidth]{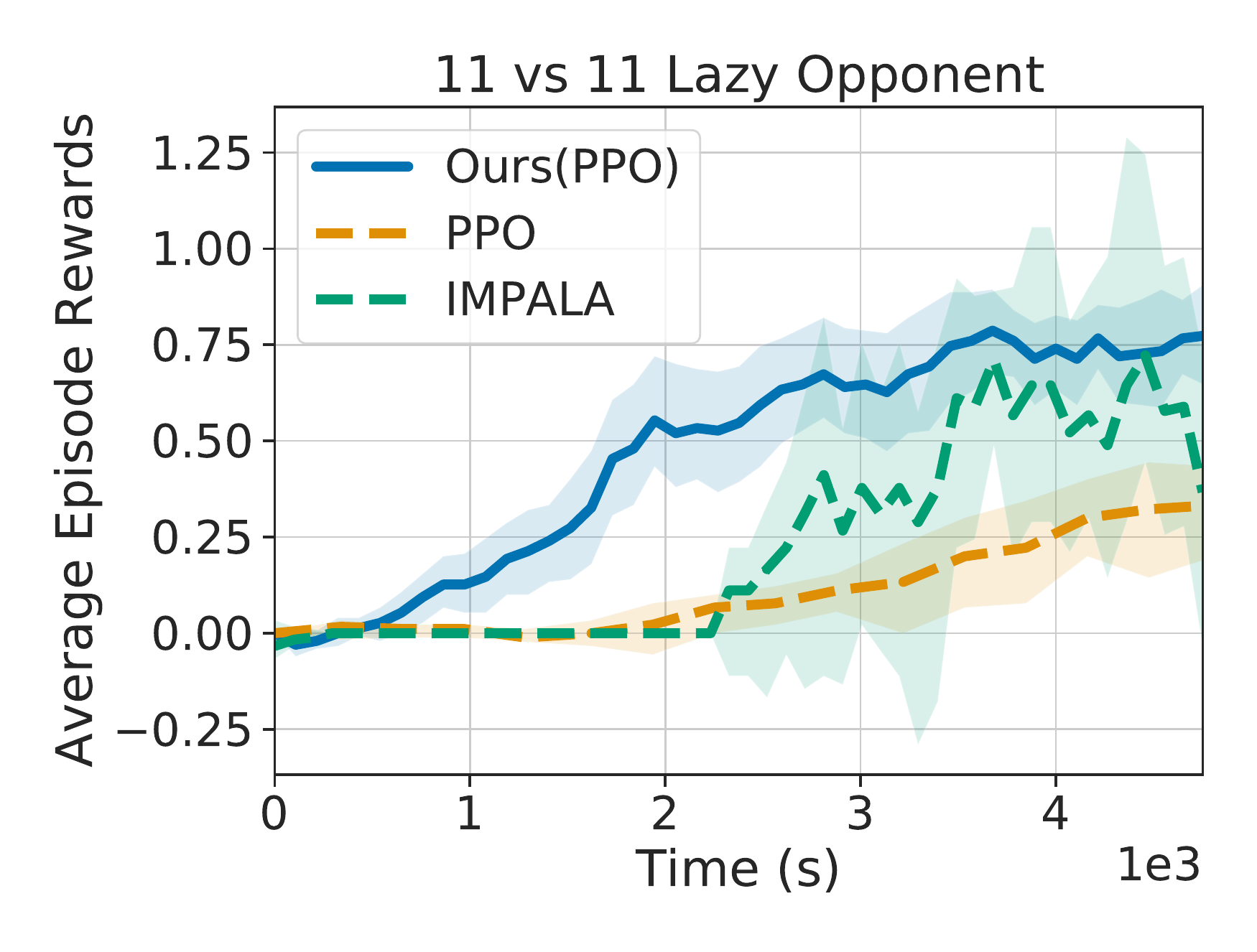}
&\hsp{-0.6cm}
\end{tabular}
\caption{GFootball: {\textbf{Time}} versus reward.} 
\label{fig:football_time_plot}
\end{figure}

\begin{figure*}[t]
\centering
\begin{tabular}{cccc}
\includegraphics[width=0.33\textwidth]{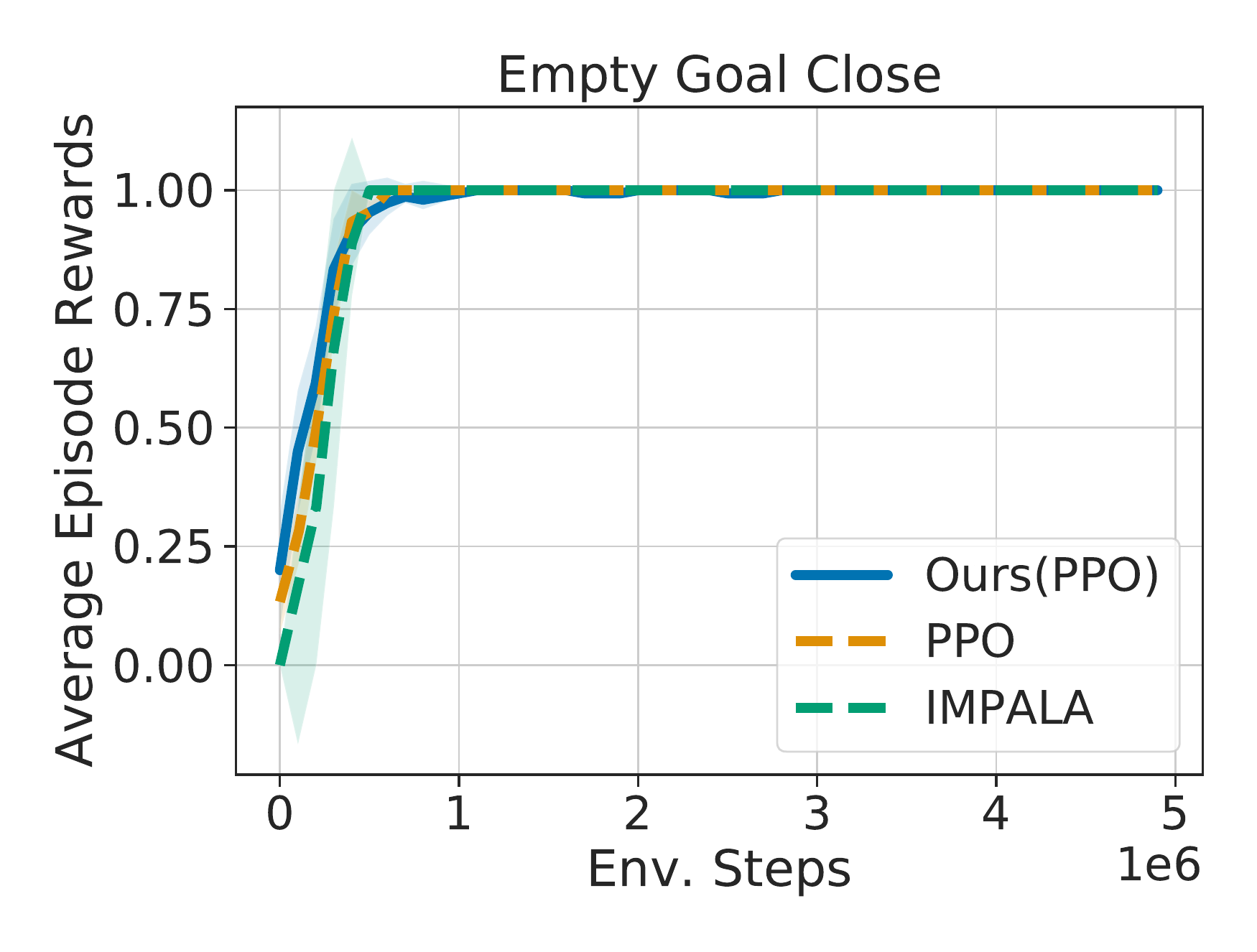}
&\hsp{-0.6cm}
\includegraphics[width=0.33\textwidth]{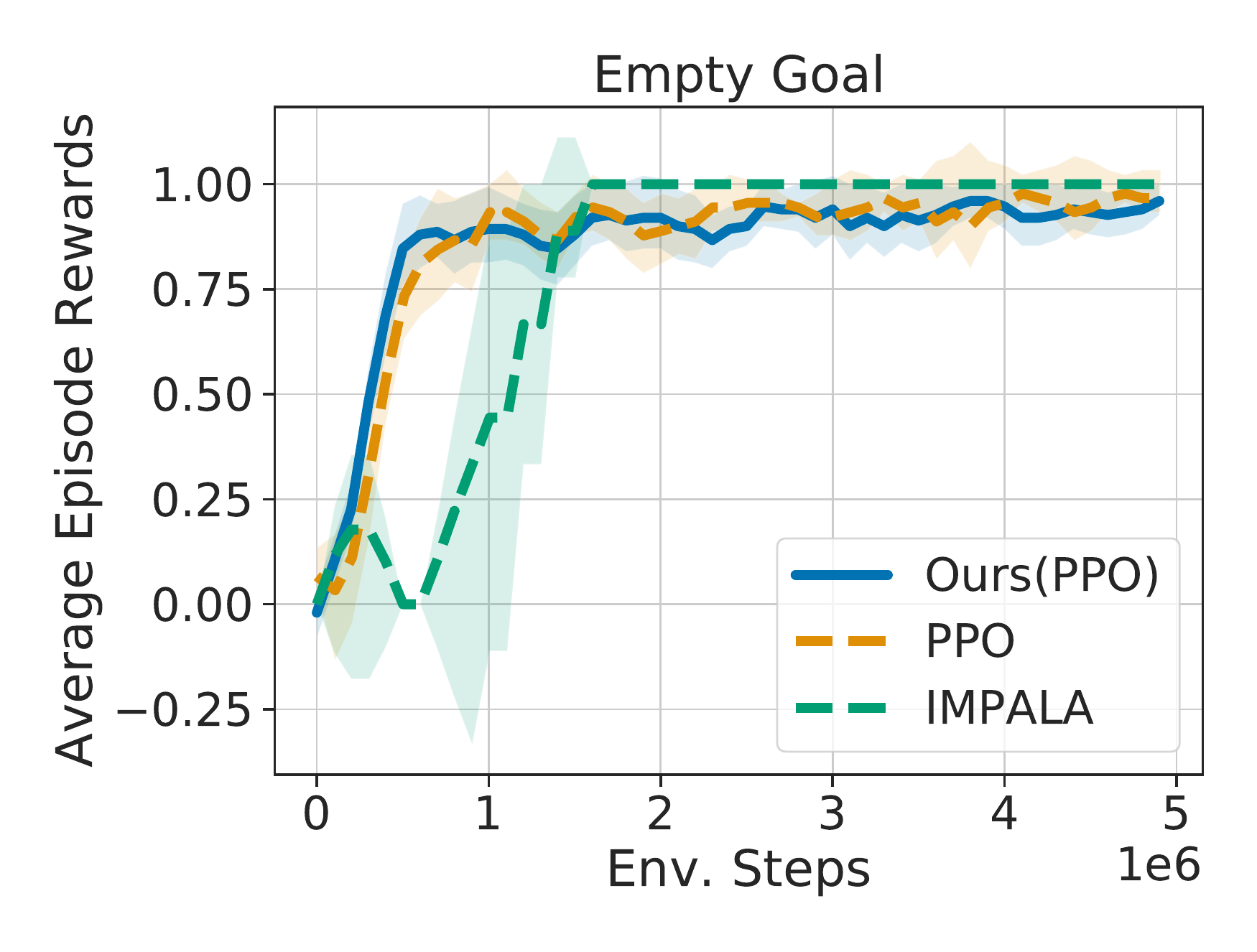}
&\hsp{-0.6cm}
\includegraphics[width=0.33\textwidth]{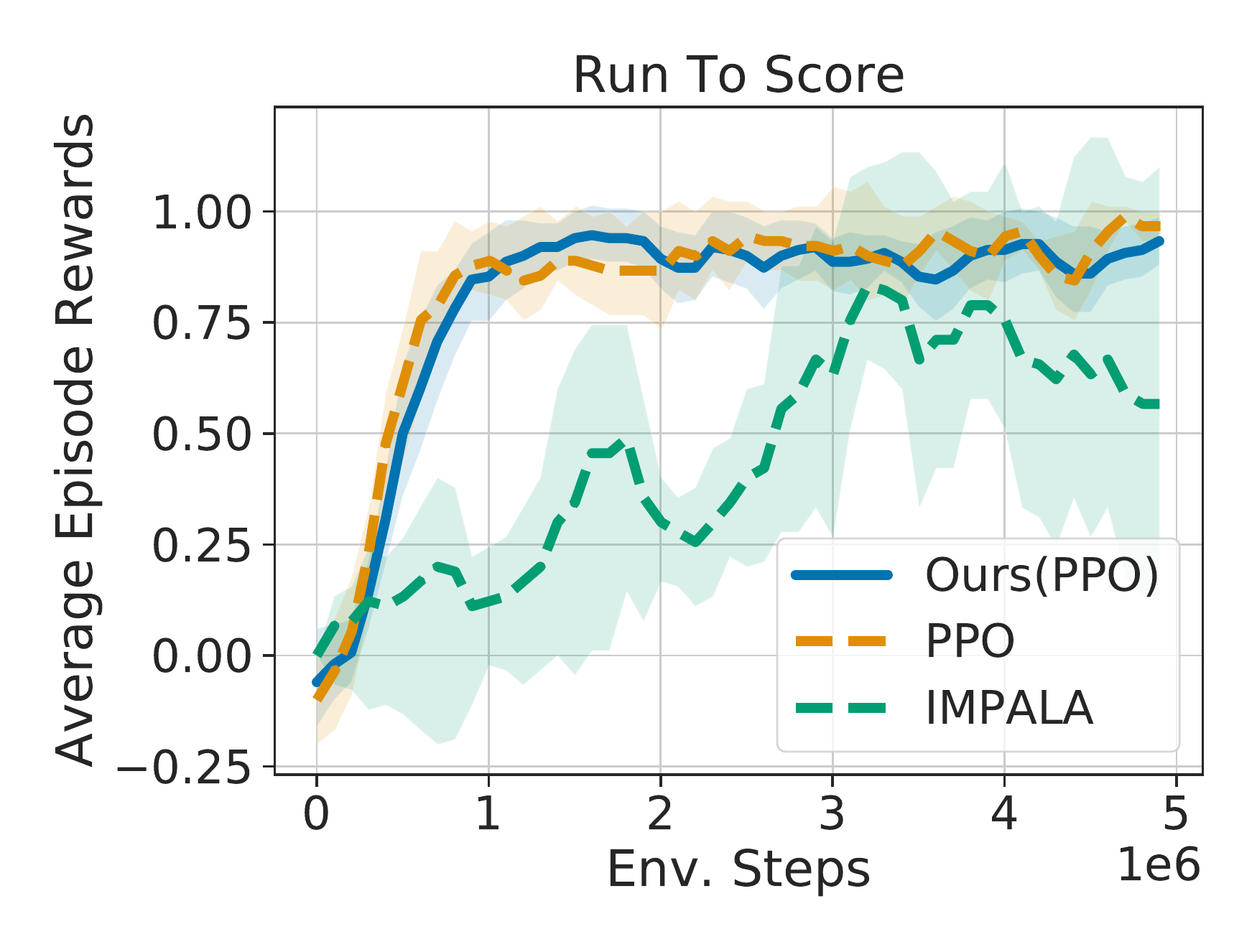}\\

\includegraphics[width=0.33\textwidth]{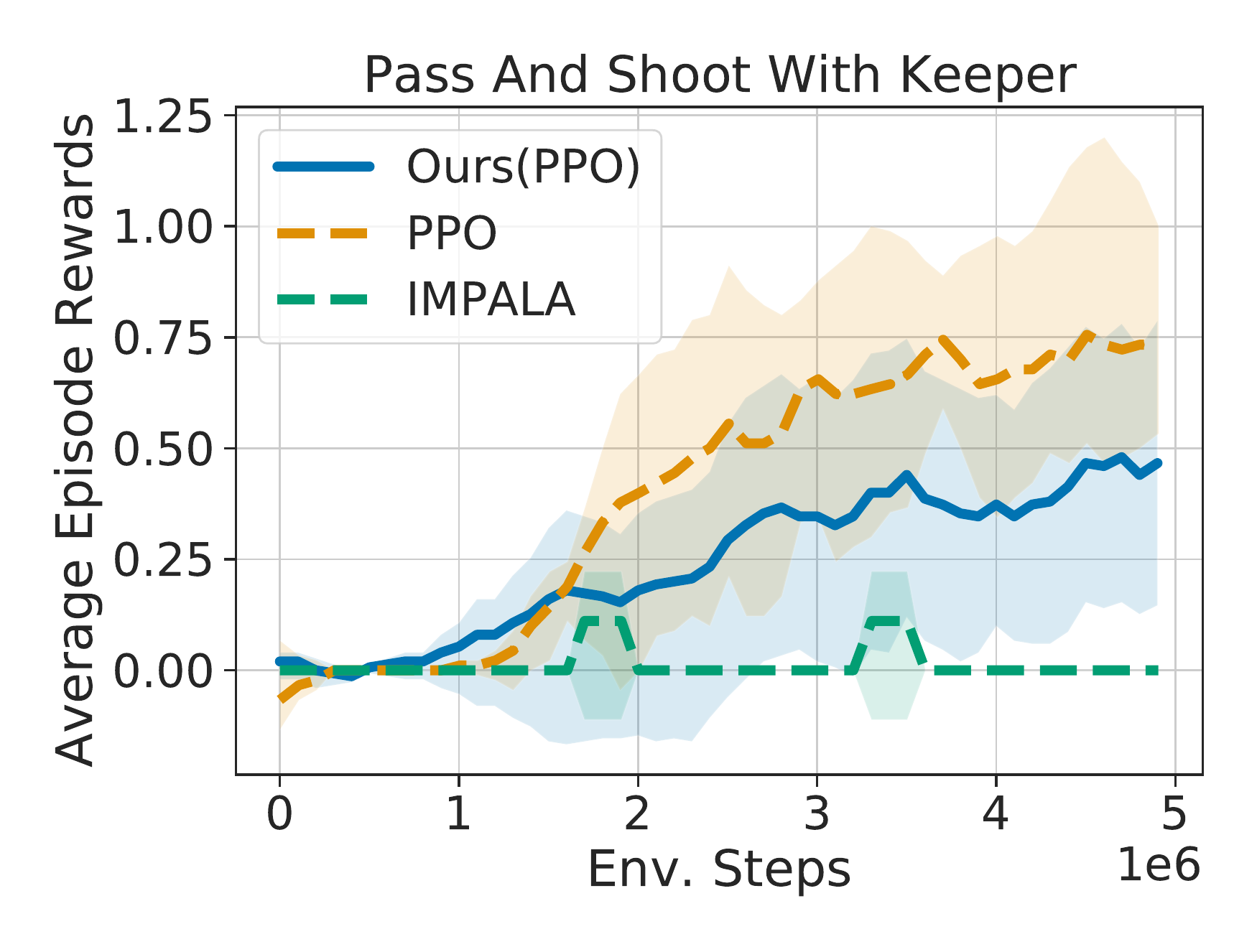}
&\hsp{-0.6cm}
\includegraphics[width=0.33\textwidth]{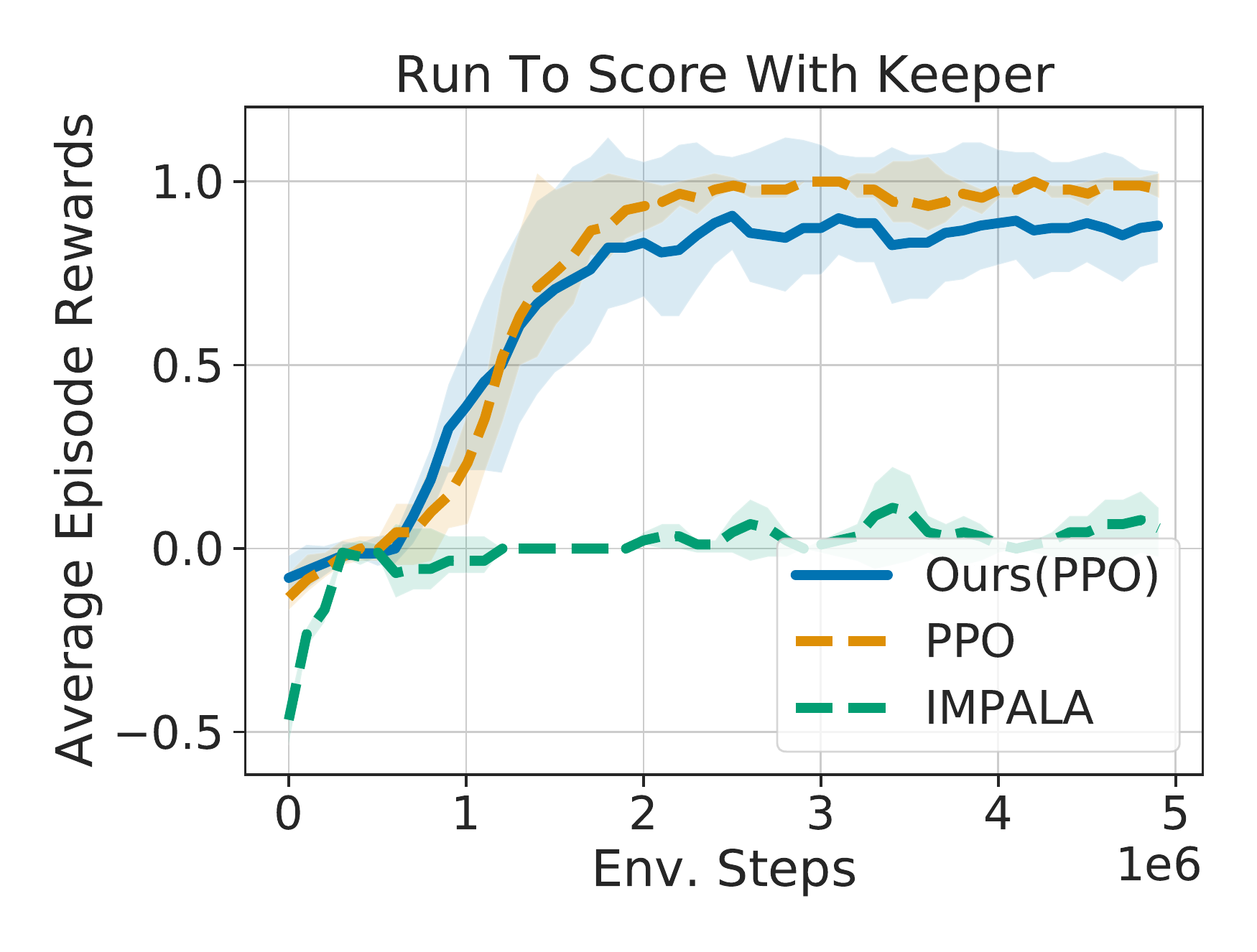}
&\hsp{-0.6cm}
\includegraphics[width=0.33\textwidth]{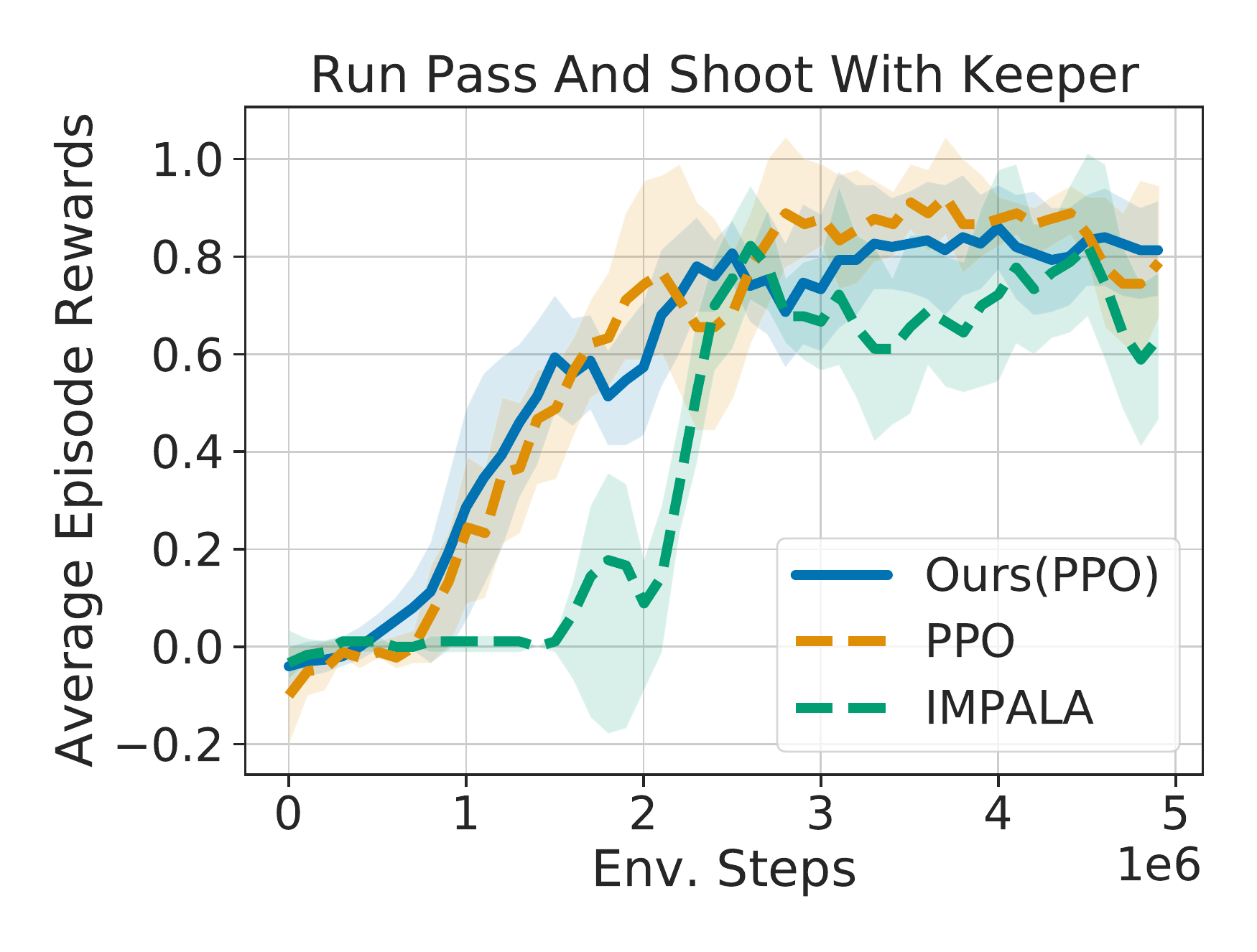}\\

\includegraphics[width=0.33\textwidth]{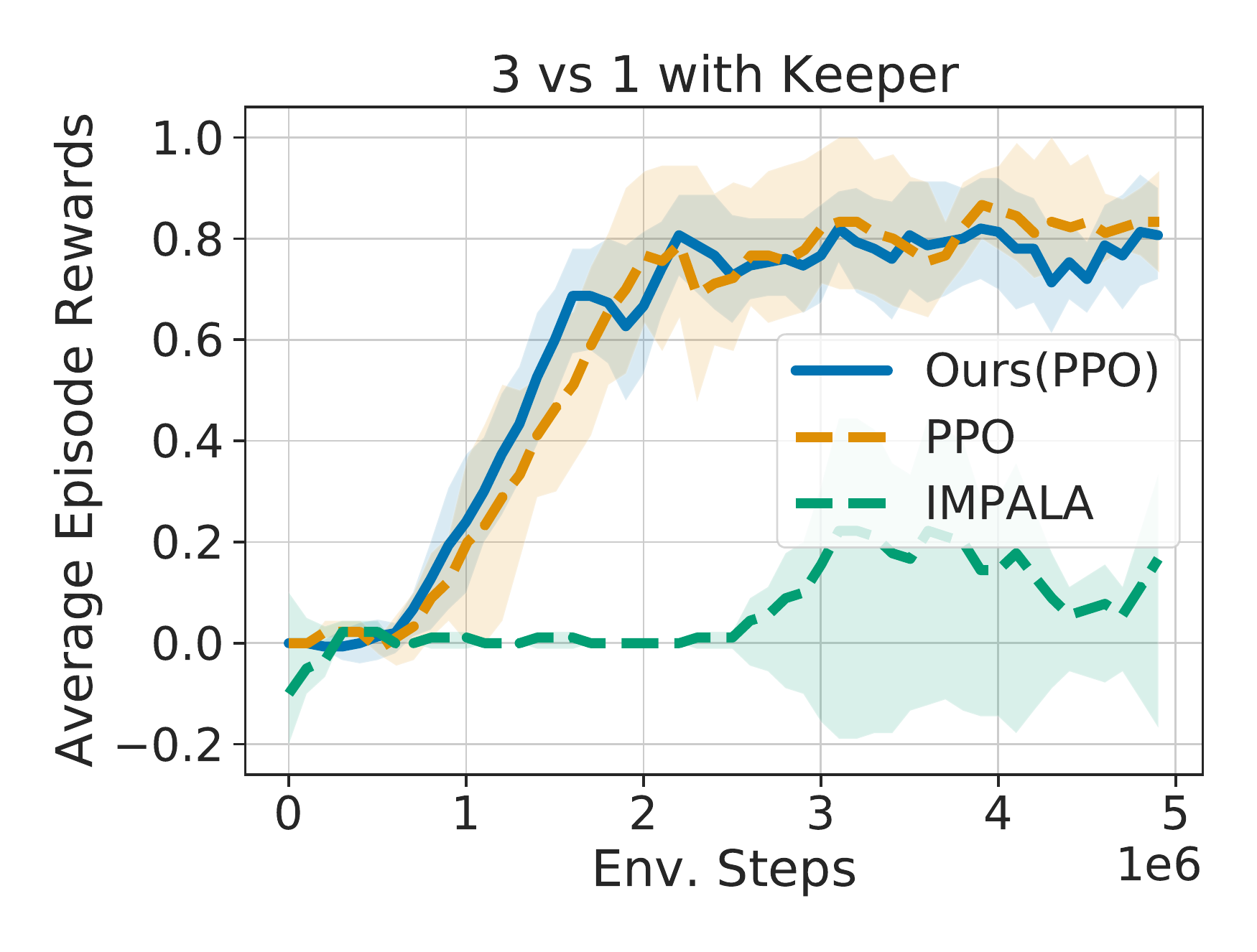}
&\hsp{-0.6cm}
\includegraphics[width=0.33\textwidth]{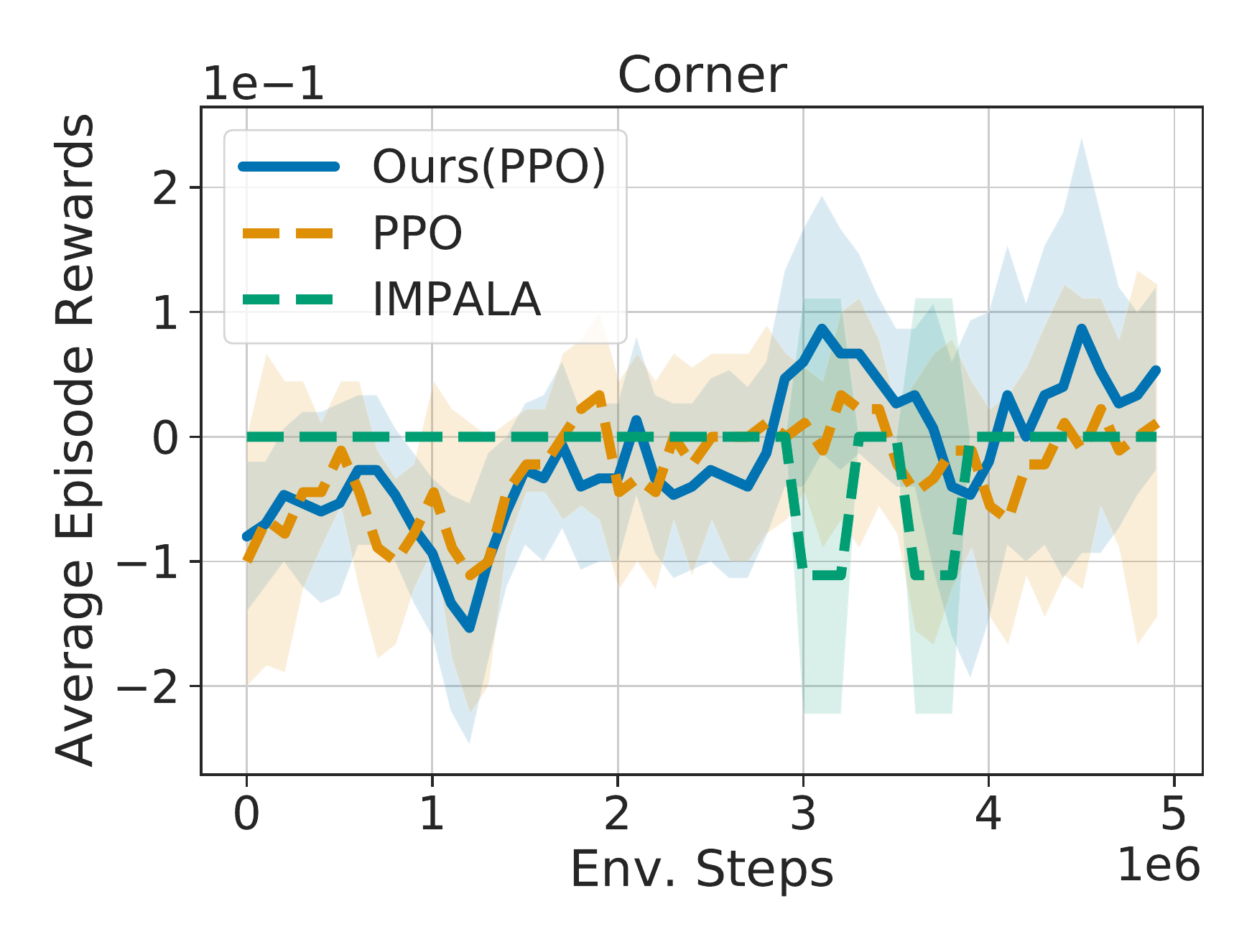}
&\hsp{-0.6cm}
\includegraphics[width=0.33\textwidth]{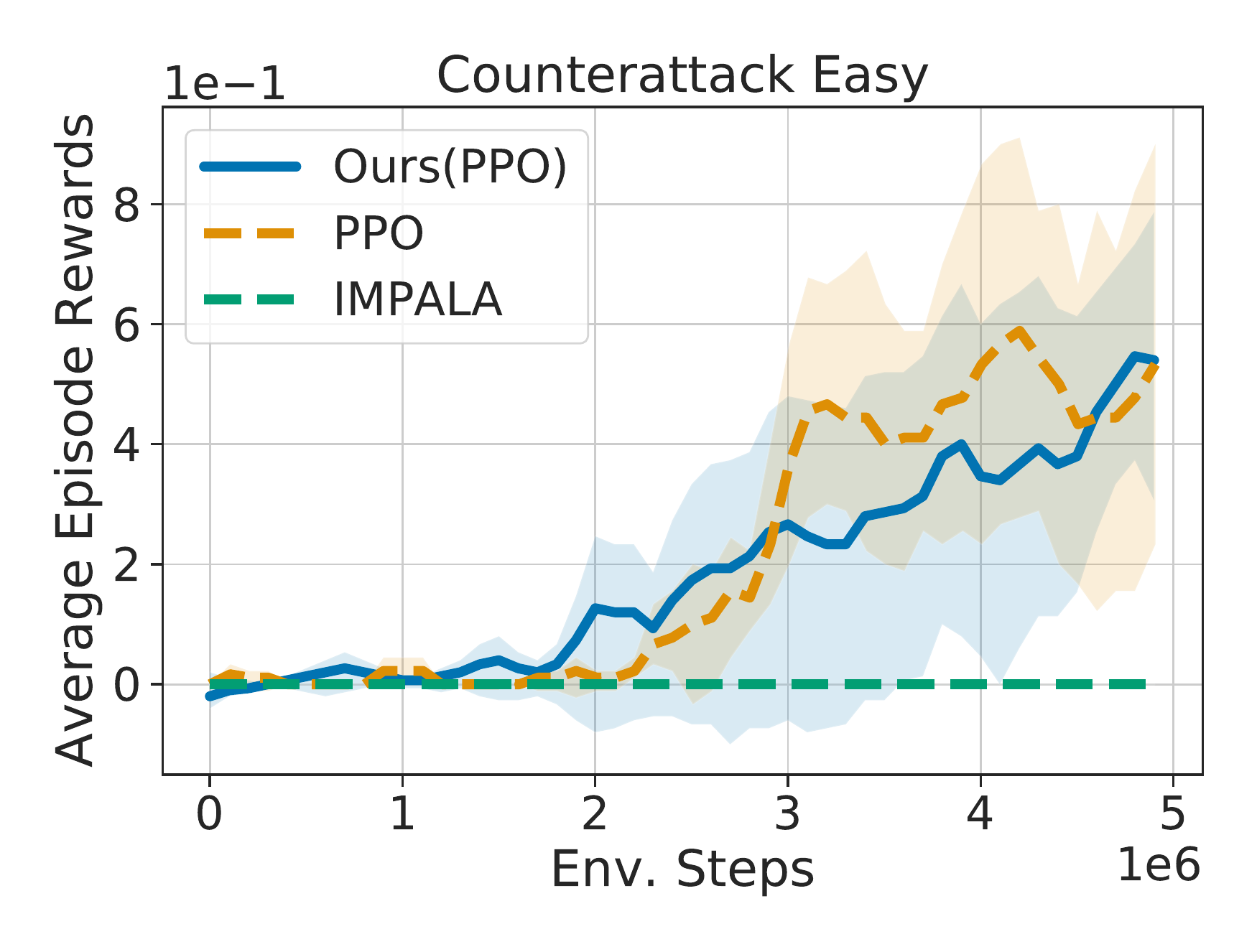}\\

\includegraphics[width=0.33\textwidth]{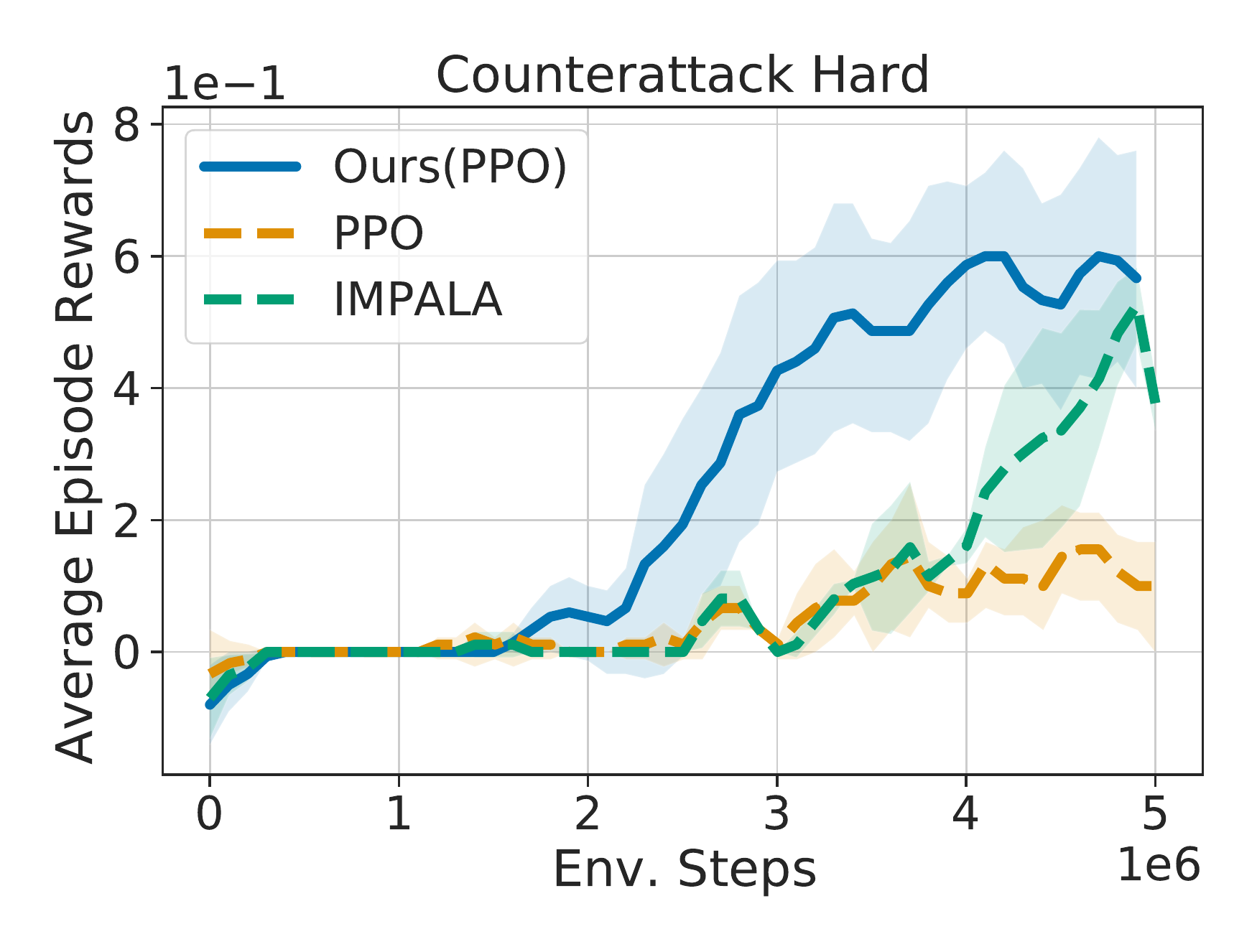}
&\hsp{-0.6cm}
\includegraphics[width=0.33\textwidth]{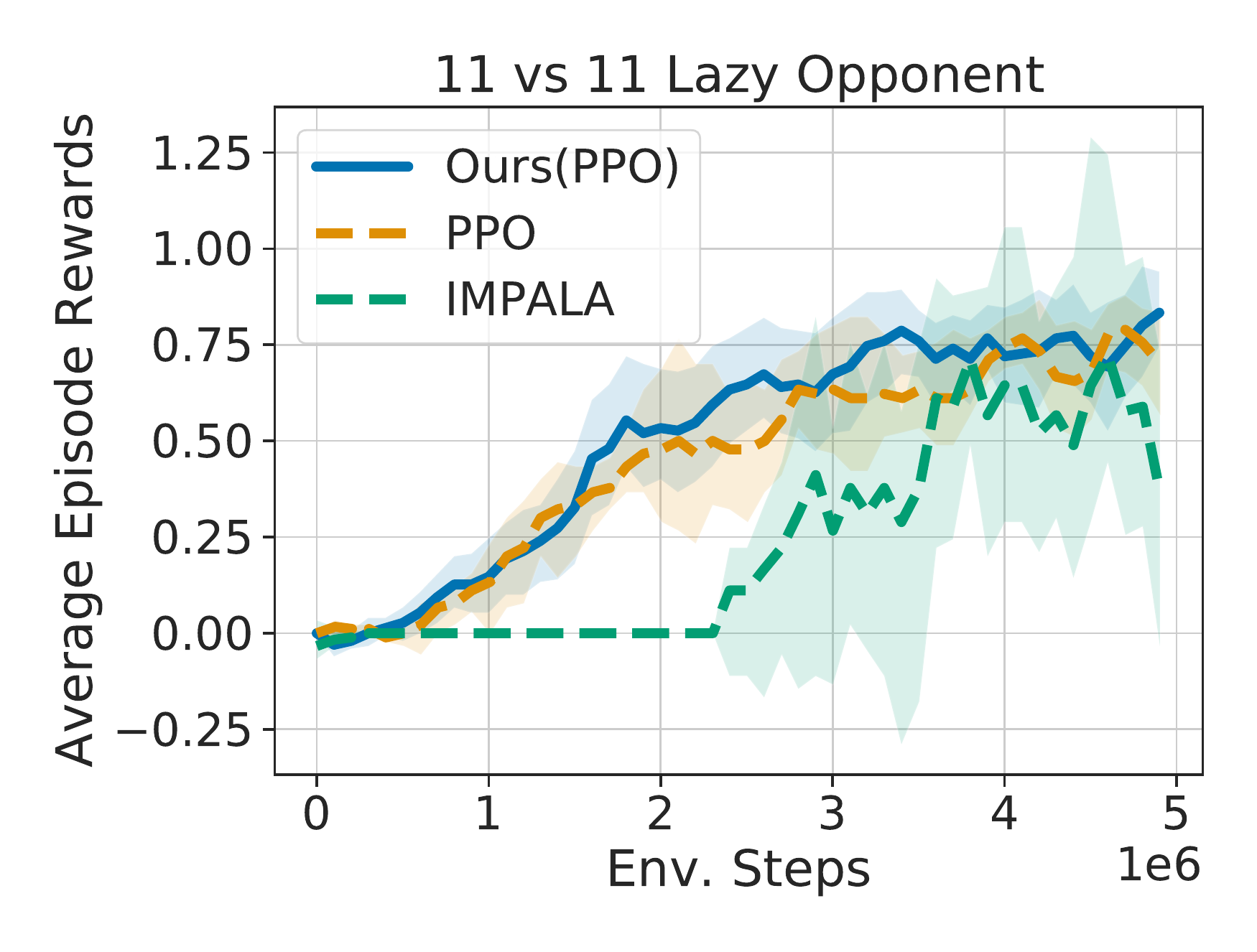}
\end{tabular}\caption{GFootball: {\textbf{Environment step}} versus reward.} 
\label{fig:football_step_plot}

\end{figure*}

\end{document}